\documentclass[3p,11pt]{elsarticle}
\usepackage{geometry}
\usepackage{graphicx}
\usepackage{epstopdf}
\usepackage{todonotes}
\usepackage{mathtools}
\usepackage{caption}
\usepackage[normalem]{ulem}
\usepackage{amsmath}
\usepackage{amssymb} 
\usepackage{amsthm}
\usepackage{amsfonts}
\usepackage{amsbsy}
\usepackage{enumerate}
\usepackage{enumitem}
\usepackage{thmtools}

\usepackage[unicode=true,bookmarks=true,bookmarksnumbered=false,bookmarksopen=false,breaklinks=true,pdfborder={0 0 0},pdfborderstyle={},colorlinks=true]{hyperref}

\ifpdf
\DeclareGraphicsExtensions{.eps,.pdf,.png,.jpg}
\else
\DeclareGraphicsExtensions{.eps}
\fi

\newtheorem{proposition}{Proposition}
\newtheorem{lemma}{Lemma}
\newtheorem{remark}{Remark}

\usepackage{algorithmicx}
\usepackage[ruled]{algorithm}
\usepackage[noend]{algpseudocode}

\def\isarxiv{1}
\usepackage{xstring}

\DeclareMathOperator{\trace}{trace}
\renewcommand{\d}{\mathrm{d}}
\newcommand{\red}[1]{{\color{red} #1}}
\renewcommand{\DH}{\mathcal{D}_{\rm H}}
\newcommand{\DKL}{\mathcal{D}_{\rm KL}}
\newcommand{\E}{\mathbb{E}}

\newcommand{\Var}{\text{Var}}

\newcommand{\R}{\mathbb{R}}

\newcommand{\bs}{\boldsymbol}

\newcommand{\mH}{\mathsf H}

\newcommand{\mG}{\mathsf G}

\newcommand{\mM}{\mathsf M}
\newcommand{\mA}{\mathsf A}

\newcommand{\mC}{\mathsf C}

\newcommand{\mQ}{\mathsf Q}
\newcommand{\mR}{\mathsf R}
\newcommand{\mI}{\mathsf I}
\newcommand{\mV}{\mathsf V}
\newcommand{\chol}{\mathsf L}
\newcommand{\tA}{\bs{\mathsf A}}
\newcommand{\tB}{\bs{\mathsf B}}
\newcommand{\tC}{\bs{\mathsf C}}
\newcommand{\tH}{\bs{\mathsf H}}

\renewcommand{\ldots}{\makebox[1em][c]{.\hfil.\hfil.}}

\newcommand{\appropto}{\mathrel{\vcenter{
  \offinterlineskip\halign{\hfil$##$\cr
    \propto\cr\noalign{\kern2pt}\sim\cr\noalign{\kern-2pt}}}}}
\usepackage{tikz}
\usetikzlibrary{positioning,shapes,calc,3d}
\usepackage{pgfplots}
\usepgfplotslibrary{fillbetween}
\usetikzlibrary{positioning,shapes,calc,backgrounds,fit,3d,plotmarks,patterns}

\pgfplotsset{%
      every axis/.append style={width=.45\textwidth,
                                axis x line=bottom, axis y line=left,
                                x axis line style={very thick,->}, y axis line style={very thick,->},
                                tick align=inside, tick style={thick},
                                every x tick label/.style={font=\small},
                                every y tick label/.style={font=\small},
                                },
      every axis legend/.append style={
                                legend columns=1,
                                font=\small,
                                draw=none,
                                fill=white,
                                },
      every axis x label/.style={at={(0.5,-0.1)},below,fill=none,fill opacity=1,text opacity=1},
      every axis y label/.style={at={(-0.12,0.5)},fill=none,fill opacity=1,text opacity=1,rotate=90},
      }

\newcommand{\Yey}{Y\!{=} y}
\newcommand{\cond}[1]{\pi_{\Theta | Y\ifthenelse{\equal{#1}{}}{}{\!{=}#1}}}

\newcommand{\lowsup}[1]{{#1}}

\newcommand\mydots{\makebox[1em][c]{.\hfil.\hfil.}}

\newcommand{\appname}{
\ifx\isarxiv\undefined
Supplement Material\xspace
\else
Appendix\xspace
\fi
}

\newcommand{\smartref}[1]{
\ifx\isarxiv\undefined
\ref{#1} of the Supplement Material\xspace
\else
Appendix \ref{#1}\xspace
\fi
}

\newcommand{\smartheading}[1]{
\ifx\isarxiv\undefined
\paragraph{#1.}
\else
\subsection{#1}
\fi
}

\allowdisplaybreaks

\newcommand{\reone}[1]{{\color{teal}#1}}
\newcommand{\retwo}[1]{{\color{blue}#1}}
\newcommand{\cancel}[1]{{\color{red}\sout{#1}}}
\newcommand{\canceleq}[1]{#1}
\newcommand{\msout}[1]{{\color{red} \text{\sout{\ensuremath{#1}}}}}

\renewcommand{\cancel}[1]{}
\renewcommand{\canceleq}[1]{}
\renewcommand{\reone}[1]{#1}
\renewcommand{\retwo}[1]{#1}

\begin{document}

\begin{frontmatter}
\title{\bf Scalable conditional deep inverse Rosenblatt transports using tensor-trains and gradient-based dimension reduction}

\author[tc]{Tiangang Cui}
\ead{tiangang.cui@monash.edu}
\author[sd]{Sergey Dolgov}
\ead{s.dolgov@bath.ac.uk}
\author[oz]{Olivier Zahm}
\ead{olivier.zahm@inria.fr}

\address[tc]{School of Mathematics, Monash University, Victoria 3800, Australia}
\address[sd]{Department of Mathematical Sciences, University of Bath, Bath, BA2 7AY, UK}
\address[oz]{Universit\'{e} Grenoble Alpes, Inria, CNRS, Grenoble INP, LJK, 38000 Grenoble, France}


\begin{abstract}
We present a novel offline-online method to mitigate the computational burden of the characterization of posterior random variables in statistical learning. In the offline phase, the proposed method learns the joint law of the parameter random variables and the observable random variables in the tensor-train (TT) format. In the online phase, the resulting order-preserving conditional transport can characterize the posterior random variables given newly observed data in real-time. Compared with the state-of-the-art normalizing flows techniques, the proposed method relies on function approximation and is equipped with thorough performance analysis. The function approximation perspective also allows us to further extend the capability of transport maps in challenging problems with high-dimensional observations and high-dimensional parameters. On the one hand, we present novel heuristics to reorder and/or reparametrize the variables to enhance the approximation power of TT. On the other hand, we integrate the TT-based transport maps and the parameter reordering/reparametrization into layered compositions to further improve the performance of the resulting transport maps. We demonstrate the efficiency of the proposed method on various statistical learning tasks in ordinary differential equations (ODEs) and partial differential equations (PDEs). 
\end{abstract}


\end{frontmatter}


\section{Introduction}\label{sec:intro}

In many statistical applications, a fundamental task is to characterize a random variable $\Theta$ conditioned on the realization of another random variable $\Yey$, denoted by $\Theta|\Yey$.
Here we consider the classical paradigm of Bayesian inference, where $\Theta$ is the unknown parameter of some statistical model and $Y$ represents observable information. In this case, the joint probability density function, denoted by $\pi_{Y,\Theta}$, of the random variables $(Y,\Theta)$ can be defined by the product of the likelihood function $\pi_{Y|\Theta}$ and the prior density $\pi_{\Theta}$. This way, given observed data $y$, the density of the conditional random variable, or the posterior density, is proportional to the joint density, i.e., $\pi_{\Theta|\Yey} \propto \pi_{Y,\Theta}(y,\cdot)$. In what follows, we let $\Theta$ and $Y$ take values in $\R^\lowsup{m}$ and $\R^\lowsup{n}$, respectively.

Most of the classical inference methods adopt an {\bf online} strategy to characterize posterior random variables---most of the computational tasks must be performed after observing the data $y$.
For example, methods such as Markov chain Monte Carlo (MCMC) \cite{liu2008monte,MCMC:BGJM_2011} and sequential Monte Carlo (SMC) \cite{Chopin_2002,SMC_2006} simulate posterior random variables via evaluating the posterior density $\pi_{\Theta|\Yey}$ at each of the candidate samples. 
This online strategy needs to be repeatedly applied to each newly observed data set and may require a significant amount of computational effort when the posterior density is costly to evaluate, for instance, problems involving ODEs, see \cite{ramsay2007parameter,girolami2008bayesian} for examples, or PDEs, see \cite{stuart2010inverse,MCMC:BuiGha_2012,MCMC:Petra_etal_2014} for examples.
To extend the applicability of the Bayesian framework to computationally costly but time-sensitive problems, we present a scalable {\bf offline-online} strategy to enable fast and accurate characterizations of posterior random variables for multiple sets of observations.

At the heart of our offline-online strategy is the approximation of the joint density $\pi_{Y,\Theta}$ by $p_{Y,\Theta} \coloneqq \mathcal{T}_\sharp\, \rho_{Y,\Theta}$, which is the pushforward of a product-form reference probability density $\rho_{Y,\Theta}(u_Y,u_\Theta)=\rho_{Y}(u_Y)\otimes\rho_{\Theta}(u_\Theta)$ on $\R^\lowsup{m} \times \R^\lowsup{n}$ under an order-preserving map
\begin{equation}\label{eq:map}
  \left[ \begin{array}{l} y \\ \theta \end{array}\right] = \mathcal{T}(u_Y, u_\Theta) = \left[ \begin{array}{l}\mathcal{T}_{Y}(u_Y) \\ \mathcal{T}_{\Theta}(u_Y, u_\Theta) \end{array}\right] .
\end{equation}
As suggested in \cite{baptista2020adaptive,spantini2019coupling}, the above triangular structure offers several useful properties: the transformed random variable $Y = \mathcal{T}_Y(U_Y)$, $U_Y\sim \rho_Y$, simply follows the marginal density $p_Y$; and more interestingly, conditioned on an observed data set $y$ and its corresponding reference variable $u_Y = \mathcal{T}_Y^\lowsup{-1}(y)$, the output of the {\bf conditional map} 
\begin{equation}\label{eq:cmap}
  \mathcal{T}_{\Theta|\Yey}(U_\Theta) \coloneqq \mathcal{T}_{\Theta}( \mathcal{T}_Y^\lowsup{-1}(y), U_\Theta ), \quad  U_\Theta \sim \rho_\Theta,
\end{equation}
follows the conditional density $p_{\Theta|\Yey}$ that approximates the posterior density $\pi_{\Theta|\Yey}$. 
This allows us to devote most of the computational resources to the offline phase to learn the map $\mathcal{T}$ before observing any data, in order to approximate the original joint density $\pi_{Y,\Theta}$ by $\mathcal{T}_\sharp\, \rho_{Y,\Theta}$. Then, in the online phase, given newly observed data, we can either carry real-time approximate posterior inference using the approximate density defined by the conditional map $\mathcal{T}_{\Theta|Y{=}y}$, or accelerate exact posterior inference using $\mathcal{T}_{\Theta|Y{=}y}$ as a preconditioner.

We present a novel framework that learns the map $\mathcal{T}$ \retwo{for high-dimensional data and parameters} from a {\bf function approximation} perspective.
\cancel{Following the work of [12,13], we employ the functional TT decomposition [14,15] of the joint density to explicitly construct the order-preserving map $\mathcal{T}$ and the subsequent conditional map $\mathcal{T}_{\Theta|Y{=}y}$.}
\retwo{We first extend the TT-based methods of \cite{cui2020deep,dolgov2020approximation} to the offline-online setting by explicitly constructing the order-preserving map $\mathcal{T}$ targeting the joint variables $(Y,\Theta)$ and the subsequent conditional map $\mathcal{T}_{\Theta|Y{=}y}$.}
For many problems of interests, the joint density can be concentrated to some sub-manifold. Then, the nonlinear interaction and the potential high-dimensionality of variables in the joint density can deteriorate the approximation power of TT. We propose a combined treatment to overcome this barrier.
\cancel{We first present a novel gradient-based method to either reorder or reparametrize the variables to enhance the approximation power of TT. }
\retwo{We first present a novel gradient-based method to either reorder or reparametrize the variables to enhance the approximation power of TT, and to reduce the dimensions of parameters and data with a controlled accuracy.}
\cancel{Then, we employ the deep inverse Rosenblatt transport (DIRT) method [12] to adaptively build the conditional map into a composition of layers, in which each layer of map is easier to construct.}
\retwo{To handle complicated interactions and concentration of the target density, we then employ the deep inverse Rosenblatt transport (DIRT) method of \cite{cui2020deep} to adaptively build the conditional map---after reordering or reparametrization---into a composition of layers, in which each layer of map is easier to construct.}

The function approximation perspective permits the control of the approximation error of the joint density in the Hellinger distance 
\begin{equation}\label{eq:Hellinger}
 \DH( \pi_{Y,\Theta}, p_{Y,\Theta} ) \coloneqq \Big(\frac12 \int \left(\sqrt{\pi_{Y,\Theta}}-\textstyle{\sqrt{p_{Y,\Theta}}}\right)^2 \d y \d\theta  \Big)^\frac12.
\end{equation}
The performance of the resulting conditional map is guaranteed with high probability. 
As shown in \ref{proof:MarkovBound}, if $\DH( \pi_{Y,\Theta}, p_{Y,\Theta} ) \leq \varepsilon $ for some $\varepsilon < \sqrt2 / 4$, then the expectation of any function $\theta \mapsto h(\theta)$ over the posterior has an error satisfying
\begin{equation}\label{eq:MarkovBound}
  \left| \E_{\pi_{\Theta|Y}}(h) - \E_{p_{\Theta|Y}}(h)\right|
  \leq \frac{4 \varepsilon}{ \sqrt2 \delta -4 \varepsilon} \left(\sqrt{ \Var_{\pi_{\Theta|Y}}(h)} + \sqrt{\Var_{p_{\Theta|Y}}(h)}\right),
\end{equation}
with probability greater than $1{-}\delta$ for some $\delta > 2 \sqrt{2} \, \varepsilon$. 

We highlight relevant works in transport maps, see, e.g., \cite{baptista2020adaptive,bigoni2019greedy,kovachki2020conditional,parno2018transport,tabak2013family,trigila2016data} and normalizing flows, see, e.g., \cite{caterini2021variational,chen2019residualflows,pmlr-v119-cornish20a,kruse2019hint,papamakarios2021normalizing}, in which the map is identified by minimizing the Kullback-Leibler (KL) divergence between the target density and its approximation over some class of tractable triangular maps. By and large, these approaches adopt a {\bf density estimation} perspective in which one only has access to samples from the target density, and so minimizing the KL divergence boils down to (penalized) maximum likelihood estimation.
In contrast, we build the map using function approximation techniques in which the square root of the density is seen as the function to approximate, and hence the Hellinger distance \eqref{eq:Hellinger} can be interpreted as an $L^\lowsup{2}$ error. In addition, the access to pointwise evaluations of the density and its derivatives enables the gradient-based variable ordering method, which is the key to successfully implementing our method for high-dimensional problems.

Section \ref{sec:background} provides background of the Rosenbaltt transport and the TT-based density approximation. Section \ref{sec:Variable_ordering} introduces the gradient-based method for variable reordering and reparametrization. Section \ref{sec:DIRT} presents the DIRT-based construction of conditional maps. Section \ref{sec:online} presents online inference algorithms based on resulting conditional maps. Section \ref{sec:numerics} provides several applications. In the Appendix, we provide proofs, derivations, and additional numerical results section-by-section.

\section{From tensor-train to the conditional transport}\label{sec:background}
We first introduce the conditional transport as a  Rosenblatt transport \cite{rosenblatt1952remarks}, and then discuss its numerical construction using TT. The following notation is used throughout the paper. 
\cancel{For a $d$-dimensional random variable $X$ and an integer $1 {<} k {<} d$, we express the first $k{-}1$ coordinates and the last $d{-}k$ coordinates of $X$ as
\(
X_{<k} := [X_1, \ldots, X_{k-1}]^\top
\) 
and 
\(
X_{>k} := [X_{k+1}, \ldots, X_{d}]^\top,
\)
respectively. We also define $X_{\leq k} \equiv [X_{<k}^\top, X_k]^\top$, $X_{\geq k} \equiv [X_{k}, X_{>k}^\top]^\top$, $X_{\leq 1} \equiv X_1$, and $X_{\geq d} \equiv X_{d}$. To be consistent with TT, we index the observable variable $Y$ in the reverse order and the parameter $\Theta$ in the forward order, i.e., $(Y, \Theta) := (Y_m, \ldots, Y_{2}, Y_1, \Theta_1, \Theta_2, \ldots, \Theta_n)$. The tuple expression $Y,\Theta = y,\theta$ is used in the subscript to denote $Y = y$ and $\Theta = \theta$.}
\reone{For a $d$-dimensional random variable $X$ and integers $i,j \in [d] := \{1, 2, \ldots, d\}$, we define $X_{i:j} := (X_i, \ldots, X_j)$. Note that for $i < j$, this leads to $X_{i:j} := (X_i, X_{i+1}, \ldots, X_j)$, whereas for $i > j$, this gives $X_{i:j} := (X_i, X_{i-1}, \ldots, X_j)$. To be consistent with TT, we index the observable variable $Y$ in the reverse order and the parameter $\Theta$ in the forward order, i.e., $(Y, \Theta) := (Y_m, \ldots, Y_{2}, Y_1, \Theta_1, \Theta_2, \ldots, \Theta_n)$.}

\subsection{Rosenblatt transport}\label{sec:RT}

Suppose $(Y, \Theta)$ jointly follow the density $\pi_{Y,\Theta}$ with respect to the Lebesgue measure. 
As a starting point, we consider the marginal densities of the random variables \cancel{$Y_m \equiv Y_{\geq m}, Y_{\geq m-1}, \ldots, Y_{\geq 2}, Y_{\geq 1} \equiv Y$} \reone{$Y_m, Y_{m:m-1}, \ldots, Y_{m:2}, Y_{m:1} \equiv Y$}, which can be expressed as
\canceleq{\[
\msout{\pi_{Y_{\geq j}}(y_{\geq j}) :=  \left\{ \begin{array}{ll} 
  \displaystyle \int \int \pi_{Y,\Theta}( y_{\geq j}, y_{< j}, \theta ) \, \d y_{< j}\, \d \theta,  &  m \geq  j > 1  \\
  \displaystyle \int \pi_{Y,\Theta}( y, \theta )\, \d\theta, & j = 1 
\end{array}\right.;}
\]}
\reone{
\begin{align}\label{eq:marginal_y}
  \pi_{Y_{m:j}}(y_{m:j})  := \left\{ \begin{array}{rll} 
  \displaystyle \int \hspace{-12pt} & \displaystyle \int \pi_{Y,\Theta}( y_{m:j}, y_{j-1:1}, \theta ) \, \d y_{j-1:1}\, \d \theta,  &   j = m, m-1, \ldots, 2  \\
  & \displaystyle \int \pi_{Y,\Theta}( y, \theta )\, \d\theta, & j = 1 
  \end{array}\right.;
\end{align}}
and the marginal densities of \cancel{$(Y, \Theta_{\leq 1}), (Y, \Theta_{\leq 2}), \ldots , (Y, \Theta_{\leq n})\equiv (Y, \Theta)$} \reone{$(Y, \Theta_{1}), (Y, \Theta_{1:2}), \ldots , (Y, \Theta_{1:n})\equiv (Y, \Theta)$}, which are given by
\canceleq{\[
  \msout{\pi_{Y,\Theta_{\leq k}}(y, \theta_{\leq k}) :=  \left\{ \begin{array}{ll} \displaystyle\int \pi_{Y,\Theta}( y, \theta_{\leq k}, \theta_{>k} )\, \d\theta_{>k}, & 1 \leq k < n \\
    \pi_{Y,\Theta}(y, \theta), & k = n 
  \end{array}\right..}
\]}
\reone{
\begin{align}\label{eq:marginal_theta}
\pi_{Y,\Theta_{1:k}}(y, \theta_{1:k})  := \left\{ \begin{array}{rll} \displaystyle\int \hspace{-12pt} & \pi_{Y,\Theta}( y, \theta_{1:k}, \theta_{k+1:n} )\, \d\theta_{k+1:n}, &  k = 1,2,\ldots, n-1 \\
& \pi_{Y,\Theta}(y, \theta), & k = n 
\end{array}\right..
\end{align}}
Then, one can decompose the joint density $\pi_{Y,\Theta}$ as
\canceleq{\begin{align*}
    \msout{\pi_{Y,\Theta}(y,\theta)} & \msout{= \pi_{Y_m}(y_m) \! \left( \prod_{j = m{-}1}^{1} \frac{\pi_{Y_{\geq j}}(y_{\geq j})}{\pi_{Y_{>j}}(y_{> j})} \right) 
    \frac{\pi_{Y, \Theta_{1}}(y, \Theta_1)}{\pi_{Y}(y)} \left( \prod_{k = 2}^{n} \frac{\pi_{Y, \Theta_{\leq k}}(y, \theta_{\leq k})}{\pi_{Y, \Theta_{<k}}(y, \theta_{< k})} \right) } \\
    & \msout{=  \pi_{Y_m}(y_{m}) \! \left( \prod_{j=m{-}1}^{1} \hspace{-10pt} \pi_{Y_j | Y_{>j}{=}y_{>j}} ( y_j | y_{>j} ) \hspace{-4pt} \right)   
    \! \pi_{\Theta_1 | Y{=}y} ( \theta_1 | y ) \! \left( \prod_{k = 2}^{n} \hspace{-3pt} \pi_{\Theta_k | Y,\Theta_{<k}{=}y,\theta_{<k}} ( \theta_k | y, \theta_{<k} ) \hspace{-3pt} \right) ,} \nonumber
\end{align*}}
\reone{
\begin{align}\label{eq:conditionals}
\pi_{Y,\Theta}(y,\theta) & = \pi_{Y_m}(y_m) \! \left( \prod_{j = m{-}1}^{1} \frac{\pi_{Y_{m:j}}(y_{m:j})}{\pi_{Y_{m:j+1}}(y_{m:j+1})} \right) 
\frac{\pi_{Y, \Theta_{1}}(y, \Theta_1)}{\pi_{Y}(y)} \left( \prod_{k = 2}^{n} \frac{\pi_{Y, \Theta_{1:k}}(y, \theta_{1:k})}{\pi_{Y, \Theta_{1:k-1}}(y, \theta_{1:k-1})} \right)  \\
& =  \pi_{Y_m}(y_{m}) \! \left( \prod_{j=m{-}1}^{1} \hspace{-10pt} \pi_{Y_j | Y_{m:j+1}} ( y_j | y_{m:j+1} ) \hspace{-4pt} \right)   
\! \pi_{\Theta_1 | Y} ( \theta_1 | y ) \! \left( \prod_{k = 2}^{n} \hspace{-3pt} \pi_{\Theta_k | Y,\Theta_{1:k-1}} ( \theta_k | y, \theta_{1:k-1} ) \hspace{-3pt} \right) , \nonumber
\end{align}
}
which is the product of a sequence of one dimensional marginal and conditional densities. 

We denote the distribution functions of the marginal random variable $Y_m$, each of the conditional random variables \cancel{$Y_j | Y_{>j}{=}y_{>j}$} \reone{$Y_j | Y_{m:j+1}$}, the conditional random variable $\Theta_1 | Y$, and each of the conditional random variables \cancel{$\Theta_k | Y, \Theta_{<k} {=} y,\theta_{<k}$} \reone{$\Theta_k | Y, \Theta_{1:k-1}$} by 
\reone{ 
\[
F_{Y_m}(y_m), \quad F_{Y_j | Y_{m:j+1}}(y_j|y_{m:j+1}), \quad F_{\Theta_1 | Y}(\theta_1|y), \quad \text{and} \quad F_{\Theta_k | Y,\Theta_{1:k-1}}(\theta_k | y, \theta_{1:k-1}),
\]
}
respectively. Each of the distribution functions defines a univariate order-preserving transformation that maps the corresponding (conditional) random variable to a uniform random variable on $[0,1]$. This sequence of distribution functions leads to the Rosenblatt transport
$\mathcal{F}:\R^\lowsup{m{+}n}\rightarrow [0,1]^\lowsup{m{+}n}$ defined by $\mathcal{F}(y,\theta)=[\mathcal{F}_Y(y) , \mathcal{F}_{\Theta|Y}(\theta|y)]^\lowsup{\top}$ where
\canceleq{\begin{equation*} 
    \arraycolsep=-1pt
    \msout{\mathcal{F}_Y(y)=
     \left[
     \begin{array}{ll}
      F_{Y_{m}} & (y_{m}) \\ 
      F_{Y_{m{-}1}|Y_m{=}y_m } & (y_{m{-}1}|y_{m}) \\ 
      \quad\vdots &\\ 
      F_{Y_1|Y_{>1}{=}y_{>1}} & (y_{1}|y_{>1})
     \end{array}
     \right]
     \quad\text{and}\quad
     \mathcal{F}_{\Theta|\Yey}(\theta|y)=
     \left[
     \begin{array}{ll}
      F_{\Theta_{1}|\Yey} & (\theta_{1}|y) \\ 
      F_{\Theta_{2}|Y,\Theta_{1}{=}y,\theta_1} & (\theta_{2}|y, \theta_{1}) \\ 
      \quad\vdots & \\ 
      F_{\Theta_{n}|Y,\Theta_{<n}{=}y,\theta_{<n}} & (\theta_{n}|y,\theta_{<n})
     \end{array}
     \right].}
\end{equation*}}
\reone{
\begin{equation}\label{eq:RosenblattComponents} 
\arraycolsep=-1pt
\mathcal{F}_Y(y)=
 \left[
 \begin{array}{ll}
  F_{Y_{m}} &(y_{m}) \\ 
  F_{Y_{m{-}1}|Y_m }&(y_{m{-}1}|y_{m}) \\ 
  \quad\vdots &\\ 
  F_{Y_1|Y_{m:2}}&(y_{1}|y_{m:2})
 \end{array}
 \right]
 \quad\text{and}\quad
 \mathcal{F}_{\Theta|Y}(\theta|y)=
 \left[
 \begin{array}{ll}
  F_{\Theta_{1}|Y}&(\theta_{1}|y) \\ 
  F_{\Theta_{2}|Y,\Theta_{1}}&(\theta_{2}|y, \theta_{1}) \\ 
  \quad\vdots & \\ 
  F_{\Theta_{n}|Y,\Theta_{1:n-1}}&(\theta_{n}|y,\theta_{1:n-1})
 \end{array}
 \right].
\end{equation}
}
The Rosenblatt transport has a lower-triangular structure, because the evaluation of each row of $\mathcal{F}$ involves a univariate (conditional) distribution function that only depends on variables in the previous dimensions. 
As a result, the Rosenblatt transport maps the joint variables $(Y, \Theta)$ to uniform random variables $(U_Y, U_\Theta)$ on the hypercube $[0,1]^{m{+}n}$ dimension-by-dimension. Denoting the density function of the uniform measure on $[0,1]^\lowsup{m{+}n}$ by $\mu_{Y,\Theta}$, the pushforward of the joint density $\pi_{Y,\Theta}$ under the Rosenblatt transport satisfies $\mathcal{F}_\sharp \, \pi_{Y,\Theta} (u_Y, u_\Theta) =\mu_{Y,\Theta} (u_Y, u_\Theta)$. 

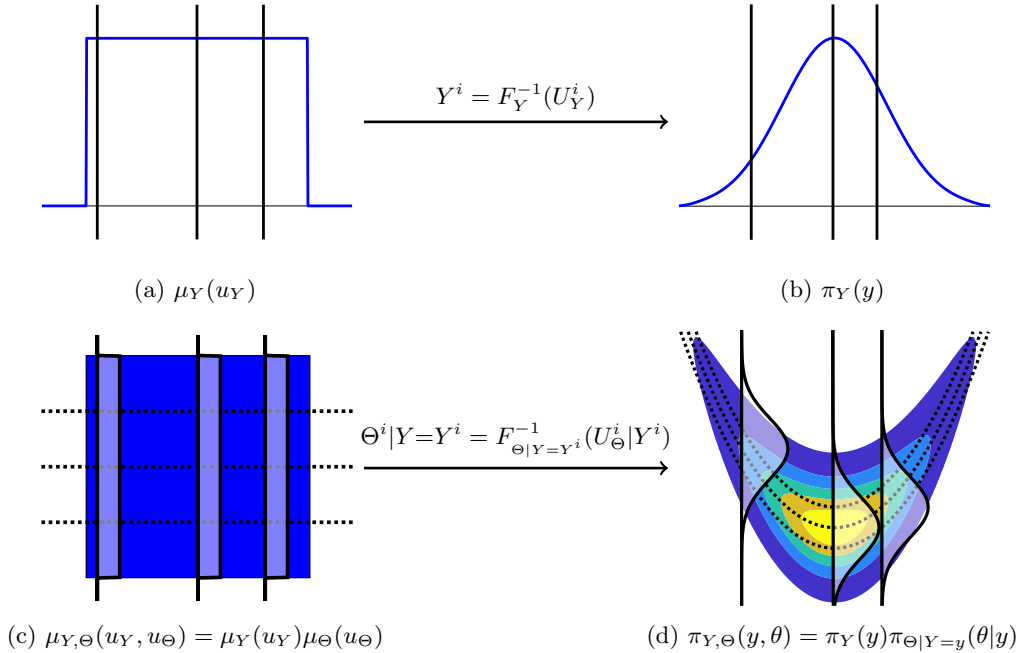
\begin{figure}[h!]
\begin{center}
\vspace{-1em}
\begin{tikzpicture}
\node[anchor=west, minimum width=0.26\linewidth, minimum height=0.26\linewidth] (r2) {\begin{minipage}{0.25\linewidth}\includegraphics[width=\linewidth]{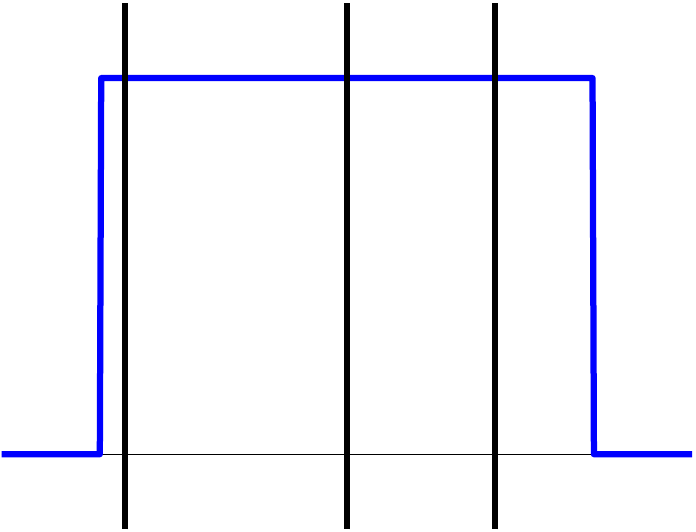}\end{minipage}};

\node[anchor=north] (u2) at ($(r2.south)+(0,0.2)$) {\begin{minipage}{0.4\linewidth} \centering\hspace{-1em}\footnotesize (a) $\mu_{Y}(u_Y) $\hspace{-1em} \end{minipage}};

\node[anchor=west, minimum width=0.26\linewidth, minimum height=0.26\linewidth] (s2) at ($(r2.east)+(4,0)$) {\begin{minipage}{0.25\linewidth}\centering \includegraphics[width=\linewidth]{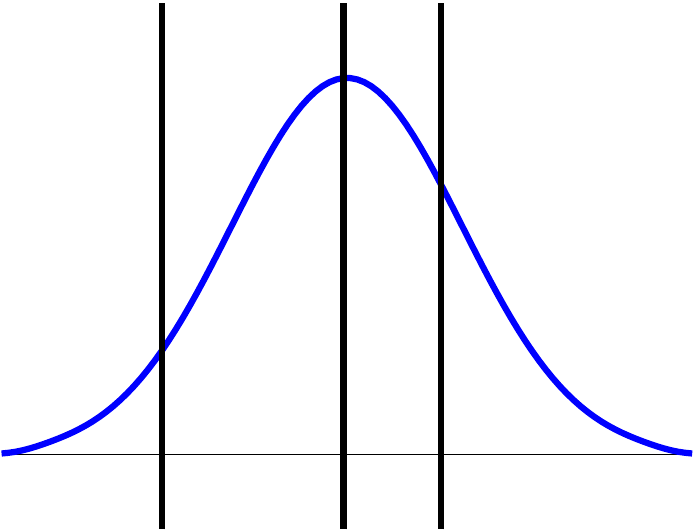}\end{minipage}};

\node[anchor=north] (j2) at ($(s2.south)+(0,0.2)$) {\begin{minipage}{0.25\linewidth} \centering\hspace{-1em}\footnotesize (b) $\pi_{Y}(y) $\hspace{-1em} \end{minipage}};

\node[anchor=north, minimum width=0.26\linewidth, minimum height=0.26\linewidth] (r1) at ($(r2.south)-(0,0.3)$) {\begin{minipage}{0.25\linewidth}\centering\includegraphics[width=\linewidth, trim=0 20pt 0 20pt, clip]{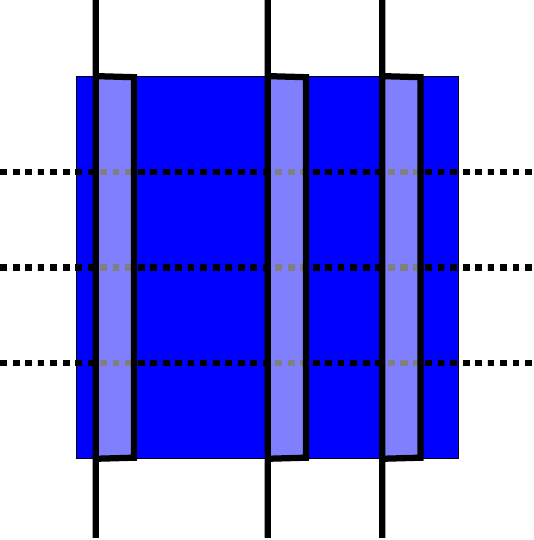} \end{minipage}};

\node[anchor=north] (u1) at ($(r1.south)+(0,0.2)$) {\begin{minipage}{0.4\linewidth} \centering\hspace{-1em}\footnotesize (c) $\mu_{Y,\Theta}(u_Y,u_\Theta) = \mu_{Y}(u_Y) \mu_{\Theta}(u_\Theta) $\hspace{-1em} \end{minipage}};

\node[anchor=north, minimum width=0.26\linewidth, minimum height=0.26\linewidth] (s1) at ($(s2.south)-(0,0.3)$) {\begin{minipage}{0.25\linewidth}\centering\includegraphics[width=\linewidth]{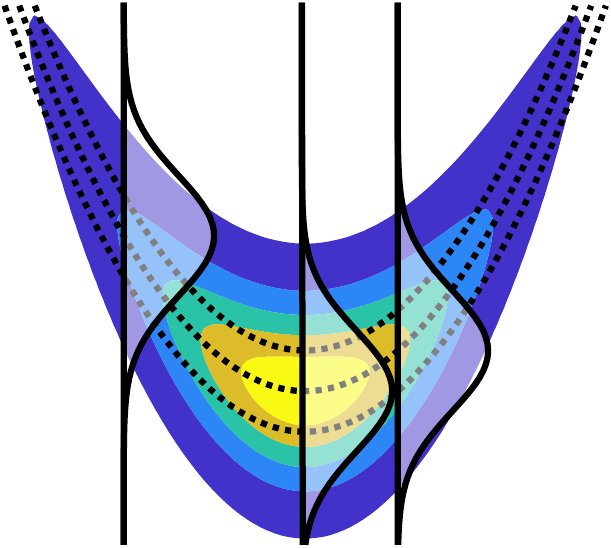} \end{minipage}};

\node[anchor=north] (j1) at ($(s1.south)+(0,0.2)$) {\begin{minipage}{0.4\linewidth} \centering\hspace{-1em}\footnotesize (d) $\pi_{Y,\Theta}(y,\theta) = \pi_{Y}(y) \pi_{\Theta|Y{=}y}(\theta|y) $\hspace{-1em} \end{minipage}};

\draw[->,line width=1pt] (r2.east) -- (s2.west) node[midway,above] {\footnotesize $Y^i = F_{Y}^{-1}(U_Y^i)$};
\draw[->,line width=1pt] (r1.east) -- (s1.west) node[midway,above] {\footnotesize $\Theta^i|Y{=}Y^i = F^{-1}_{\scriptscriptstyle \Theta|Y{=}Y^i}(U_\Theta^i| Y^i)$};
\end{tikzpicture}\vspace{-1.5em}
\end{center}
\caption{An example of the inverse Rosenblatt transport with univariate random variables $Y$ and $\Theta$. Here $Y$ has the marginal density $\pi_Y$ and $(Y, \Theta)$ jointly follow the density $\pi_{Y,\Theta}$. From (a) to (b): mapping uniform random variables $U_Y^i$ to $Y^i$ by inverting the marginal distribution function. From (c) to (d): drawing uniform random variables $U_\Theta^i$ and mapping them to $\Theta^i|Y{=}Y^i$ by inverting the conditional distribution function. \retwo{The top row shows density functions in one dimension. The bottom row shows contours of density functions in two dimension.} }\label{fig:demo}
\end{figure}

By inverting $\mathcal{F}$, we can equivalently express the joint density $\pi_{Y, \Theta}$ as the pushforward of the uniform density $\mu_{Y,\Theta}$ under the inverse map $\mathcal{T}:=\mathcal{F}^\lowsup{-1}$, i.e., $\pi_{Y,\Theta} = \mathcal{T}_\sharp\,\mu_{Y,\Theta}$. 
Thus, joint random variables $(Y,\Theta)$ can be drawn by mapping uniform random variables $(U_Y, U_\Theta) \sim \mu_{Y,\Theta}$ through the inverse Rosenblatt transport $\mathcal{T}$. Similar to the Rosenblatt transport, the inverse map $(y,\theta) = \mathcal{T}(u_Y, u_\Theta)$ also has a lower-triangular structure in the block form of
\[
\left[\begin{array}{ll} y \\ \theta \end{array} \right] = \left[\begin{array}{ll} \mathcal{F}^{-1}_Y(u_Y) \\ \mathcal{F}^{-1}_{\Theta|Y}( u_\Theta | y) \end{array} \right],
\]
and thus can be evaluated as a sequence of univariate inverse transforms. See Figure \ref{fig:demo} for a two dimensional example. 
As a final remark, for any observed data $y$, we can characterize the posterior random variable $\Theta|\Yey$ by using the inverse conditional map $\mathcal{T}_{\Theta|\Yey}(U_\Theta) = \mathcal{F}^{-1}_{\Theta|Y}(U_\Theta|y)$ with $U_\Theta \sim \mu_\Theta$. In the rest of this paper, we will introduce the construction of the Rosenblatt transport for high-dimensional random variables and the multilayered construction of the conditional map that follows the general form of \eqref{eq:cmap}.

\subsection{Tensor-train for density approximation and marginalization}

The key to materializing the Rosenblatt transport is the construction of the marginal densities in \eqref{eq:marginal_y} and \eqref{eq:marginal_theta}, and hence the conditional density functions and distribution functions. 
Treating the multivariate density function as a continuous analogue of a tensor, we can compute the sequence of marginalizations by decomposing the density function into a functional-form of TT \cite{bigoni2016spectral,gorodetsky2019continuous}. Specifically, we consider the TT decomposition of the square root of the joint density,
\begin{align}\label{eq:g_TT}
 \sqrt{\pi_{Y,\Theta}(y,\theta)} \approx g_{Y,\Theta}(y,\theta) & := \mG_{-m}(y_m)   \hdots   \mG_{-1}(y_1)   \mG_1(\theta_1)   \hdots   \mG_n(\theta_n) 
\end{align}
where each $\mG_{i} {\,:\,} \R {\,\rightarrow\,} \R^\lowsup{r_{i{-}1} \times r_{i}} $ is a matrix-valued function called the $i$-th TT core and the function $\text{rank}_\text{TT}(g_{Y,\Theta}) = (r_{{-}(m{+}1)}, \ldots, r_n )$ gives the TT ranks of $g_{Y,\Theta}$. Here negative integers are used to index the TT cores of the observable variables $y_j, m {\,\geq\,} j {\,\geq\,} 1$. 
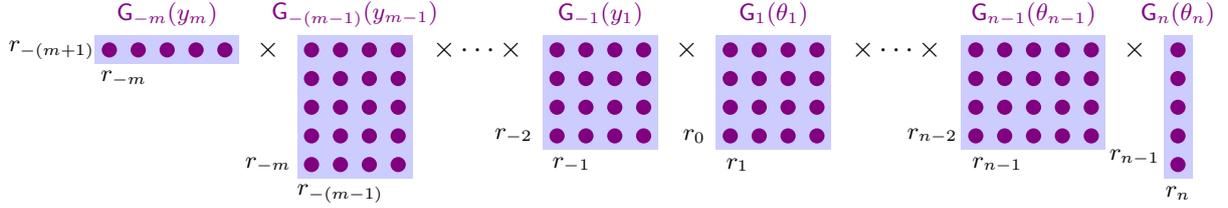
\begin{figure}[h!]
\begin{center}
\begin{tikzpicture}[scale=0.95]
\def\blue{blue!20!white}
\def\red{red!50!blue}
\def\xshift{0}
\def\yshift{1.6}
\def\h{0.4}
\def\w{2}
\fill[\blue] (\xshift,\yshift) rectangle (\xshift+\w,\yshift+\h);
\foreach \x in {0.2,0.6,...,\w} {
  \foreach \y in {0.2} {
    \fill[\red] (\xshift+\x,\yshift+\y) circle (3pt);
    }
  }
\node at (\xshift-0.6,\yshift+0.2) {\footnotesize $r_{{-}(m{+}1)}$};
\node at (\xshift+0.4,\yshift-0.2) {\footnotesize $r_{{-}m}$};
\node[text=\red] at (\xshift+1,\yshift+\h+0.3) {\footnotesize $\mG_{-m}(y_m)$};
\node at (2.4, 1.8) {$\times$};
\def\xshift{2.8} 
\def\yshift{0}
\def\h{2}
\def\w{1.6}
\fill[\blue] (\xshift,\yshift) rectangle (\xshift+\w,\yshift+\h);
\foreach \x in {0.2,0.6,...,\w} {
  \foreach \y in {0.2,0.6,...,\h} {
    \fill[\red] (\xshift+\x,\yshift+\y) circle (3pt);
    }
  }  
\node at (\xshift-0.4,\yshift+0.2) {\footnotesize $r_{{-}m}$};
\node at (\xshift+0.6,\yshift-0.2) {\footnotesize $r_{{-}(m{-}1)}$};
\node[text=\red] at (\xshift+0.8,\yshift+\h+0.3) {\footnotesize $\mG_{-(m{-}1)}(y_{m{-}1})$};
\node at (5.3, 1.8) {$\times \cdots \times$};
\def\xshift{6.2}
\def\yshift{0.4}
\def\h{1.6}
\def\w{1.6}
\fill[\blue] (\xshift,\yshift) rectangle (\xshift+\w,\yshift+\h);
\foreach \x in {0.2,0.6,...,\w} {
  \foreach \y in {0.2,0.6,...,\h} {
    \fill[\red] (\xshift+\x,\yshift+\y) circle (3pt);
    }
  }  
\node at (\xshift-0.4,\yshift+0.2) {\footnotesize $r_{{-}2}$};
\node at (\xshift+0.4,\yshift-0.2) {\footnotesize $r_{{-}1}$};
\node[text=\red] at (\xshift+0.8,\yshift+\h+0.3) {\footnotesize $\mG_{-1}(y_1)$};
\node at (8.2, 1.8) {$\times$};
\def\xshift{8.6}
\def\yshift{0.4}
\def\h{1.6}
\def\w{1.6}
\fill[\blue] (\xshift,\yshift) rectangle (\xshift+\w,\yshift+\h);
\foreach \x in {0.2,0.6,...,\w} {
  \foreach \y in {0.2,0.6,...,\h} {
    \fill[\red] (\xshift+\x,\yshift+\y) circle (3pt);
    }
  }  
\node at (\xshift-0.3,\yshift+0.2) {\footnotesize $r_0$};
\node at (\xshift+0.3,\yshift-0.2) {\footnotesize $r_1$};
\node[text=\red] at (\xshift+0.8,\yshift+\h+0.3) {\footnotesize $\mG_{1}(\theta_1)$};
\node at (11.1, 1.8) {$\times \cdots \times$};
\def\xshift{12}
\def\yshift{0.4}
\def\h{1.6}
\def\w{2}
\fill[\blue] (\xshift,\yshift) rectangle (\xshift+\w,\yshift+\h);
\foreach \x in {0.2,0.6,...,\w} {
  \foreach \y in {0.2,0.6,...,\h} {
    \fill[\red] (\xshift+\x,\yshift+\y) circle (3pt);
    }
  }  
\node at (\xshift-0.4,\yshift+0.2) {\footnotesize $r_{n-2}$};
\node at (\xshift+0.5,\yshift-0.2) {\footnotesize $r_{n-1}$};
\node[text=\red] at (\xshift+1,\yshift+\h+0.3) {\footnotesize $\mG_{n-1}(\theta_{n-1})$};
\node at (14.4, 1.8) {$\times$};
\def\xshift{14.8}
\def\yshift{0}
\def\h{2}
\def\w{0.4}
\fill[\blue] (\xshift,\yshift) rectangle (\xshift+\w,\yshift+\h);
\foreach \x in {0.2} {
  \foreach \y in {0.2,0.6,...,\h} {
    \fill[\red] (\xshift+\x,\yshift+\y) circle (3pt);
    }
  }  
\node at (\xshift-0.4,\yshift+0.3) {\footnotesize $r_{n-1}$};
\node at (\xshift+0.2,\yshift-0.2) {\footnotesize $r_n$};
\node[text=\red] at (\xshift+0.2,\yshift+\h+0.3) {\footnotesize $\mG_{n}(\theta_{n})$};
\end{tikzpicture}\vspace{-1.5em}
\end{center}
\caption{A schematic of the TT decomposition in \eqref{eq:g_TT}. Note that variables $y_m$ and $\theta_n$ have vector-valued TT cores ($r_{{-}(m{+}1)} {\,=\,} r_n {\,=\,} 1$). Each dot represents a scalar-valued function.}\label{fig:tt}
\end{figure}
We set $r_{{-}(m{+}1)} {\,=\,} r_n {\,=\,} 1$ and use the convention $r_{-1}{\,=\,}r_0$. As shown in Figure \ref{fig:tt}, for a given pair of variables $(y,\theta) = (y_{m}, y_{m-1}, \ldots, y_1, \theta_1, \ldots, \theta_{n-1}, \theta_n)$, one can evaluate each of the matrix-valued TT cores and compute a sequence of matrix-vector products to evaluate $g_{Y,\Theta}(y,\theta)$.

The TT decomposition \eqref{eq:g_TT} can be built by alternating-direction TT-cross approximation methods \cite{goreinov1997pseudo,mahoney2009cur,oseledets2010tt} using $\mathcal{O}(\sum_{k=-m}^{n} r_{k-1}r_k)$ samples of $\pi_{Y,\Theta}(y,\theta)$,
which has complexity scaling linearly in dimension. The procedure is described in \ref{sec:cross}. The TT decomposition \eqref{eq:g_TT} leads to an approximate joint density
\begin{equation}\label{eq:Pitilde_g}
 p_{Y,\Theta}(y,\theta) \propto g_{Y,\Theta}(y,\theta)^2 .
\end{equation}
Approximation of this form offers threefold benefits. Firstly, it preserves non-negativity, which is crucial for building order-preserving maps that are later used for random variable generation and coordinate transformation in our proposed framework. 
As an alternative, one may directly decompose the joint density function $\pi_{Y,\Theta}$ into the TT form, however, the rank adaptation and truncation used in building the TT decomposition can violate the non-negativity of the joint density function. A discrete analogue is that the truncated SVD (or any rank truncated matrix decomposition) of a matrix filled with non-negative entries generally cannot preserve non-negativity. As a result, the conditional density functions built by directly decomposing $\pi_{Y,\Theta}$ may take negative values in some region, and thus the resulting Rosenblatt transport is no longer order-preserving.

Secondly, following \cite{dolgov2014alternating}, we can adapt the TT ranks of $g_{Y,\Theta}$ in the cross approximation to reach a desired error threshold $\varepsilon$ such that $\| \sqrt{\pi_{Y,\Theta}(y,\theta)} - g_{Y,\Theta}(y,\theta) \|_2 \leq \varepsilon/\sqrt{2}$, where $\|\cdot\|_{2}$ denotes the $L^\lowsup{2}$-norm. As shown in \ref{sec:sirt_error}, this function approximation error can be directly translated into the Hellinger error of the approximate joint density \eqref{eq:Pitilde_g}, which satisfies $\DH(\pi_{Y,\Theta},p_{Y,\Theta}) \leq \varepsilon$.

More importantly, the separable form of the TT decomposition permits explicit dimension-by-dimension marginalization of the approximation $p_{Y,\Theta}(y,\theta)$. 
To illustrate the idea, for some index $1 \leq k < n$, we group the TT cores in \eqref{eq:g_TT} into two vector-valued functions
\canceleq{\[
\msout{\mG_{\leq k}(y,\theta_{\leq k}) = \mG_{-m}(y_m)  \hdots   \mG_k(\theta_k) \quad \text{and} \quad \mG_{> k}(\theta_{>k}) = \mG_{k+1}(\theta_{k+1})   \hdots   \mG_n(\theta_n),
}
\]}
\[
\reone{
 \mG_{-m:k}(y,\theta_{1:k}) = \mG_{-m}(y_m)  \hdots  \mG_k(\theta_k) \quad \text{and} \quad \mG_{k+1:n}(\theta_{k+1:n}) = \mG_{k+1}(\theta_{k+1})   \hdots   \mG_n(\theta_n),
}
\]
where the outputs of $\mG_{-m:k}:\R^{m{+}k} \rightarrow \R^{1 \times r_k}$ and $\mG_{k+1:n}:\R^{n{-}k} \rightarrow \R^{r_k \times 1}$ are row vectors and column vectors, respectively. This way, the TT decomposition can be written as
\[
\reone{
g_{Y,\Theta}(y,\theta) = \sum_{\alpha_k = 1}^{r_k} \mG_{-m:k}^{(\alpha_k)}(y,\theta_{1:k}) \, \mG_{k+1:n}^{(\alpha_k)}(\theta_{k+1:m}) ,
}
\]
where \reone{$\mG_{-m:k}^{(\alpha_k)}(y,\theta_{-m:k})$ and $\mG_{k+1:n}^{(\alpha_k)}(\theta_{k+1:n})$} are $\alpha_k$-th elements of the vector-valued functions. Then, each pair of random variables $(Y,\Theta_{1:k})$ has the approximate marginal density
\canceleq{\begin{align*}
  \msout{p_{Y,\Theta_{\leq k}}(y, \theta_{\leq k})} & \msout{\propto \int \left( g_{Y,\Theta}( y, \theta_{\leq k}, \theta_{>k} ) \right)^2\, \d \theta_{>k}}\nonumber \\
  & \msout{= \sum_{\alpha_k = 1}^{r_k} \sum_{\beta_k = 1}^{r_k} \mG_{\leq k}^{(\alpha_k)}(y,\theta_{\leq k}) \, \mG_{\leq k}^{(\beta_k)}(y,\theta_{\leq k})\int \mG_{> k}^{(\alpha_k)}(\theta_{> k}) \mG_{> k}^{(\beta_k)}(\theta_{> k}) \, \d\theta_{>k} }\nonumber \\
  & \msout{= \sum_{\alpha_k = 1}^{r_k} \sum_{\beta_k = 1}^{r_k} \mG_{\leq k}^{(\alpha_k)}(y,\theta_{\leq k}) \, \mG_{\leq k}^{(\beta_k)}(y,\theta_{\leq k})\,\bar\mM^{(\alpha_k,\beta_k)}_{>k} }\nonumber \\
  & \msout{= \sum_{\tau_k = 1}^{r_k} \bigg( \sum_{\alpha_k = 1}^{r_k} \mG_{\leq k}^{(\alpha_k)}(y,\theta_{\leq k}) \bar\chol^{(\alpha_k,\tau_k)}_{>k} \bigg)^2,}
\end{align*}}
\reone{
\begin{align}\label{eq:marginal1}
p_{Y,\Theta_{1:k}}(y, \theta_{1:k}) & \propto \int \left( g_{Y,\Theta}( y, \theta_{1:k}, \theta_{k+1:n} ) \right)^2\, \d \theta_{k+1:n}\nonumber \\
& = \sum_{\alpha_k = 1}^{r_k} \sum_{\beta_k = 1}^{r_k} \mG_{-m:k}^{(\alpha_k)}(y,\theta_{1:k}) \, \mG_{-m:k}^{(\beta_k)}(y,\theta_{1:k})\int \mG_{k+1:n}^{(\alpha_k)}(\theta_{k+1:n}) \mG_{k+1:n}^{(\beta_k)}(\theta_{k+1:n}) \, \d\theta_{k+1:n} \nonumber \\
& = \sum_{\alpha_k = 1}^{r_k} \sum_{\beta_k = 1}^{r_k} \mG_{-m:k}^{(\alpha_k)}(y,\theta_{1:k}) \, \mG_{-m:k}^{(\beta_k)}(y,\theta_{1:k})\,\bar\mM^{(\alpha_k,\beta_k)}_{k+1:n} \nonumber \\
& = \sum_{\tau_k = 1}^{r_k} \bigg( \sum_{\alpha_k = 1}^{r_k} \mG_{-m:k}^{(\alpha_k)}(y,\theta_{1:k}) \bar\chol^{(\alpha_k,\tau_k)}_{k+1:n} \bigg)^2,
\end{align}
where $\bar\chol_{k+1:n}$ is the Cholesky factorization of the symmetric positive definite matrix $\bar\mM_{k+1:n} \in \R^{r_k \times r_k}$, i.e., $\bar\chol_{k+1:n}^{} \bar\chol_{k+1:n}^\top = \bar\mM_{k+1:n}$.}
As detailed in \ref{sec:tt_marginal}, the matrix $\bar\mM_{k+1:n}$ is recursively computed by a sequence of one-dimensional integrals starting with the last variable $\theta_n$. 
Extending this into the observable variables, for $m \geq j \geq 1$, we can similarly group the TT cores as \cancel{$\mG_{\leq -j}(y_{\geq j}) = \mG_{-m}(y_m)  \hdots   \mG_{-j}(y_j)$ and $\mG_{> -j}(y_{<j}, \theta) =  \mG_{-j+1}(y_{j-1}) \hdots   \mG_n(\theta_n)$} \reone{$\mG_{-m:-j}(y_{m:j}) = \mG_{-m}(y_m)  \hdots   \mG_{-j}(y_j)$ and $\mG_{-(j-1):n}(y_{j-1:1}, \theta) =  \mG_{-(j-1)}(y_{j-1}) \hdots   \mG_n(\theta_n)$}, and continue the recursion used in \eqref{eq:marginal1} to compute the matrix 
\canceleq{\[
\msout{
\bar\mM_{>-j}^{(\alpha_{-j}, \beta_{-j})}  = \int \mG_{> k}^{(\alpha_{-j})}(y_{<j}, \theta) \mG_{> -j}^{(\beta_{-j})}(y_{<j}, \theta) \, d y_{<j} \, \d\theta.
}
\]}
\[
\reone{
\bar\mM_{-(j-1):n}^{(\alpha_{-j}, \beta_{-j})}  = \int \mG_{-(j-1):n}^{(\alpha_{-j})}(y_{j-1:1}, \theta) \mG_{-(j-1):n}^{(\beta_{-j})}(y_{j-1:1}, \theta) \, d y_{j-1:1} \, \d\theta.
}
\]
Then, using the Cholesky factorization $\bar\chol_{-(j-1):n}^{} \bar\chol_{-(j-1):n}^\top = \bar\mM_{-(j-1):n}$, the approximate marginal densities for \cancel{$Y_{m}, Y_{\geq m{-}1}, \ldots, Y_{\geq 1}$} \reone{$Y_{m}, Y_{m:m{-}1}, \ldots, Y_{m:1}$} can also be expressed as a sum-of-squares
\canceleq{\begin{equation*}
\msout{
  p_{Y_{\geq j}}(y_{\geq j}) \propto \sum_{\tau_{-j} = 1}^{r_{-j}} \bigg( \sum_{\alpha_{-j} = 1}^{r_{-j}} \mG_{\leq -j}^{(\alpha_{-j})}(y_{\geq j}) \bar\chol^{(\alpha_{-j},\tau_{-j})}_{>-j} \bigg)^2.}
\end{equation*}}
\begin{equation}\label{eq:marginal2}
\reone{
p_{Y_{m:j}}(y_{m:j}) \propto \sum_{\tau_{-j} = 1}^{r_{-j}} \bigg( \sum_{\alpha_{-j} = 1}^{r_{-j}} \mG_{-m:-j}^{(\alpha_{-j})}(y_{m:j}) \bar\chol^{(\alpha_{-j},\tau_{-j})}_{-(j-1):n} \bigg)^2.
}
\end{equation}
Note that after integrating over all dimensions, the resulting scalar \cancel{$M_{\geq -m} \in \R$} \reone{$M_{-m:n} \in \R$} gives the normalizing constant of the approximate joint density $p_{Y,\Theta}$.

\subsection{Squared inverse Rosenblatt transport}

The squared TT decomposition $p_{Y,\Theta}$ constructs the approximate marginal densities of the random variables \cancel{$Y_m, \ldots, Y_{\geq j}, \ldots, (Y, \Theta_{\leq k}), \ldots , (Y, \Theta_{< n})$} \reone{$Y_m, \ldots, Y_{m:j}, \ldots, (Y, \Theta_{1:k}), \ldots , (Y, \Theta_{1:n-1})$}. Using the corresponding conditional distribution functions
\canceleq{\begin{align*}
  \msout{F_{Y_j | Y_{>j}{=}y_{>j}}(y_j|y_{>j})} & \msout{= \int_{-\infty}^{y_j} \frac{p_{Y_{\geq j}}(y_{> j}, y_j')}{p_{Y_{> j}}(y_{> j})} \d y_j',} \\
  \msout{F_{\Theta_k | Y,\Theta_{<k}{=} y,\theta_{<k}}(\theta_k | y, \theta_{<k})} & \msout{= \int_{-\infty}^{\theta_k} \frac{p_{Y,\Theta_{\leq k}}(y, \theta_{< k}, \theta_k')}{p_{Y,\Theta_{< k}}(y, \theta_{< k})} \d \theta_k' ,}
\end{align*}}
\reone{
\begin{align}
 F_{Y_j | Y_{m:j+1}}(y_j|y_{m:j+1}) &= \int_{-\infty}^{y_j} \frac{p_{Y_{m:j}}(y_{m:j+1}, y_j')}{p_{Y_{m:j+1}}(y_{m:j+1})} \d y_j', \label{eq:cdf_sirty}\\
 F_{\Theta_k | Y,\Theta_{1:k-1}}(\theta_k | y, \theta_{1:k-1}) &= \int_{-\infty}^{\theta_k} \frac{p_{Y,\Theta_{1:k}}(y, \theta_{1:k-1}, \theta_k')}{p_{Y,\Theta_{1:k-1}}(y, \theta_{1:k-1})} \d \theta_k' ,\label{eq:cdf_sirtt}
\end{align}
}
we can then follow the definition in \eqref{eq:RosenblattComponents} to build a Rosenblatt transport $\mathcal{F}: \R^\lowsup{m{+}n} \rightarrow [0,1]^\lowsup{m{+}n}$ and its inverse $\mathcal{T} = \mathcal{F}^{-1}$ such that $p_{Y,\Theta} = \mathcal{T}_\sharp \,\mu_{Y,\Theta}$. 
More generally, given an arbitrary product-form probability density function $\rho_{Y,\Theta}(y,\theta) = \prod_{i = 1}^\lowsup{m} \rho_{Y_i}(y_i)\prod_{j = 1}^\lowsup{n} \rho_{\Theta_j}(\theta_j)$, we can define a diagonal map $\mathcal{R}_{Y,\Theta}$ such that $\mathcal{R}_\sharp \,\rho_{Y,\Theta}= \mu_{Y,\Theta}$. Then, the composite map $\mathcal{T}=\mathcal{F}^\lowsup{-1}\circ \mathcal{R}$ also has the lower-triangular structure and satisfies $\mathcal{T}_\sharp\,\rho_{Y,\Theta} = p_{Y,\Theta}$. 
This allows one to transform non-uniform reference random variables and also provides flexibility in the layered construction introduced later on.

Utilizing the transport map $\mathcal{T}$, we can divide the inference into the {\bf offline phase} and the {\bf online phase}.
In the offline phase, we first decompose the square root of the joint density $\sqrt{\pi_{Y,\Theta}}$ into the TT approximation $g_{Y,\Theta}$, and then construct all the conditional distribution functions based on the squared TT decomposition to obtain the inverse Rosenblatt $\mathcal{T}$. 
The offline phase requires the majority of the computational resources, where $\mathcal{O}(\sum_{k=-m}^{n} r_{k-1}r_k)$ of potentially costly joint density evaluations and an additional $\mathcal{O}(\sum_{k=-m}^{n} r_{k-1}^2 r_k + r_{k-1}r_k^2)$ floating point operations (flops) are needed to build the TT decomposition and conditional distribution functions.
See Algorithm \ref{alg:sirt} for a summary of the procedure. 
In the online phase, for observed data $y$, the TT-based conditional density $p_{\Theta|\Yey} = p_{Y,\Theta}(y,\theta) / p_{Y}(y)$ provides a surrogate to the true posterior density $\pi_{\Theta|\Yey}(\theta|y)$. 
We can draw i.i.d. samples from the TT-based surrogate by evaluating 
the conditional map  $\mathcal{T}_{\Theta|Y{=}y}:= \mathcal{F}^\lowsup{-1}_{\Theta|Y}( \cdot | y)$, where each evaluation only costs $\mathcal{O}(\sum_{k=0}^{n} r_{k-1}r_k)$ flops. 

\begin{algorithm}[h]
\caption{Squared inverse Rosenblatt transport (SIRT).}\label{alg:sirt}
\begin{algorithmic}[1]
\Procedure{SIRT}{$\rho_{Y, \Theta}$, $\pi_{Y, \Theta}$, $\varepsilon$}
\State Approximate $\sqrt{\pi_{Y, \Theta}}$ by a TT decomposition $g_{Y, \Theta}$ such that $\|\sqrt{\pi_{Y, \Theta}}-g_{Y, \Theta}\|_{2}\leq \varepsilon / \sqrt{2}$.
\State Compute the conditional distribution functions \eqref{eq:cdf_sirty} and \eqref{eq:cdf_sirtt}. 
\State Assemble the Rosenblatt transport $\mathcal{F}:\R^\lowsup{d}\rightarrow[0,1]^\lowsup{d}$ as in \eqref{eq:RosenblattComponents}.
\State Assemble the diagonal map $\mathcal{R}:\R^\lowsup{d}\rightarrow[0,1]^\lowsup{d}$ such that $\mathcal{R}_\sharp\,\rho_{Y, \Theta}=\mu_{Y, \Theta}$.
\State \Return the lower-triangular map $\mathcal{T} = \mathcal{F}^\lowsup{-1}\circ \mathcal{R}$ that satisfies $\DH(\mathcal{T}_\sharp\,\rho_{Y, \Theta},\pi_{Y, \Theta})\leq\varepsilon$.
\EndProcedure
\end{algorithmic}
\end{algorithm}

It is worth mentioning that TT has also been employed to explore the conditional distribution in alternative ways. In the Bayesian context, the work of \cite{eigel2020low,eigel2018sampling} employs TT to approximate elements of the posterior density, such as the log-likelihood function, to compute posterior statistics. In comparison, our method approximates the square root of the target density using TT, which naturally devising an order-preserving transport map. This becomes a key ingredient for building the composition of maps in Section \ref{sec:DIRT} that is able to characterize concentrated distributions.

\section{Variable reordering and reparametrization}\label{sec:Variable_ordering}

\subsection{Variable reordering}
The efficiency of TT decompositions strongly depends on the variable ordering. Since various subsets of the observable variables may interact differently with various subsets of the parameters in the joint density $\pi_{Y,\Theta}$, we propose to reorder the variables so that the most {\bf interdependent} variables are placed in the middle of the TT decomposition, whereas the rest of {\bf mutually independent} variables are placed in the tail coordinates of TT. The rationale is the following: for some indices $s < m$ and $t < n$, if the variables $Y_{m},\hdots,Y_{s+1}, \Theta_{t+1},\hdots,\Theta_n$ and the middle block of variables \cancel{$(Y_{\leq s}, \Theta_{\leq t})$} \reone{$(Y_{s:1}, \Theta_{1:t})$} are mutually independent, then the joint density $\pi_{Y,\Theta}$ can be equivalently expressed in a product-form of
\begin{equation}\label{eq:pi_with_independence_structure}
 \widetilde\pi_{Y,\Theta}(y,\theta) = 
 \bigg(\prod_{i=s{+}1}^{m} \widetilde\pi_{Y_i}(y_i)  \bigg)
 \reone{\widetilde\pi_{Y_{s:1},\Theta_{1:t}}\left(y_{s:1},\theta_{1:t}\right)}
 \bigg(\prod_{j=t+1}^{n} \widetilde\pi_{\Theta_j}(\theta_j) \bigg) .
\end{equation}
Such a density can be approximated using a tensor train $g_{Y,\Theta}\approx\sqrt{\widetilde\pi_{Y,\Theta}}$ with truncated ranks $(1,\hdots,1,r_{{-}(s{+}1)},\hdots,r_t,1,\hdots,1)$. In practice, however, $\pi_{Y,\Theta}$ might not have exactly the mutual independence structure as in \eqref{eq:pi_with_independence_structure}. Instead, after some variable reordering or reparametrization, it can have a {\bf low-dependency structure} in the sense that $\sqrt{\pi_{Y,\Theta}}$ can be accurately approximated by a TT $g_{Y,\Theta}$, where $\text{rank}_\text{TT}(g_{Y,\Theta}) =(r_{-m},\hdots,r_i,\hdots,r_n)$ quickly decrease to 1 in the left tail coordinates ($i\rightarrow -m$) and in the right tail coordinates ($i\rightarrow n$).

We propose to exploit the derivative information of the joint density $\pi_{Y,\Theta}$ to detect its low-dependency structure and to reorder/reparametrize the variables. As a baseline to detect the desired low-dependency structure in \eqref{eq:pi_with_independence_structure}, we define a reference density 
\begin{equation}\label{eq:rho}
 \rho_{Y,\Theta}(y,\theta) = \rho_{Y}(y) \rho_{\Theta}(\theta), \;\; {\rm where} \;\;
 \rho_{Y}(y) = \prod_{i=1}^{m} \pi_{Y_i|\Theta=\theta_0}(y_i) \;\; {\rm and} \;\; \rho_{\Theta}(\theta) = \prod_{j=1}^{n} \pi_{\Theta_j}(\theta_j), 
\end{equation}
where $\theta_0\in\R^\lowsup{n}$ is an anchor variable. In our analysis, any product-form density $\rho_{Y,\Theta}(y,\theta)=\rho_{Y_m}(y_m)\hdots \rho_{\Theta_n}(\theta_n)$ can be used. The construction in \eqref{eq:rho} is a natural choice in the Bayesian context, as the likelihood function conditioned on the anchor variable and the prior density are often known analytically. 
Then, we measure the dependency among various coordinates of $(Y,\Theta)$ using the sensitivity matrices 
\begin{align}
 H_Y      &= \int \bigg(\nabla_y      \log\frac{\pi_{Y,\Theta}}{\rho_{Y,\Theta}}   \bigg)\bigg(  \nabla_y  \log\frac{\pi_{Y,\Theta}}{\rho_{Y,\Theta}} \bigg)^\top \d\pi_{Y,\Theta}    ,\label{eq:defHY} \\
 H_\Theta &= \int \bigg(\nabla_\theta \log\frac{\pi_{Y,\Theta}}{\rho_{Y,\Theta}}   \bigg)\bigg(  \nabla_\theta  \log\frac{\pi_{Y,\Theta}}{\rho_{Y,\Theta}} \bigg)^\top \d\pi_{Y,\Theta}, \label{eq:defHTheta}
\end{align}
which bounds the statistical divergence of $\pi_{Y,\Theta}$ from the product-form $\rho_{Y,\Theta}(y,\theta)$ that represents independence. Note that $\trace(H_\Theta){+}\trace(H_Y)=\E_{\pi_{Y,\Theta}}[\| \nabla\log\pi_{Y,\Theta} - \nabla\log\rho_{Y,\Theta} \|^\lowsup{2}]$ is known as the relative Fisher information and can be used to bound the KL divergence $\DKL(\pi_{Y,\Theta}||\rho_{Y,\Theta})$, see \cite{otto2000generalization}. The following proposition offers a viable path to reordering variables according to the diagonals of $H_\Theta$ and $H_Y$. 

\begin{proposition}\label{prop:bound_Hellinger}
For any truncated data rank $s \leq m$ and truncated parameter rank $t \leq n$, the lower-dimensional density function \reone{$\widetilde \pi_{Y_{s:1},\Theta_{1:t}} : \R^s\times \R^t \rightarrow \R$} given by
\canceleq{\begin{align*}
\msout{
  \widetilde \pi_{Y_{\leq s},\Theta_{\leq t}}(y_{\leq s},\theta_{\leq t})
  \propto \left(\int \left(\frac{\pi_{Y,\Theta}(y,\theta)}{\rho_{Y,\Theta}(y,\theta)}\right)^\frac12  \rho_{Y_{> s},\Theta_{> t}}(y_{> s},\theta_{> t}) \d y_{>s}\d\theta_{>t} \right)^2 \rho_{Y_{\leq s},\Theta_{\leq t}}(y_{\leq s},\theta_{\leq t}) 
}
\end{align*}}
\reone{
 \begin{align}
  &\widetilde \pi_{Y_{s:1},\Theta_{1:t}}(y_{s:1},\theta_{1:t}) \nonumber \\
  &\qquad \propto \left(\int \left(\frac{\pi_{Y,\Theta}(y,\theta)}{\rho_{Y,\Theta}(y,\theta)}\right)^\frac12  \rho_{Y_{m:s+1},\Theta_{t+1,n}}(y_{m:s+1},\theta_{t+1:n}) \d y_{m:s+1}\d\theta_{t+1:n} \right)^2 \rho_{Y_{s:1},\Theta_{1:t}}(y_{s:1},\theta_{1:t}) \nonumber 
 \end{align}
 }
 is a minimizer of $\widetilde \pi_{Y_{s:1},\Theta_{1:t}} \mapsto \DH(\pi_{Y,\Theta},\widetilde\pi_{Y,\Theta})$, where 
 \begin{equation}\label{eq:api_with_independence_structure}
\widetilde\pi_{Y,\Theta}(y,\theta) = \bigg( \prod_{i=s+1}^{m} \rho_{Y_i}(y_i) \bigg) \reone{\widetilde\pi_{Y_{s:1},\Theta_{1:t}}\left(y_{s:1},\theta_{1:t}\right)}
 \bigg(\prod_{j=t+1}^{n} \rho_{\Theta_j}(\theta_j) \bigg).
\end{equation}
 Furthermore, if we assume that $\rho_{Y,\Theta}$ satisfies the Poincaré inequality, i.e., there exists a constant $\kappa<\infty$ so that  $\Var_{\rho_{Y,\Theta}}(h) \leq \kappa \, \E_{\rho_{Y,\Theta}}(\|\nabla h\|_2^\lowsup{2})$ holds for any sufficiently smooth function $h:\R^\lowsup{n{+}m}\rightarrow\R$, then we have
 \begin{equation}\label{eq:bound_Hellinger}
  \DH(\pi_{Y,\Theta},\widetilde\pi_{Y,\Theta})^2
  \leq \frac{\kappa}{4} \bigg(\sum_{i=s+1}^{m} (H_Y)_{ii} + \sum_{j=t+1}^{n} (H_\Theta)_{jj} \bigg). 
 \end{equation}
\end{proposition}

The proof is given in Section \ref{sec:proofs}. This proposition suggests to sort the variables $\theta_{1},\hdots,\theta_{n}$ and $y_{1},\hdots,y_{m}$ such that the diagonals of the sensitivity matrices, $i\mapsto(H_Y)_{ii}$ and $j\mapsto(H_\Theta)_{ii}$, are decreasing. This way, the right-hand-side of \eqref{eq:bound_Hellinger} is minimized for any $s$ and $t$.
This proposition also suggests that one only needs to build the TT decomposition of \reone{$\widetilde\pi_{Y_{s:1},\Theta_{1:t}}$} after truncating the variables $Y$ and $\Theta$ to the first $s$ and $t$ components, while controlling the resulting truncation error by \eqref{eq:bound_Hellinger}.

\subsection{Variable reparametrization}

In addition to variable reordering, we can also reparametrize the variables $y$ and $\theta$ to further reduce the interdependence among $(Y,\Theta)$.
To preserve the product-form structure of the reference density in \eqref{eq:rho} after reparametrization, we restrict the analysis to Gaussian prior and Gaussian observation error. %

\begin{proposition}\label{prop:bound_Hellinger_rotation}
Assume the joint density takes the form 
 \begin{equation}\label{eq:PiGaussian}
  \pi_{Y,\Theta}(y,\theta)\propto \exp \Big( -\frac12\|y-G(\theta)\|_2^2 - \frac12\|\theta-\theta_0\|_2^2\Big) , 
 \end{equation}
 where $\theta\mapsto G(\theta)\in\R^\lowsup{m}$ is a forward operator mapping the parameter $\theta$ to observables and $\|\cdot\|_2$ is the Euclidean norm.
The matrices $H_Y$ and $H_\Theta$ are given by
\begin{equation}\label{eq:HwithGaussianJoint}
 H_Y \! = \! \int \!\! \big( G(\theta) - G(\theta_0)\big)\big( G(\theta) - G(\theta_0)\big)^\top \!\! \d \pi_\Theta \quad\text{and}\quad H_\Theta \! = \! \int \! \nabla G(\theta)^\top \nabla G(\theta)  \d \pi_{\Theta},
\end{equation}
respectively. Let $A\in\R^\lowsup{m{\times} m} $ and $B\in\R^\lowsup{n{\times} n}$ be unitary matrices containing the eigenvectors corresponding to decreasing eigenvalues of $H_Y$ and $H_\Theta$, respectively.
 Then there exists a density function $\widetilde \pi_{Y,\Theta}^\lowsup{A,B}(y,\theta) = \widetilde \pi_{Y,\Theta}(A^\lowsup{\top}y,B^\lowsup{\top}\theta ),$ with $\widetilde\pi_{Y,\Theta}$ as in \eqref{eq:api_with_independence_structure} such that
 \begin{equation}\label{eq:bound_Hellinger_rotation}
  \DH(\pi_{Y,\Theta},\widetilde\pi_{Y,\Theta}^\lowsup{A,B})^2
  \leq \frac14 \bigg(\sum_{i=s+1}^{m} \lambda_i(H_Y)+ \sum_{i=t+1}^{n} \lambda_j(H_\Theta)  \bigg) , 
 \end{equation}
 holds for any truncated data rank $s\leq m$ and truncated parameter rank $t\leq n$, where $\lambda_i(\cdot)$ denotes the $i$-th largest eigenvalue.
\end{proposition}

The proof is given in Section \ref{sec:proofs}.
Proposition \ref{prop:bound_Hellinger_rotation} provides an intuitive guideline that we can reparametrize $Y$ using the leading principal components of the model outputs $G(\Theta)$  and reparametrize $\Theta$ using the directions in which the linearized model $\nabla G(\Theta)$ varies the most. Proposition \ref{prop:bound_Hellinger_rotation} also certifies the dimension truncation after  reparametrization (by choosing $s{<}m$ and $t{<}n$) using the eigenvalues of the sensitivity matrices $H_Y$ and $H_\Theta$.
Although the result of Proposition \ref{prop:bound_Hellinger_rotation} is based on Gaussian assumptions, for problems equipped with non-Gaussian priors and non-Gaussian observation errors, there may exist analytical transformations that can reparametrize the joint density into the form of \eqref{eq:PiGaussian} before applying Proposition \ref{prop:bound_Hellinger_rotation}. 


\begin{remark}[Whitening transform]
For problems which are not equipped with Euclidean norms as in \eqref{eq:PiGaussian}, one can apply whitening transforms in order to apply Proposition \eqref{prop:bound_Hellinger_rotation}.
As a starting point, we consider a general parameter $\Xi \in \R^{n}$ equipped with a Gaussian prior $\mathcal{N}(\xi_0, \Gamma)$, where $\xi_0 \in \R^{n}$ and $\Gamma \in \R^{ n \times n }$ are the mean vector and covariance matrix, respectively. Given a smooth parameter-to-observable map $\xi\mapsto \Phi(\xi)\in\R^\lowsup{m}$, we assume the observable random variable $Z$ follows $\mathcal{N}(\Phi(\xi), \Sigma)$. Then the joint density takes the form
\begin{equation}\label{eq:PiGaussianColored}
  \pi_{Z,\Xi}(z,\xi)\propto \exp \Big( -\frac12\|z-\Phi(\xi)\|_{\Sigma}^2 - \frac12\|\xi-\xi_0\|_{\Gamma}^2\Big) ,
\end{equation}
where $\| v \|_\Gamma = \sqrt{v^\lowsup{\top} \Gamma^\lowsup{-1} v}$ is the matrix weighted norm.
Then, the whitening transforms
\[
Y = \Sigma^{-\frac12} Z, \quad {\rm and} \quad \Theta = \Gamma^{-\frac12} \Xi,
\]
permits one to obtain a joint density as in \eqref{eq:PiGaussian} where $G(\theta) := \Sigma^{-\frac12}\Phi(\Gamma^{\frac12}\theta)$ and $\theta_0 = \Gamma^{-\frac12} \xi_0$.
This change of variable makes Proposition \eqref{prop:bound_Hellinger_rotation} applicable to joint densities in the form of \eqref{eq:PiGaussianColored}.

\end{remark}

\subsection{Proofs of Proposition \ref{prop:bound_Hellinger} and Proposition \ref{prop:bound_Hellinger_rotation}}\label{sec:proofs}
\retwo{
The proofs of Proposition \ref{prop:bound_Hellinger} and Proposition \ref{prop:bound_Hellinger_rotation} rely on the following lemmas which are given for a general probability densities on $\R^d$. To denote operations such as marginalization, we define an index set $\gamma \subseteq [d]:=\{1, 2, \ldots, d\}$ and let $\gamma^c = [d] \setminus \gamma$. For a subset of parameters indexed by $\gamma$, we defined it as $x_\gamma = (x_{\gamma_1}, \ldots, x_{\gamma_{|\gamma|}})$, where $|\gamma|$ is the cardinality of $\gamma$.

\begin{lemma}\label{prop:defG_optimal}
Let $\pi$ and $\rho$ be probability densities on $\R^d$. For any index set $\gamma \subseteq [d]$, let $\widetilde\pi^*$ be the probability density on $\R^d$ defined by $\widetilde\pi^*(x) \propto g(x_{\gamma})^\lowsup{2}\rho(x)$, where $g:\R^{|\gamma|}\rightarrow\R_{\geq0}$ is given by
\begin{equation}\label{eq:defG_optimal}
    g(x_{\gamma}) = \int \left( \frac{\pi(x_{\gamma},x_{\gamma^c})}{\rho(x_{\gamma},x_{\gamma^c})}\right)^\frac12\rho_{\gamma^c|\gamma}(x_{\gamma^c}|x_{\gamma})\,\d x_{\gamma^c}.
\end{equation}
Here, $\rho_{\gamma^c|\gamma}(x_{\gamma^c}|x_{\gamma}) = \rho(x_{\gamma},x_{\gamma^c}) / \rho_{\gamma}(x_{\gamma})$ is a conditional probability density.
For any function $\widetilde g:\R^{|\gamma|}\rightarrow\R_{\geq0}$ such that $\widetilde\pi(x) \propto \widetilde g(x_{\gamma})^\lowsup{2}\rho(x)$ is a probability density, we have
\begin{equation}\label{eq:Pythagore}
    \DH(\pi,\widetilde\pi)^2 = \DH(\pi,\widetilde\pi^*)^2 + \|g\sqrt{\rho(x)}\| \DH(\widetilde\pi^*,\widetilde\pi)^2 .
\end{equation}
In particular $\DH(\pi,\widetilde\pi^*) \leq \DH(\pi,\widetilde\pi)$ holds for any $\widetilde g$ and equality is attained when $\DH(\widetilde\pi^*,\widetilde\pi)=0$, meaning that $\widetilde g=g$ in the $L^2$-sense.
\end{lemma}

\begin{proof}
 See \ref{proof:defG_optimal}.
\end{proof}

\begin{lemma}\label{prop:bound_Hellinger_Poincare}
 Let $\rho(x) = \rho_1(x_1)\hdots\rho_d(x_d) $ be a product-form probability density on $\R^d$ that satisfies the Poincaré inequality, that is, there exists a constant $\kappa<\infty$ such that
 \begin{equation}\label{eq:PoincareRho}
   \int \Big(h(x) - \int h(x) \rho(x)\d x \Big)^2 \rho(x)\d x \leq \kappa \int \| \nabla h(x) \|^2 \rho(x)\d x,
 \end{equation}
 holds for any sufficiently smooth function $h:\R^d\rightarrow\R$, where $\|\cdot\|$ denotes the euclidean norm.
 Let $\pi$ be a probability density on $\R^d$. Given some reduced dimensional index set $\gamma\in[d]$, consider the probability density $\widetilde\pi^*$ on $\R^d$ defined by
 $\widetilde\pi^*(x) \propto g(x_{\gamma})^2  \rho(x)$ where $g:\R^{|\gamma|}\rightarrow\R_{\geq0}$ is defined by
 \begin{equation}\label{eq:defG_proof}
  g(x_{\gamma}) = \int \left(\frac{\pi(x_{\gamma},x_{\gamma^c})}{\rho(x_{\gamma},x_{\gamma^c})}\right)^\frac12\rho_{\gamma^c}(x_{\gamma^c})\d x_{\gamma^c} .
 \end{equation}
 Then we have
 \begin{equation}\label{eq:bound_Hellinger_BIS}
  \DH(\pi,\widetilde\pi^*)^2
  \leq \frac{\kappa}{4} \sum_{i=s+1}^{d} H_{ii} ,
 \end{equation}
 where $H = \int \big(\nabla\log\frac{\pi(x)}{\rho(x)}\big)\big(\nabla\log\frac{\pi(x)}{\rho(x)}\big)^\top \pi(x) \d x$.

\end{lemma}

\begin{proof}
 See \ref{proof:bound_Hellinger_Poincare}.
\end{proof}

A straightforward application of the above lemmas leads to the proof of Proposition \ref{prop:bound_Hellinger}.

\begin{proof}[Proof of Proposition \ref{prop:bound_Hellinger}]
Recall that the data $y$ and the parameter $\theta$ have dimensions $m$ and $n$, respectively. We denote the reduced dimensions of data and parameter by $s$ and $t$, respectively. Proposition \ref{prop:bound_Hellinger} is a direct application of Lemmas \ref{prop:defG_optimal} and \ref{prop:bound_Hellinger_Poincare} by setting
\begin{align*}
 x &= (y_{m}, y_{m-1}, \ldots, y_1, \theta_1, \ldots, \theta_{n-1}, \theta_n)\in\R^d, &  d &=m + n, \\
 x_{\gamma} &= (y_{s},y_1,\hdots,\theta_1,\hdots,\theta_{t})\in\R^\lowsup{|\gamma|}, & \gamma &= \{m{-}s{+}1, \ldots, m, m{+}1, \ldots, m{+}t\} ,
\end{align*}
and defining the probability densities \(\pi(x) = \pi_{Y,\Theta}(y, \theta),\) and\(\rho(x) = \rho_{Y,\Theta}(y,\theta).\) 
\end{proof}

We now give the proof for Proposition \ref{prop:bound_Hellinger_rotation}.

\begin{proof}[Proof of Proposition \ref{prop:bound_Hellinger_rotation}]
First, we show that the joint density \eqref{eq:PiGaussian} yields the sensitivity matrices as in \eqref{eq:HwithGaussianJoint}.
Because the joint density $\pi_{Y,\Theta}$ and the reference density $\rho_{Y,\Theta}$ as in \eqref{eq:rho} write
\begin{align*}
 \pi_{Y,\Theta}(y,\theta)&\propto \exp \Big( -\frac12\|y-G(\theta)\|_2^2 - \frac12\|\theta-\theta_0\|_2^2\Big)  \\
 \rho_{Y,\Theta}(y,\theta)&\propto \exp \Big( -\frac12\|y-G(\theta_0)\|_2^2 - \frac12\|\theta-\theta_0\|_2^2\Big) ,
\end{align*}
we obtain
\[
 \log \left(\frac{\pi_{Y,\Theta}}{\rho_{Y,\Theta}}(y, \theta)\right) = C + \frac12 \left\| y - G(\theta_0) \right\|^2_2  - \frac12 \left\| y - G(\theta) \right\|^2_2 ,
\]
for some constant $C$, so that
\[
\nabla_y\log \left(\frac{\pi_{Y,\Theta}}{\rho_{Y,\Theta}}(y, \theta)\right) = G(\theta) - G(\theta_0) ,\quad {\rm and} \quad \nabla_\theta  \log \left(\frac{\pi_{Y,\Theta}}{\rho_{Y,\Theta}}(y, \theta)\right) = \nabla G(\theta)^\top \left( y - G(\theta) \right),
\]
where $\nabla G(\theta) \in \R^{m \times n} $ is the Jacobian of the forward model at $\theta$.
Then, the matrix $H_Y$ defined in \eqref{eq:defHY} can be expressed as
\begin{align*}
H_Y
& = \int \big( G(\theta) - G(\theta_0)\big)\big(G(\theta) - G(\theta_0)\big)^\top \pi_{Y,\Theta}(y, \theta) \d y \d \theta \\
& = \int \big( G(\theta) - G(\theta_0)\big)\big(G(\theta) - G(\theta_0)\big)^\top \bigg(\int  \pi_{Y|\Theta}(y|\theta) \d y \bigg) \pi_{\Theta}(\theta) \d \theta \\
& = \int \big( G(\theta) - G(\theta_0)\big)\big(G(\theta) - G(\theta_0)\big)^\top \d \pi_{\Theta} .
\end{align*}
The matrix $H_\Theta$ defined in \eqref{eq:defHTheta} as
\begin{align*}
H_\Theta
& = \int  \nabla G(\theta)^\top  \left( y - G(\theta) \right) \left( y - G(\theta) \right)^\top  \nabla G(\theta) \pi_{Y,\Theta}(y, \theta) \d y \d \theta, \\
& = \int \nabla G(\theta)^\top \left( \int \left( y - G(\theta) \right) \left( y - G(\theta) \right)^\top \pi_{Y|\Theta}(y|\theta) \d y \right) \nabla G(\theta)  \pi_{\Theta}(\theta)\d \theta \\
&= \int \nabla G(\theta)^\top \nabla G(\theta)  \d \pi_{\Theta},
\end{align*}
where, for the last equality, we used the fact that $Y|\Theta=\theta=\theta \sim\mathcal{N}(G(\theta),I_d)$ has, by assumption \eqref{eq:PiGaussian}, an identity covariance matrix.

The rest of the proof consists essentially in applying Proposition \ref{prop:bound_Hellinger} after the change of variables
\begin{align*}
 \bar y = A^\top y\quad {\rm and } \quad \bar\theta = B^\top\theta,
\end{align*}
where $A\in\R^\lowsup{m\times m}$ and $B\in\R^\lowsup{n\times n}$ are the unitary matrices containing the eigenvectors (with corresponding eigenvalues sorted in the decreasing order) of $H_Y$ and $H_\Theta$ respectively.
Because $\rho_{Y,\Theta}$ is Gaussian with identity covariance, for $(U_\Theta,U_Y)\sim\rho_{Y,\Theta}$, the reparametrized random variable $(A^\lowsup{\top}U_Y, B^\lowsup{\top}U_\Theta)$ is also Gaussian with identity covariance.
Since any Gaussian density with identity covariance matrix satisfies Poincar\'{e} inequality with $\kappa=1$, see \cite{chernoff1981note}, Proposition \ref{prop:bound_Hellinger} ensures there exists a probability density $\widetilde\pi_{\bar Y, \bar \Theta}(\bar y, \bar \theta)$ as in \eqref{eq:pi_with_independence_structure} such that
\begin{equation}\label{eq:peitq}
 \DH\left(\pi_{\bar Y, \bar \Theta} \,,\, \widetilde\pi_{\bar Y, \bar \Theta}\right)^2 \leq \frac14 \left(\sum_{i=s+1}^{m} (H_Y^{A})_{ii} + \sum_{j=t+1}^{n} (H_\Theta^B)_{jj} \right),
\end{equation}
where $\pi_{\bar Y, \bar \Theta}$ is the probability density of $(A^\lowsup{\top}Y,B^\lowsup{\top}\Theta)$
and where $H_Y^\lowsup{A}$ and $H_\Theta^\lowsup{B}$ are given by
\begin{align*}
 H_Y^{A}       &= \int \left(\nabla_y      \log\frac{\pi_{\bar Y, \bar \Theta}}{\rho_{\bar Y, \bar \Theta}}   \right)\left(  \nabla_y  \log\frac{\pi_{\bar Y, \bar \Theta}}{\rho_{\bar Y, \bar \Theta}} \right)^\top\d\pi_{\bar Y, \bar \Theta} ,  \\
 H_\Theta^{B} &= \int \left(\nabla_\theta \log\frac{\pi_{\bar Y, \bar \Theta}}{\rho_{\bar Y, \bar \Theta}}   \right)\left(  \nabla_\theta  \log\frac{\pi_{\bar Y, \bar \Theta}}{\rho_{\bar Y, \bar \Theta}} \right)^\top \d\pi_{\bar Y, \bar \Theta}  .
\end{align*}
Applying the chain rule, we have $H_Y^{A} = A^\top H_Y A$ and $H_\Theta^{B} = B^\top H_\Theta B$ and, because $A$ and $B$ contain the eigenvectors of $H_Y$ and $H_\Theta$, the inequality in \eqref{eq:peitq} becomes
$$
 \DH\left(\pi_{\bar Y, \bar \Theta} \,,\, \widetilde\pi_{\bar Y, \bar \Theta}\right)^2 \leq \frac14 \left(\sum_{i=s+1}^{m} \lambda_i(H_Y)+ \sum_{j=t+1}^{n} \lambda_j(H_\Theta)  \right) .
$$
Finally, since the change of variables $\bar y = A^\top y$ and $\bar\theta = B^\top\theta$ are isometries, we have
$$
 \DH\left(\pi_{\bar Y, \bar \Theta} \,,\, \widetilde\pi_{\bar Y, \bar \Theta}\right)^2
 = \DH(\pi_{Y,\Theta},\widetilde\pi_{Y,\Theta}^{A,B})^2,
$$
where $\widetilde\pi_{Y,\Theta}^{A,B}(y,\theta) := \widetilde\pi_{Y,\Theta}(A^\lowsup{\top}y,B^\lowsup{\top}\Theta)$.
This concludes the proof.
\end{proof}
}

\section{Deep inverse Rosenblatt transport}\label{sec:DIRT}

For target densities that are concentrated to a small subdomain or have complicated correlation structures, it can be challenging to construct in one step a TT approximation to $\sqrt{\pi_{Y,\Theta}}$. DIRT alleviates this difficulty by building a composition of transport maps 
$$
 \mathcal{T}_L = \mathcal{Q}_0 \circ \mathcal{Q}_1 \circ \hdots\circ \mathcal{Q}_L,
$$
guided by a sequence of bridging densities
\[
 \pi_{Y,\Theta}^\lowsup{0} ,\, \pi_{Y,\Theta}^\lowsup{1} ,\,\hdots,\, \pi_{Y,\Theta}^\lowsup{L}\coloneqq \pi_{Y,\Theta},
\]
with increasing complexity. 
For example, using the tempering technique \cite[see, e.g.,][]{gelman1998simulating,neal1996sampling,swendsen1986replica}, one may consider the bridging densities $\pi_{Y,\Theta}^\lowsup{\ell} \propto \varphi_{Y,\Theta}^\lowsup{\beta_\ell}\, \rho_{Y,\Theta}$, where $\varphi_{Y,\Theta} = \pi_{Y,\Theta} / \rho_{Y,\Theta}$, $0 = \beta_0 \leq \beta_1\leq \cdots\leq \beta_L=1$ are some chosen temperatures, and $\rho_{Y,\Theta}$ as in \eqref{eq:rho}. 

We propose to build the layers $\mathcal{Q}_0 ,\hdots, \mathcal{Q}_L$ in a greedy way. For each $0\leq\ell\leq L$, the composed map $\mathcal{T}_{\ell}=\mathcal{Q}_0 \circ \hdots\circ \mathcal{Q}_\ell$ yields an approximation $p_{Y,\Theta}^\lowsup{\ell} := (\mathcal{T}_{\ell})_\sharp \rho_{Y,\Theta}$ to the bridging density $\pi_{Y,\Theta}^\lowsup{\ell}$. 
At each iteration, a new layer of map $\mathcal{Q}_{\ell{+}1}$ is constructed as a SIRT so that 
the new pushforward density $p_{Y,\Theta}^\lowsup{\ell{+}1} = (\mathcal{T}_\ell \circ \mathcal{Q}_{\ell{+}1})_\sharp \rho_{Y,\Theta}^\lowsup{\ell}$ approximates the next bridging density $\pi_{Y,\Theta}^\lowsup{\ell{+}1}$ in a way that
\[
\DH( p_{Y,\Theta}^\lowsup{\ell{+}1} , \pi_{Y,\Theta}^\lowsup{\ell{+}1}  ) \leq \omega \, \DH (  p_{Y,\Theta}^\lowsup{\ell},  \pi_{Y,\Theta}^\lowsup{\ell{+}1} ),
\] 
for some relative error bound $\omega <1$. This is equivalent to
\begin{equation}\label{eq:Greedy_DIRT}
 \DH\big( (\mathcal{Q}_{\ell{+}1})_\sharp \rho_{Y,\Theta} , \mathcal{T}_{\ell}^\sharp \pi_{Y,\Theta}^{\ell{+}1} \big)
 \leq \omega \, \DH\big( \rho_{Y,\Theta} , \mathcal{T}_{\ell}^\sharp \pi_{Y,\Theta}^{\ell{+}1} \big).
\end{equation}
Thus, in each layer, we can first decompose the square root of the pullback of the next bridging density, $\sqrt{\mathcal{T}_{\ell}^\sharp \pi_{Y,\Theta}^{\ell{+}1}}$, in the TT format, and then build the next map $\mathcal{Q}_{\ell{+}1}$ as an SIRT. The following proposition shows the convergence of the greedy construction of DIRT. 

\begin{proposition}\label{prop:DIRT_CV}
Assume there exist constants $\eta(L) < 1$ and $\omega<1$ such that the bridging densities satisfy
\begin{equation}\label{eq:BoundOnBridgingPdfs}
   \sup_{0 \leq \ell < L} \DH\big( \pi_{Y,\Theta}^{\ell} ,  \pi_{Y,\Theta}^{\ell+1} \big) = \eta(L),
\end{equation}
the initial map $\mathcal{Q}_0$ satisfies $\DH( (\mathcal{Q}_0)_\sharp \rho_{Y,\Theta} , \pi_{Y,\Theta}^{0} ) \leq \omega \eta(L)$, and the maps $\mathcal{Q}_{1},\hdots,\mathcal{Q}_{L}$ satisfy \eqref{eq:Greedy_DIRT} for any $\ell$.
Then the resulting composite map $\mathcal{T}_L=\mathcal{Q}_{0}\circ\mathcal{Q}_{1}\circ\hdots\circ\mathcal{Q}_{L}$ defines a pushforward density $p_{Y,\Theta} = (\mathcal{T}_L)_\sharp \rho_{Y,\Theta} $ that satisfies
\begin{equation}\label{eq:DIRT_CV}
  \DH\big( p_{Y,\Theta} , \pi_{Y,\Theta} \big) \leq \frac{\omega}{1-\omega} \, \eta(L).
\end{equation}
\end{proposition}

The proof is given in \ref{proof:DIRT_CV}. Under mild assumptions, \ref{sec:dirt_fix} shows that tempered bridging densities $\pi_{Y,\Theta}^\lowsup{\ell} \propto \varphi_{Y,\Theta}^\lowsup{\beta_\ell}\, \rho_{Y,\Theta}$ with uniformly spaced temperatures $\beta_\ell=\ell/L$ satisfy \eqref{eq:BoundOnBridgingPdfs} with $\eta(L)=\mathcal{O}(1/L)$. 
Algorithm \ref{alg:dirt} details the construction of DIRT with a controlled error, in which each layer of the composition $\mathcal{Q}_{\ell{+}1}$ is constructed using the SIRT Algorithm \ref{alg:sirt}. The resulting map $\mathcal{T}_L$ is lower-triangular as a composition of lower-triangular maps.

\begin{algorithm}[h]
\caption{Deep Inverse Rosenblatt Transform (DIRT).}\label{alg:dirt}
\begin{algorithmic}[1]
\Procedure{DIRT}{$\rho_{Y, \Theta}$, $\{\pi_{Y,\Theta}^\lowsup{\ell}\}_{\ell=1}^\lowsup{L}$, $\varepsilon$} 
\Comment{$\varepsilon \leq \omega\eta(L)$}
\State $\mathcal{T}_{0} \gets$ \Call{SIRT}{$\rho_{Y,\Theta}$, $\pi_{Y,\Theta}^\lowsup{0}$, $\varepsilon$} \Comment{$\mathcal{Q}_{0}\equiv\mathcal{T}_{0}$.}
\For{$\ell = 0, \ldots, L-1$}
\State $\mathcal{Q}_{\ell+1} \gets$ \Call{SIRT}{$\rho_{Y,\Theta}$, $\mathcal{T}_{\ell}^\lowsup{\sharp}\pi_{Y,\Theta}^\lowsup{\ell+1}$, $\varepsilon$}
\State $\mathcal{T}_{\ell+1} \gets \mathcal{T}_{\ell}\circ \mathcal{Q}_{\ell+1}$
\EndFor
\State \Return $\mathcal{T}_L$
\EndProcedure
\end{algorithmic}
\end{algorithm}

\section{Online inference}\label{sec:online}
In the offline phase, we apply the variable reordering/reparametrization discussed in Section \ref{sec:Variable_ordering}, followed by DIRT in Section \ref{sec:DIRT} to build the triangular map 
\[
  \left[ \begin{array}{l} y \\ \theta \end{array}\right] = \mathcal{T}(u_Y, u_\Theta) = \left[ \begin{array}{l}\mathcal{T}_Y(u_Y) \\ \mathcal{T}_{\Theta}(u_Y, u_\Theta) \end{array}\right],
\]
where
\canceleq{\[
\msout{
\mathcal{T}_Y(u_Y) = \left[ \begin{array}{r} u_{Y_{>s}} \phantom{)} \\ \mathcal{T}_{Y_{\leq s}}(u_{Y_{\leq s}})\end{array}\right] \quad \text{and} \quad \mathcal{T}_\Theta(u_\Theta) = \left[ \begin{array}{r} \mathcal{T}_{\Theta_{\leq t}}(u_{Y_{\leq s}}, u_{\Theta_{\leq t}}) \\ u_{\Theta_{>t}} \phantom{)} \end{array}\right] ,}
\]}
\[
\reone{
\mathcal{T}_Y(u_Y) = \left[ \begin{array}{r} u_{Y_{m:s+1}} \phantom{)} \\ \mathcal{T}_{Y_{s:1}}(u_{Y_{s:1}})\end{array}\right] \quad \text{and} \quad \mathcal{T}_\Theta(u_Y,u_\Theta) = \left[ \begin{array}{r} \mathcal{T}_{\Theta_{1:t}}(u_{Y_{s:1}}, u_{\Theta_{1:t}}) \\ u_{\Theta_{t+1:n}} \phantom{)} \end{array}\right] ,
}
\]
so that $\mathcal{T}_\sharp \rho_{Y,\Theta}$ approximates the joint density $\pi_{Y,\Theta}$ with a controlled error. In the online phase, the pushforward of $\rho_Y$ under the marginal map $\mathcal{T}_Y$ and the pushforward of $\rho_\Theta$ under the conditional map $\mathcal{T}_{\Theta|\Yey}(\cdot) \coloneqq \mathcal{T}_{\Theta}( \mathcal{T}_Y^\lowsup{-1}(y), \cdot )$ for given data $y$ respectively take the form 
\begin{align}
p_{Y} (y) & := \big(\mathcal{T}_{Y}\big){}_\sharp \rho_Y (y) =  \rho_Y \big(\mathcal{T}_{Y}^{-1}(y)\big) \big|\nabla \mathcal{T}_{Y}^{-1}(y)\big|,\\
p_{\Theta|\Yey} (\theta|y) & := \big(\mathcal{T}_{\Theta|\Yey}\big){}_\sharp \rho_\Theta (\theta) =  \rho_\Theta \big(\mathcal{T}_{\Theta|\Yey}^{-1}(\theta)\big) \big|\nabla \mathcal{T}_{\Theta|\Yey}^{-1}(\theta)\big|,
\end{align}
which approximate the evidence $\pi_Y$ and the posterior $\pi_{\Theta|\Yey}$, respectively. This way, carrying approximate posterior inference via the conditional map does not incur any additional likelihood function evaluations. For example, we can draw i.i.d. random samples or quasi Monte Carlo (QMC) samples $\{U_\Theta^i\}_{i = 1}^{N_{\rm online}}$ from the reference density $\rho_\Theta$, and then evaluate the conditional map, $\Theta^i = \mathcal{T}_{\Theta|\Yey}(U_\Theta^i)$, to obtain approximate posterior samples. 

Although the conditional map can offer accurate posterior approximations over possible observed data with high probability (cf. the Markov bound in \eqref{eq:MarkovBound}), we may still want to correct for the potential bias in the approximate inference results. One option is to post-process the approximate inference results using weights. Given a set of samples drawn from the approximate posterior, $\{\Theta^i\}_{i = 1}^{N_{\rm online}}$, we can compute the weights 
\[
W^i = \frac{\pi_{\Theta|\Yey}(\Theta^i|y)}{p_{\Theta|\Yey}(\Theta^i|y)}, \quad i = 1, \ldots, N_{\rm online},
\] 
to obtain weighted posterior samples and to estimate posterior expectations of some statistics using importance sampling (IS). In this case, the efficiency of the post-processing is controlled by the divergence of $\pi_{\Theta|\Yey}$ from its approximation $p_{\Theta|\Yey}$.

Another option is to accelerate exact posterior inference using the conditional map as a preconditioner. Since the pushforward density $(\mathcal{T}_{\Theta|\Yey})_\sharp \rho_\Theta$ approximates the posterior density $\pi_{\Theta|\Yey}$, the pullback of the posterior density,
\(
(\mathcal{T}_{\Theta|\Yey}){}^\sharp \pi_{\Theta|\Yey}, 
\)
can be viewed as a perturbation of the reference density $\rho_\Theta$. Instead of directly sampling the complicated posterior density using MCMC, one can simulate Markov chains invariant to the pullback of the posterior density, which has geometry similar to that of the reference density. This makes the design of efficient MCMC proposals a much easier task. For example, with a Gaussian reference density, e.g., $\rho_\Theta = \mathcal{N}(0, \mI)$, we can use the preconditioned Crank-Nicholson (pCN) proposal of \cite{MCMC:BRSV_2008,MCMC:CRSW_2013},
\begin{equation}\label{eq:pcn}
\theta' = \frac{2 - \Delta t}{2 + \Delta t} \, \theta + \frac{2\sqrt{2\Delta t}}{2 + \Delta t}\, u_\Theta, \quad u_\Theta \sim \rho_\Theta, 
\end{equation}
where $\Delta t > 0$ is the step size, to construct Markov chains. The pCN proposal is suitable for sampling perturbations of Gaussian densities, because it has an acceptance probability one for sampling Gaussian densities with arbitrary dimensions.


\section{Numerical examples}\label{sec:numerics}

\subsection{Parameter estimation of an SIR model}

{\bfseries Setup.} We first consider the susceptible-infected-recovered (SIR) model commonly used in epidemiology, which is governed by a system of ODEs
\begin{equation}
 \frac{dS(t)}{dt}  = -\beta S I, \qquad
 \frac{dI(t)}{dt}  = \beta S I - \gamma I, \qquad
 \frac{dR(t)}{dt}  = \gamma I, \label{eq:sir}
\end{equation}
started from $S(0)=99$, $I(0)=1$, $R(0)=0$,
and solved using \texttt{ode45} in Matlab with a relative error threshold $10^{-6}$.
The parameters $\theta = (\beta,\gamma)$ are subject to inference from the posterior distribution, conditional on $4$ noisy observations of the number of infected individuals.
The latter are modelled as
\begin{equation}\label{eq:sir:y}
y = \left(I(1.25), I(2.5), I(3.75), I(5)\right) + e^y, \quad \mbox{where} \quad e^y_i \sim \mathcal{N}(0,1), \quad i=1,\ldots,4.
\end{equation}
Reasonable values of $\theta$ vary between $0$ and $2$,
so we use a uniform prior $\pi_\Theta =  \mu_{\scriptscriptstyle [0,2]^2}$.
The total initial number of individuals of $100$ yields $y$ in the range $[0,100]$.

Therefore, in the initial DIRT layer, we introduce $17$ equispaced discretization points on $[0,2]$ for $\theta_1,\theta_2$, and $17$ equispaced discretization points on $[0,100]$ for $y_1,\ldots,y_4$.
We set the TT ranks to $17$ and the reference measure to the truncated normal distribution on $[-3,3]$. We use nine layers of maps in DIRT with the tempering powers starting from $\beta_0=10^{-4}$ and increasing geometrically as $\beta_{\ell+1}=\sqrt{10} \cdot \beta_\ell$ until $\beta_8=1$. 
\cancel{Constructing DIRT with those parameters requires a total of $182070$ evaluations of the bridging densities in all layers,
and 21 seconds of CPU time in Matlab R2020a on an Intel Core i7-9750H CPU.}
\retwo{In the offline phase, constructing DIRT with those parameters requires a total of $182070$ evaluations of the bridging densities in all layers,
and 21 seconds of CPU time in Matlab R2020a on an Intel Core i7-9750H CPU. To compare DIRT with density estimation approaches, we employ the Hierarchical Invertible Neural Transport (HINT) \cite{kruse2019hint} in this example. HINT approximates a transport map using i.i.d. samples drawn from the joint density that can be generated explicitly in this case. 
The HINT map is constructed from neural networks with $8$ dense layers of size $200\times 200$ each, trained by $50$ epochs of $10^6$ training samples.
The training took 28 minutes on a NVidia GeForce GTX 1650 Max-Q card. }

{\bfseries Accuracy and online performance.}  Since we are mostly interested in sampling from the posterior rather than the joint,
the accuracy of \cancel{DIRT} \retwo{DIRT and HINT} is benchmarked as follows.
We sample $256$ realizations of $\Theta \sim  \pi_\Theta$ and
$e^y \sim \mathcal{N}(0,\mI)$.
By solving the ODE model~\eqref{eq:sir} for each $\Theta$ adding the corresponding $e^y$ as in~\eqref{eq:sir:y}, we generate $256$ realizations of synthetic data. Then, we plug each data set $y^\lowsup{i}$ into the \cancel{conditional DIRT } \retwo{conditional maps defined by DIRT and HINT} to obtain a resulting approximate posterior density $p_{\Theta|Y{=}y^\lowsup{i}}$ and draw $N=5\cdot 10^4$ conditional random samples $\Theta|Y{=}y^\lowsup{i}$ from each $p_{\Theta|Y{=}y^\lowsup{i}}$. \cancel{Producing one conditional sample from DIRT takes 43 microseconds of CPU time.} These samples are used to estimate the Hellinger distance between the approximate posterior density $p_{\Theta|Y{=}y^\lowsup{i}}$ and the exact posterior density $\pi_{\Theta|Y=y^\lowsup{i}}$. The $256$ instances of the Hellinger distances, corresponding to the different realizations of $Y$, are binned in a histogram as shown in Figure~\ref{fig:sir:acc} \cancel{(left)}.
\cancel{We observe that although conditioning the approximate joint density on the tail of the data distribution might result in a rather large relative error in the DIRT-approximation of the corresponding conditional density, such outliers are rare.}
\retwo{For the DIRT-approximation (left, Figure~\ref{fig:sir:acc}), we observe that although conditioning the approximate joint density on the tail of the data distribution might result in a rather large relative error, such outliers are rare.} This behavior is expected as the Hellinger error of the approximate conditional density is controlled with high probability in our framework. \retwo{In comparison, the approximate joint density obtained by HINT (right, Figure~\ref{fig:sir:acc}) is generally less accurate than those of DIRT. However, HINT uses an order of magnitude more training samples and CPU time. This accuracy margin is not surprising. DIRT is a function approximation procedure that uses target density functions evaluated at collocation points to train the transport map. In comparison, HINT is a density estimation procedure trained by minimizing some loss function approximated using samples drawn from the target distribution. This way, HINT does not use the target density values in training and its loss function is subject to the Monte Carlo error.}

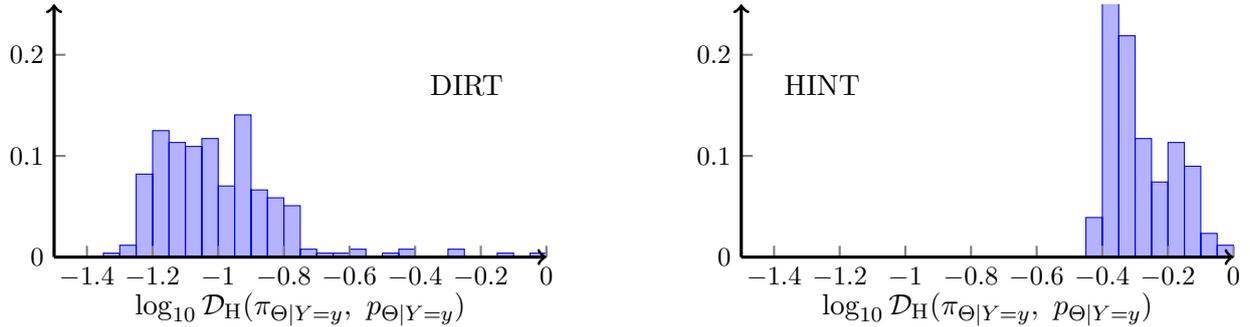
\begin{figure}[htb]
\centering\begin{tikzpicture}
 \begin{axis}[%
 width=0.49\linewidth,
 height=0.3\linewidth,
 area style,
 xlabel={$\log_{10}\DH(\pi_{\Theta|Y=y}, ~p_{\Theta|Y=y})$},
 xmin=-1.5, xmax=0, ymin=0,ymax=0.25,
 ytick={0,0.1,0.2,0.3},
 yticklabels={0,0.1,0.2,0.3},
 ]
  \addplot+[ybar interval,mark=no] plot coordinates{
     (   -1.3500,    0.0039)     
   (   -1.3000,    0.0117)
   (   -1.2500,    0.0820)
   (   -1.2000,    0.1250)
   (   -1.1500,    0.1133)
   (   -1.1000,    0.1094)
   (   -1.0500,    0.1172)
   (   -1.0000,    0.0703)
   (   -0.9500,    0.1406)
   (   -0.9000,    0.0664)
   (   -0.8500,    0.0586)
   (   -0.8000,    0.0508)
   (   -0.7500,    0.0078)
   (   -0.7000,    0.0039)
   (   -0.6500,    0.0039)
   (   -0.6000,    0.0078)
   (   -0.5500,         0)
   (   -0.5000,    0.0039)
   (   -0.4500,    0.0078)
   (   -0.4000,         0)
   (   -0.3500,         0)
   (   -0.3000,    0.0078)
   (   -0.2500,         0)
   (   -0.2000,         0)
   (   -0.1500,    0.0039)
   (   -0.1000,         0)
   (   -0.0500,    0.0039)
   (         0,    0.0039)
  };
  \node[anchor=north east] at (axis cs: -0.1,0.19) {DIRT};
 \end{axis}
\end{tikzpicture}
\hfill
\begin{tikzpicture}
 \begin{axis}[%
 width=0.49\linewidth,
 height=0.3\linewidth,
 area style,
 xlabel={$\log_{10}\DH(\pi_{\Theta|Y=y}, ~p_{\Theta|Y=y})$},
 xmin=-1.5, xmax=0, ymin=0,ymax=0.25,
 ytick={0,0.1,0.2,0.3},
 yticklabels={0,0.1,0.2,0.3},
 ]
  \addplot+[ybar interval,mark=no] plot coordinates{
   (-0.4500,     0.0391)      
   (-0.4000,     0.3125)
   (-0.3500,     0.2188)
   (-0.3000,     0.1172)
   (-0.2500,     0.0742)
   (-0.2000,     0.1133)
   (-0.1500,     0.0898)
   (-0.1000,     0.0234)
   (-0.0500,     0.0117)
   (      0,     0.0117)
  };
  \node[anchor=north west] at (axis cs: -1.4,0.19) {HINT};
 \end{axis}
\end{tikzpicture}
\caption{Top: histograms (out of $256$ experiments with different $y$) of Hellinger distances between exact and approximate posterior densities for DIRT (left) and HINT (right).
}
\label{fig:sir:acc}
\end{figure}

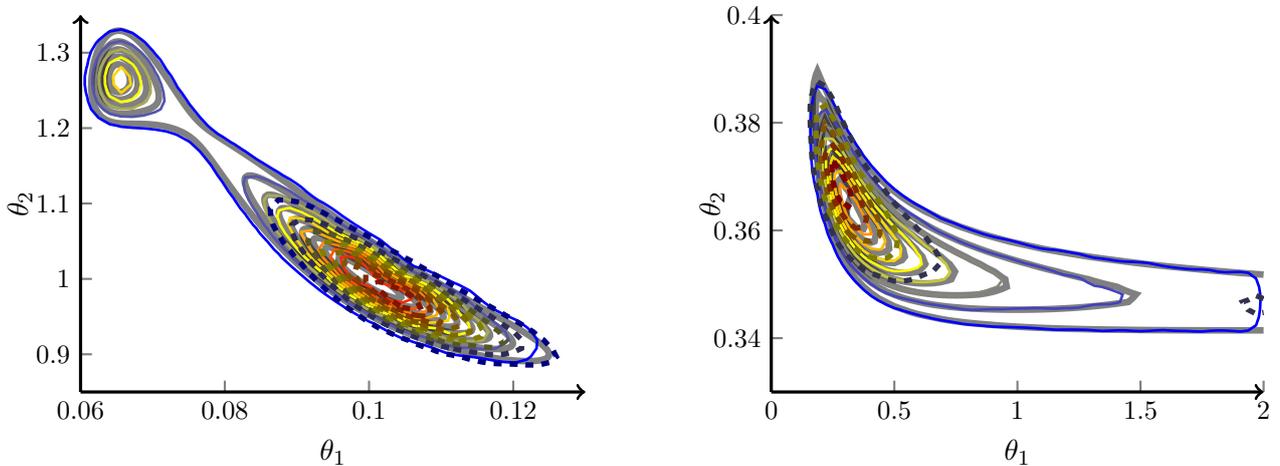
\begin{figure}[htb]
\centering
\noindent\hspace{-12pt}\begin{tikzpicture}
 \begin{axis}[%
width=0.5\linewidth,
height=0.40\linewidth,
xmin=0.06,xmax=0.13,ymin=0.85,ymax=1.35,
xtick={0.06,0.08,0.1,0.12},
xticklabels={0.06,0.08,0.1,0.12},
xlabel={$\theta_1$},ylabel={$\theta_2$},
]
\pgfplotsset{
    contour/every contour plot/.style={labels=false},
}
 \addplot[contour prepared,contour prepared format=matlab,contour/draw color={black},opacity=0.5,line width=2.5pt] table{contour0ex.dat};
 \addplot[contour prepared,contour prepared format=matlab,line width=1pt] table{contour0.dat};
 \addplot[contour prepared,contour prepared format=matlab,contour/draw color={mapped color!50!black},dashed,line width=2pt] table{contour0h.dat};
 \end{axis}
\end{tikzpicture}
\hfill
\begin{tikzpicture}
 \begin{axis}[%
width=0.49\linewidth,
height=0.40\linewidth,
xmin=0.0,xmax=2.0,ymin=0.33,ymax=0.40,
xlabel={$\theta_1$},ylabel={$\theta_2$},
]
\pgfplotsset{
    contour/every contour plot/.style={labels=false},
}
 \addplot[contour prepared,contour prepared format=matlab,contour/draw color={black},opacity=0.5,line width=2.5pt] table{contour8ex.dat};
 \addplot[contour prepared,contour prepared format=matlab,line width=1pt] table{contour8n.dat};
 \addplot[contour prepared,contour prepared format=matlab,contour/draw color={mapped color!50!black},dashed,line width=2pt] table{contour8hn.dat};
 \end{axis}
\end{tikzpicture}
\caption{Exact posterior densities (thick grey) and approximations for two realizations of $y$ for DIRT (solid bright color lines) and HINT (dashed dark lines).
}
\label{fig:sir:contours}
\end{figure}

\cancel{
We employ the Hierarchical Invertible Neural Transport (HINT) \cite{kruse2019hint} in this example to compare DIRT with density estimation approaches. HINT approximates a transport map using i.i.d. samples drawn from the joint density that can be generated explicitly in this case. 
The HINT map is constructed from neural networks with $8$ dense layers of size $200\times 200$ each, trained by $50$ epochs of $10^6$ training samples.
The training took 28 minutes on a NVidia GeForce GTX 1650 Max-Q card,
and the evaluation of one test sample took 13 microseconds.
A histogram of Hellinger-distance errors in the HINT approximation of the posterior is shown on Figure~\ref{fig:sir:acc} (right). We see that DIRT is systematically more accurate than HINT but uses an order of magnitude fewer training evaluations of the joint density and CPU time.}

Figure~\ref{fig:sir:contours} \cancel{shows} \retwo{further illustrates this accuracy margin by showing} the contours of approximate posterior densities conditioned on two different data realizations.
The left plot of Figure~\ref{fig:sir:contours} shows an interesting multimodal posterior distribution, in which the ``true'' parameter $\theta=(0.1, 1)$ is used to generate the synthetic data. This corresponds to a realistic balance between infection and recovery rates.
For this test case, DIRT gives the Hellinger distance $\DH(\pi_{\Theta|Y=y},~ p_{\Theta|Y=y}) \approx 0.32,$
whereas HINT gives
$\DH(\pi_{\Theta|Y=y},~ p_{\Theta|Y=y}) \approx 0.68.$ \cancel{As a final remark, although DIRT appears to be more computationally expensive than HINT in the online phase, this bottleneck can be largely alleviated by replacing the current MATLAB implementation of DIRT with GPUs and parallel computing platforms.}

\retwo{In the online phase, producing one conditional sample from DIRT and HINT respectively take 43 and 13 microseconds of CPU time. Although DIRT appears to be more computationally expensive than HINT in the online phase, this bottleneck can be largely alleviated by replacing the current MATLAB implementation of DIRT with GPUs and parallel computing platforms.}

{\bfseries Accelerating posterior sampling.}  We also demonstrate exact posterior inference accelerated using the conditional DIRT as a preconditioner. In this experiment, we fix the data $y$ produced from \eqref{eq:sir} with the ``true'' parameter $\theta=(0.1, 1)$. We also transform the uniform prior density into a Gaussian density so that we can apply the No-U-Turn sampler (NUTS) \cite{hoffman2014no} and the pCN proposal \eqref{eq:pcn} to sample the exact posterior distribution to setup baselines of the benchmark. The step size of NUTS is automatically adjusted. For pCN, we use a manually tuned step size $\Delta t=\mathrm{e}^{-7}$. Then, we apply pCN and NUTS to the pullback of the posterior density,
\(
(\mathcal{T}_{\Theta|\Yey}){}^\sharp \pi_{\Theta|\Yey},
\)
and refer to the corresponding sampling schemes as pCN-DIRT and NUTS-DIRT, respectively. In contrast to the unpreconditioned pCN, pCN-DIRT has high acceptance probability with arbitrary step size. Thus, we experiment it with $\Delta t=10$ and $\Delta t=2$, which correspond to a proposal with negatively correlated samples and an independence Metropolis proposal, respectively. We also report the performance of importance sampling using random samples drawn from the conditional DIRT (IS-DIRT).

\cancel{For each MCMC method we produce $N=5\cdot 10^4$ posterior samples $\Theta|Y=y$, and measure its performance using the estimated IACT of the Markov chain. For IS-DIRT, we measure its performance by the ratio between the total number of samples and the effective sample size (ESS). Moreover, we compute the empirical standard deviations of the estimated posterior means: we split $N=5\cdot 10^4$ points into $5$ batches of $10^4$ samples each, compute averages of $\theta_1,\theta_2$ over each batch, and finally compute std over the $5$ batches.
Table~\ref{tab:sir-err} shows that all MCMC methods are more accurate when preconditioned by conditional DIRT.
Since the pullback distribution is approximately normal, the pCN is stable even with negatively correlated proposal samples (with $\Delta t>2$), which further reduces the IACT and the estimation error.
Random samples and QMC samples drawn by IS-DIRT and the paths of the first 5 MCMC samples are shown in Figure~\ref{fig:sir:chains}. In this example, the preconditioned methods are able to explore the second mode in the first several iterations.}

\retwo{We consider two benchmarks to quantify the sampling performance. First, we consider the integrated autocorrelation time (IACT) of the parameter Markov chains, averaged over parameters, to measure the independence of simulated samples. Since different MCMC methods use different numbers of posterior density evaluations in each Markov chain step, we also report the IACT per density evaluation, which is the product of the IACT and the average number of density evaluations per MCMC step. For IS-DIRT, we measure its performance by the ratio of the total number of samples to the effective sample size (ESS).} 

\retwo{Second, we demonstrate the variance-reduction performance of various methods discussed here using some statistical estimators. Suppose we want to compute the posterior expectation of some quantity of interest (QoI), $h$, using the sample average, denoted by $\hat h$. Since all methods discussed here are asymptotically converging, we assume that the variance of the estimator $\hat h$ satisfies the central limit theorem, and thus follows
\[
  \mathrm{var}(\hat h) = \frac{\tau(h)\, \mathrm{var}(h)}{N}, 
\]
where $\mathrm{var}(h)$ is the posterior variance of $h$, $N$ is the number of samples used in computing the estimator $\hat h$, and $\tau(h)$ is a QoI-specific prefactor that measures either the correlation of MCMC samples or the weight imbalance of importance sampling. To summarize the computational efficiency of the estimator, we then express $\mathrm{var}(\hat h)$ in the form of
\[
  \mathrm{var}(\hat h) = \frac{\tau_{\rm online}(h)}{N_{\rm online}} ,
\]
where $N_{\rm online}$ is the number of posterior density evaluations needed to generate $N$ samples, and $\tau_{\rm online}(h)$ is an overhead factor that measures the deficiency of the estimator. This way, we have 
\begin{equation}\label{eq:overhead}
\tau_{\rm online}(h) = \mathrm{var}(\hat h) \times N_{\rm online} .
\end{equation}
Note that the overhead factor $\tau_{\rm online}$ is only a computable indicator of the statistical efficiency. With finite sample size, removing the initialization bias of MCMC is often an art, whereas the self-normalized importance sampling has a bias that decays at a faster rate than the square root of the variance.}

\begin{table}[!h]
\caption{\retwo{SIR model. The setup of sampling methods, the number of density evaluations for generating $N = 50000$ samples, the IACT, the IACT per density evaluation, the estimated square roots of the vairances of $(\hat\theta_1, \hat\theta_2)$, as well as the overhead factors $\tau_{\rm online}(\theta_1)$ and $\tau_{\rm online}(\theta_2)$. $^\dagger$The ratio of the sample size to the ESS is used here for IS-DIRT.}}
\label{tab:sir-err}
\small
\begin{center}
\renewcommand{\arraystretch}{1.2}
\begin{tabular}{c|ccc|c|cc}
                     & pCN & pCN-DIRT & pCN-DIRT & IS-DIRT  & NUTS & NUTS-DIRT \\
step size            & $\Delta t=\mathrm{e}^{-7}$ & $\Delta t=10$ & $\Delta t=2$ & -  & auto & auto \\
$N_{\mbox{online}}$          &  50000    & 50000    & 50000   &  50000 & 149738 & 116938  \\ \hline
IACT                 & 251.8     & 1.61     & 3.94     & - & 29.5     & 9.43   \\
IACT per eval.  & 251.8     & 1.61     & 3.94     & 1.88 $^\dagger$ & 88.3     & 22.1   \\\hline
$\surd\mbox{var}(\hat\theta_1)$  & 2.40e-3  & 1.44e-4  & 4.58e-4 & 4.98e-4 & 1.10e-3 & 4.91e-4 \\
$\tau_{\rm online}(\theta_1)$  & 2.88e-1  & 1.04e-3  & 1.05e-2 & 1.24e-2 & 1.82e-1 & 2.82e-2 \\\hline
$\surd \mbox{var}(\hat\theta_2)$ & 1.74e-2  & 9.87e-4  & 3.11e-3 & 3.16e-3  & 8.04e-3 & 4.08e-3 \\
$\tau_{\rm online}(\theta_2)$ & 15.1  & 4.87e-2  & 4.84e-1 & 4.99e-1  & 9.68 & 1.95 \\
\end{tabular}
\end{center}
\end{table}

\retwo{For each method we produce $N=5\cdot 10^4$ samples from the posterior $\Theta|Y=y$. We use the parameters $(\theta_1, \theta_2)$ as the QoIs, and thus their sample averages, $(\hat\theta_1, \hat\theta_2)$, give the posterior mean estimate of the parameters. In this experiment, we split $N=5\cdot 10^4$ points into $5$ batches of $10^4$ samples each, compute $\hat\theta_1,\hat\theta_2$ over each batch, and finally compute variances over the $5$ batches. Table~\ref{tab:sir-err} summarizes the IACT, IACT per density evaluation (or the ratio of the sample size to the ESS), estimated square roots of the variances of $(\hat\theta_1, \hat\theta_2)$, as well as the overhead factors  $\tau_{\rm online}(\theta_1)$ and $\tau_{\rm online}(\theta_2)$. We observe that all MCMC methods are more efficient when preconditioned by the conditional DIRT. Since the pullback distribution is approximately normal, the pCN is stable even with negatively correlated proposal samples (with $\Delta t>2$), which further reduces the IACT and the overhead factors of the estimators shown. Overall, pCN-DIRT with  $\Delta t=10$ produces the most computationally efficient result.}

\begin{figure}[!h]
\noindent
\begin{tikzpicture}
 \begin{axis}[%
width=0.49\linewidth,
height=0.45\linewidth,
xmin=0.06,xmax=0.13,ymin=0.85,ymax=1.35,
xtick={0.06,0.08,0.1,0.12},
xticklabels={0.06,0.08,0.1,0.12},
xlabel={$\theta_1$},ylabel={$\theta_2$},
]
\pgfplotsset{
    contour/every contour plot/.style={labels=false},
}
 \addplot[contour prepared,contour prepared format=matlab] table{contour0ex.dat};
 \addplot[fill=black,draw=red,mark size=2.0pt,mark=*,line width=1pt,mark options={color=black}] table{%
   0.0956931   1.0275114
   0.0888990   1.0531503
  };
 \addplot[fill=black,draw=red,mark size=2.0pt,mark=*,line width=1pt,mark options={color=black}] table{%
   0.0888990   1.0531503
   0.1005234   1.0293665
 };
 \addplot[fill=black,draw=red,mark size=2pt,mark=*,line width=1pt,mark options={color=black}] table{%
   0.1005234   1.0293665
   0.1159371   0.9141797
 };
 \addplot[fill=black,draw=red,mark size=2.0pt,mark=*,line width=1pt,mark options={color=black}] table{%
   0.1159371   0.9141797
   0.0697678   1.2822385
 };
 \end{axis}
\end{tikzpicture}
\hfill
\begin{tikzpicture}
 \begin{axis}[%
width=0.49\linewidth,
height=0.45\linewidth,
xmin=0.06,xmax=0.13,ymin=0.85,ymax=1.35,
xtick={0.06,0.08,0.1,0.12},
xticklabels={0.06,0.08,0.1,0.12},
xlabel={$\theta_1$},ylabel={$\theta_2$},
]
\pgfplotsset{
    contour/every contour plot/.style={labels=false},
}
 \addplot[contour prepared,contour prepared format=matlab] table{contour0ex.dat};
 \addplot[fill=black,draw=red,mark size=1pt,mark=*,line width=.5pt,mark options={color=black}] table{%
   x           y
   0.1000000   1.0000000
   0.1016633   0.9998254
  };
 \addplot[fill=black,draw=red,mark size=1pt,mark=*,line width=0.5pt,mark options={color=black}] table{%
   0.1016633   0.9998254
   0.1016633   0.9998254
 };
 \addplot[fill=black,draw=red,mark size=1pt,mark=*,line width=0.5pt,mark options={color=black}] table{%
   0.1016633   0.9998254
   0.1016633   0.9998254
 };
 \addplot[fill=black,draw=red,mark size=1pt,mark=*,line width=0.5pt,mark options={color=black}] table{%
   0.1016633   0.9998254
   0.1031757   1.0022011
 };
 \end{axis}
\end{tikzpicture}
\caption{Exact posterior density (contours) and the paths of 5 posterior samples (dotted lines) from pCN-DIRT (left) and unpreconditioned NUTS (right).
}
\label{fig:sir:chains}
\end{figure}
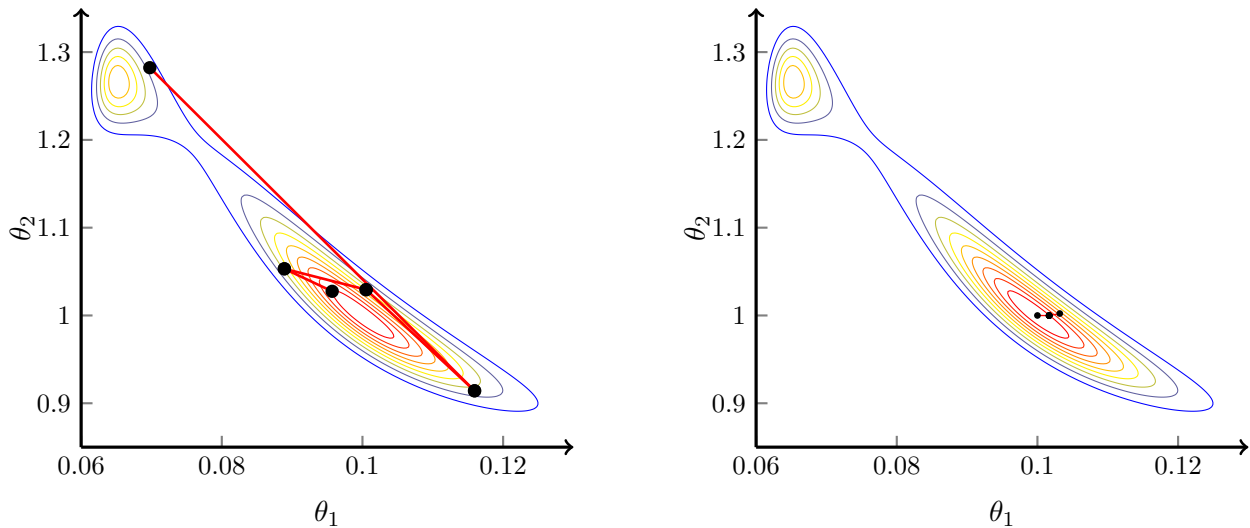

\retwo{We further visualize the benifit of preconditioning MCMC using the conditional DIRT. Figure~\ref{fig:sir:chains} shows the first 5 MCMC samples generated by pCN-DIRT and unpreconditioned NUTS. In this example, the preconditioned methods are able to explore the second mode in a few iterations, whereas unpreconditioned  samples are clustered near the initial state.}

\cancel{The last row of Table~\ref{tab:sir-err} shows the actual number of density evaluations required for producing $N=5\cdot 10^4$ posterior samples.
pCN needs $1$ evaluation per sample by construction,
but NUTS performs several iterations, resulting in a two to threefold larger cost.
Comparing $N_{\mbox{online}}$ with $N_{\mbox{offline}} = 182070$ evaluations in the construction of DIRT,
we conclude that the offline cost is already paid off by only $2$ -- $4$ posterior evaluations.}

\subsection{Linear elasticity analysis of a wrench} 

We consider a two-dimensional linear elasticity problem \cite{lam2020multifidelity,smetana2020randomized} that models the displacement field $u:\mathcal{D} \rightarrow \mathbb{R}^\lowsup{2}$ using the PDE 
\[
\mathrm{div}\left(K:\varepsilon(u)\right) = f \quad \text{on} \quad \mathcal{D} \subset \mathbb{R}^\lowsup{2}.
\]
which models the stress equilibrium in a physical body $\mathcal{D}$ subject to external forces $f$.
Here, $\varepsilon(u) = \frac{1}{2}(\nabla u + \nabla u^\lowsup{\top})$ is the strain field, and $K$ is the Hooke tensor such that
\[
K:\varepsilon(u) = \frac{E}{1+\nu} \varepsilon(u) + \frac{\nu E}{1-\nu^2} \mathrm{trace}(\varepsilon(u))\begin{pmatrix} 1&0 \\0&1\end{pmatrix}
,\]
where $\nu=0.3$ is Poisson's ratio and $E>0$ is spatially varying Young's modulus.
In this example, we aim to estimate Young's modulus field $E:\mathcal{D}\rightarrow\R_{>0}$. It is modeled by a random field that follows a log-normal prior $\log E\sim \mathcal{N}(0,C)$, where $C(x,x') = \exp(-\|x-x'\|_2^2)$ is a Gaussian covariance function on $\mathcal{D}\times \mathcal{D}$.
After finite element discretization, $\log E$ is replaced with a piecewise constant field whose elementwise values are gathered in a vector $\theta\in\R^\lowsup{n}$ where $n=925$ is the number of elements in the mesh. The parameter $\theta$ thus follows a centered Gaussian prior with covariance matrix $\Sigma_{ij}=C(x_i,x_j)$, where $x_i\in\mathcal{D}$ is the center of the $i$-th element.
We compute the numerical solution $u^\lowsup{h}=u^\lowsup{h}(\theta)$ by the Galerkin projection of $u=u(\theta)$ onto the space of continuous piecewise affine functions over a triangular mesh.
The domain $\mathcal{D}$, the mesh, the boundary conditions, and a sample von Mises stress of the solution are shown in Figure~\ref{fig:wrench} (middle).

We observe the vertical displacements $u^h_2$ of the $m=26$ points of interests located along the green line where the force $f$ is applied, see Figure~\ref{fig:wrench} (middle).
The perturbed observations are $y=u^h_2+e$ where $e$ is a zero-mean $H^\lowsup{1}$-normal noise with the signal-to-noise-ratio $10$.
Following Proposition \ref{prop:bound_Hellinger_rotation}, \cancel{we reparametrize and truncate both $\theta$ and $y$ by replacing them with $\theta_{\leq t}=B_{t}^\lowsup{\top}\theta\in\mathbb{R}^\lowsup{t}$ and $y_{\leq s}=A_{s}^\lowsup{\top}y\in\mathbb{R}^\lowsup{s}$, where $A_{s}$ and $B_{t}$ are the matrices containing the leading eigenvectors of the $s$ and $t$ largest eigenvalues of $H_Y$ and $H_\Theta$, respectively.}
\retwo{we first compute the eigendecomposition of the sensitivity matrices $H_Y$ and $H_\Theta$, and use their leading eigenvectors to construct bases $A_{s} \in \R^{m\times s}$ and $B_{t} \in \R^{n \times t}$ for truncate and reparameterize the data and the parameter, respectively. This way, the reduced dimensional parameter and data take the form $\bar\theta_{1:t}=B_{t}^\lowsup{\top}\theta\in\mathbb{R}^\lowsup{t}$ and $\bar y_{1:s}=A_{s}^\lowsup{\top}y\in\mathbb{R}^\lowsup{s}$, respectively.
The first three reduced parameter basis vectors are shown in Figure~\ref{fig:wrench-basis}.
Then, we can construct DIRT that approximates the reduced dimensional joint density $\widetilde\pi_{\bar Y_{s:1},\bar\Theta_{1:t}}\left(\bar y_{s:1},\bar\theta_{1:t}\right)$ defined in \eqref{eq:api_with_independence_structure}.}
Moreover, to reduce the offline training time, DIRT approximates the joint density where the exact observation map $u_2^h(\theta)$ is replaced by a surrogate TT approximation \cite{ds-alscross-2019},
which can be precomputed using only $485 \pm 28$ solves of the elasticity PDE. \reone{Other surrogate modelling techniques---for example, those based on polynomial chaos \cite{babuvska2007stochastic,cohen2010convergence,SuMo:XiuKar_2002,SuMo:MarNa_2009,schwab2012sparse,yan2019adaptive}, reduced order models \cite{ROM:BWG_2008_2,IP:ChenSch_2015,cohen2022nonlinear,ROM:CMW_2014,ROM:CMW_2016,ROM:GFWG_2008,ROM:LWG_2010}, and neural networks \cite{li2020fourier,lu2021learning,tripathy2018deep,zhou2020adaptive,zhu2019physics}---can also be used here to accelerate the training.
 }

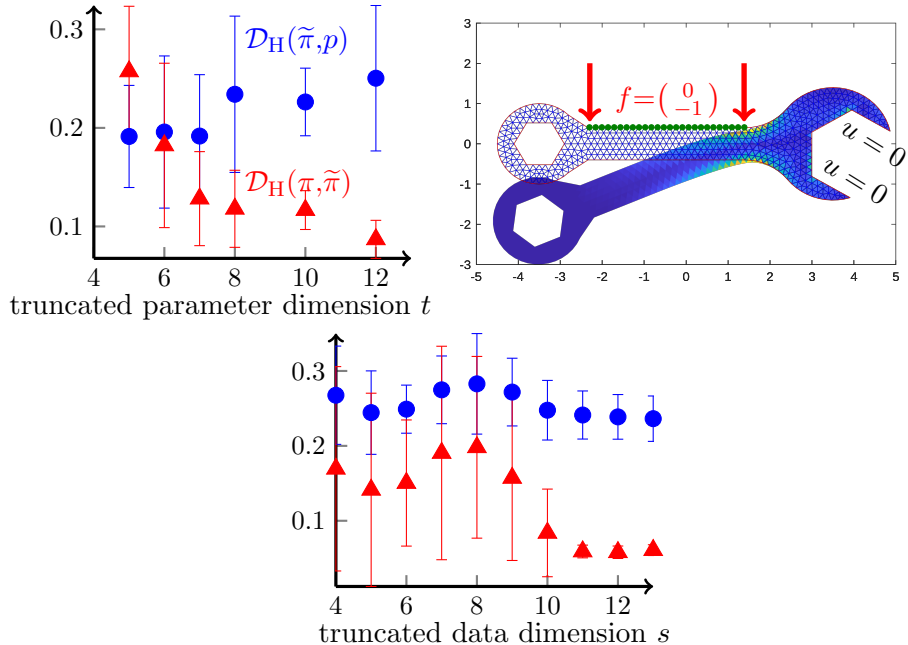
\begin{figure}[h]
\centering
\noindent\begin{tikzpicture}
\begin{axis}[%
width=0.35\linewidth,
height=0.3\linewidth,
xlabel={\retwo{\hspace{-24pt}truncated parameter dimension $t$\hspace{-24pt}}},
xmin=4,xmax=13,
legend style={at={(1.02,0.99)},anchor=north west}
]
\addplot+[mark=*,only marks,mark options={mark size=3},error bars/.cd,y dir=both,y explicit] coordinates{
(5 , 0.1913) +- (0, 0.0518)
(6 , 0.1958) +- (0, 0.0772)
(7 , 0.1917) +- (0, 0.0623)
(8 , 0.2341) +- (0, 0.0793)
(10, 0.2263) +- (0, 0.0343)
(12, 0.2504) +- (0, 0.0738)
                                } node[pos=0.65,anchor=south west,minimum height=40pt] {$\DH(\widetilde\pi,\!p)$}; 
\addplot+[mark=triangle*,only marks,mark options={mark size=4},error bars/.cd,y dir=both,y explicit] coordinates{
(5 , 0.2573) +- (0, 0.0661)
(6 , 0.1822) +- (0, 0.0834)
(7 , 0.1282) +- (0, 0.0477)
(8 , 0.1179) +- (0, 0.0391)
(10, 0.1166) +- (0, 0.0197)
(12, 0.0869) +- (0, 0.0193)
                                } node[pos=0.65,anchor=south west] {$\DH(\pi,\!\widetilde\pi)$}; 
\end{axis}
\end{tikzpicture}
\begin{tikzpicture}
 \node at (0,0.25) {\includegraphics[width=0.35\linewidth,height=0.21\linewidth]{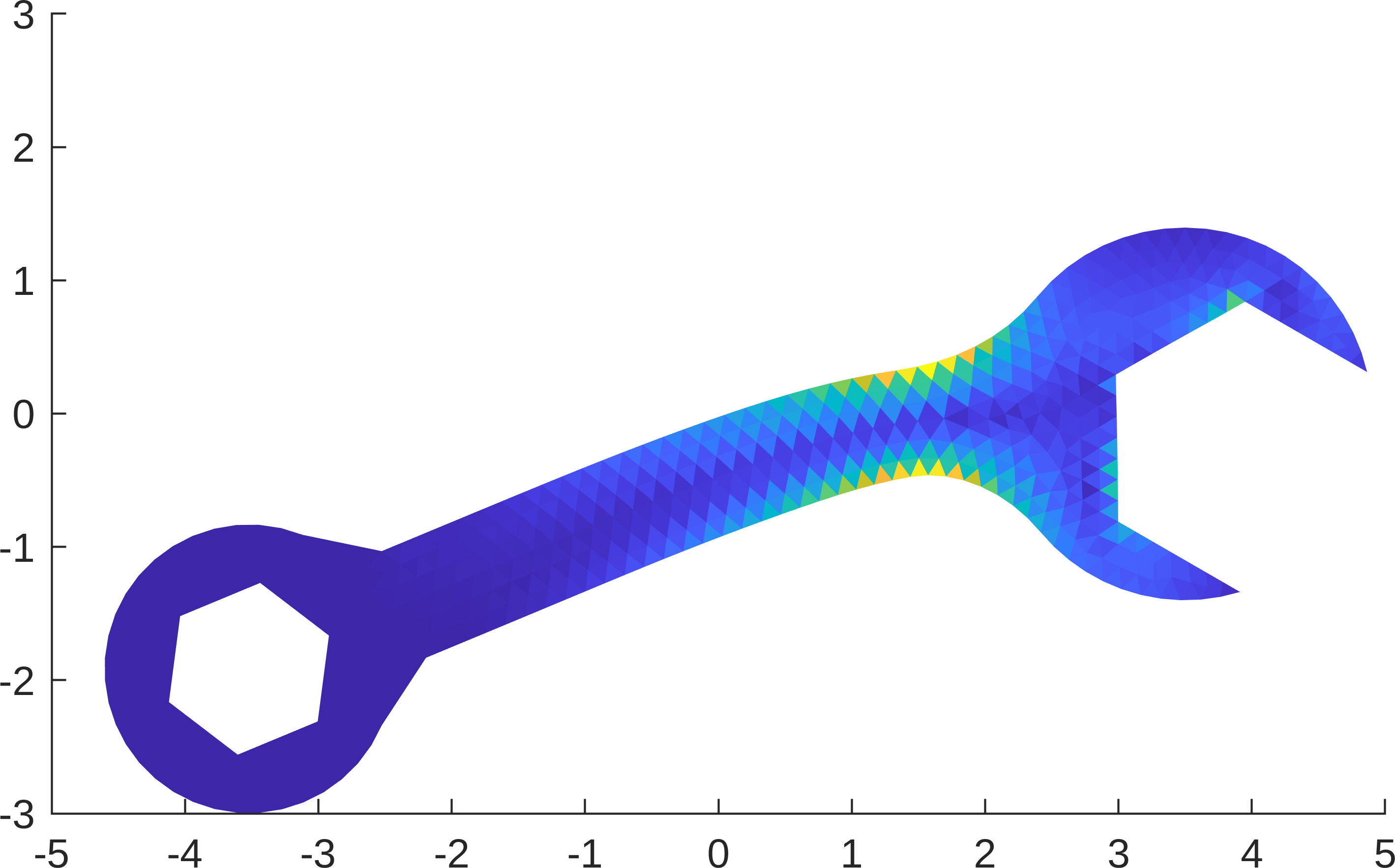}};
 \node at (0,0.25) {\includegraphics[width=0.35\linewidth,height=0.21\linewidth]{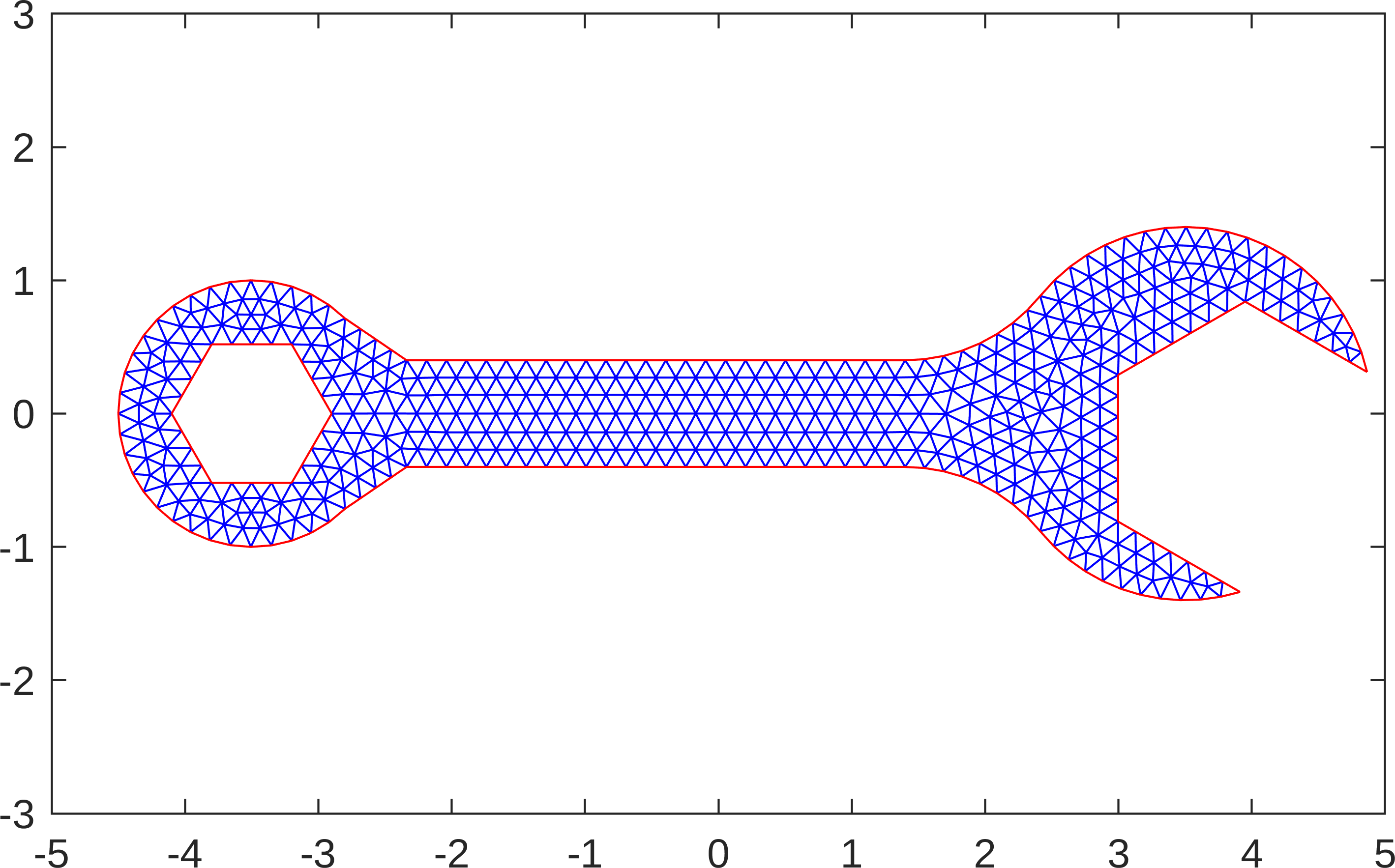}};  
 \node[rotate=-32] at (2.9*0.8,-0.2*0.7) {$u=0$};
 \node[rotate=-32] at (3.2*0.8,0.5*0.7) {$u=0$};
 \draw[->,red,line width=2pt] (-1.18,1.4) -- (-1.18,0.63);
 \draw[->,red,line width=2pt] (0.85,1.4) -- (0.85,0.63);
 \node[red,anchor=west] at (-1.4*0.7,1.3*0.7) {$f {=} \begin{psmallmatrix}0\\{-}1\end{psmallmatrix}$};
 \foreach \x in {-2.33824, -2.18871, -2.03918, -1.88965, -1.74012, -1.59059, -1.44106, -1.29153, -1.142, -0.992471, -0.842941, -0.693412, -0.543882, -0.394353, -0.244824, -0.0952941, 0.0542353, 0.203765, 0.353294, 0.502824, 0.652353, 0.801882, 0.951412, 1.10094, 1.25047, 1.4} {
 \node[ellipse,draw=green!50!black,fill=green!50!black,inner sep=0.7pt,anchor=center] at (0.545*\x+0.09,0.55) {};
 }
 \node at (0,-2) {};
\end{tikzpicture}
\hspace{-1em}
\begin{tikzpicture}
\begin{axis}[%
width=0.35\linewidth,
height=0.3\linewidth,
xlabel={\retwo{truncated data dimension $s$}},
xmin=4,xmax=13,
]
\addplot+[mark=*,only marks,mark options={mark size=3},error bars/.cd,y dir=both,y explicit] coordinates{
(4 , 0.2676) +- (0, 0.0656)
(5 , 0.2444) +- (0, 0.0557)
(6 , 0.2490) +- (0, 0.0321)
(7 , 0.2748) +- (0, 0.0452)
(8 , 0.2828) +- (0, 0.0671)
(9 , 0.2718) +- (0, 0.0452)
(10, 0.2476) +- (0, 0.0398)
(11, 0.2412) +- (0, 0.0321)
(12, 0.2387) +- (0, 0.0298)
(13, 0.2363) +- (0, 0.0303)
                                };
\addplot+[mark=triangle*,only marks,mark options={mark size=4},error bars/.cd,y dir=both,y explicit] coordinates{
(4 , 0.1694) +- (0, 0.1364)
(5 , 0.1413) +- (0, 0.1290)
(6 , 0.1504) +- (0, 0.0842)
(7 , 0.1905) +- (0, 0.1425)
(8 , 0.1981) +- (0, 0.1213)
(9 , 0.1572) +- (0, 0.1102)
(10, 0.0838) +- (0, 0.0585)
(11, 0.0590) +- (0, 0.0085)
(12, 0.0580) +- (0, 0.0083)
(13, 0.0608) +- (0, 0.0072)
                                }; 
\end{axis}
\end{tikzpicture}
\caption{Wrench setup (middle), Hellinger distances (mean$\pm$std over 10 data realizations) between DIRT approximation $p$ and reduced posterior $\widetilde \pi$ \retwo{(blue markers)} and between the reduced posterior and the original posterior $\pi$ \retwo{(red markers)} \retwo{with varying truncated parameter dimension $t$ and a fixed truncated data dimension $s=13$ (left), and varying truncated data dimension $s$ and a fixed truncated parameter dimension $t=13$ (right).}}
\label{fig:wrench}
\end{figure}

The Hellinger distances between the conditional DIRT and the reduced-dimensional posterior \eqref{eq:api_with_independence_structure}, and the Hellinger distances between reduced and original posteriors are shown in Figure~\ref{fig:wrench} (left and right).
We fix the TT ranks to be $13$ and specify tempering parameters $\beta_0=10^\lowsup{-3}$, $\beta_{\ell+1} = \sqrt{10} \beta_{\ell}$. 
\cancel{We observe that the DIRT accuracy degrades slowly with increasing dimensions, since the TT decomposition needs to capture more correlations.}
\retwo{We observe that the DIRT accuracy degrades slowly with increasing truncated parameter dimensions while remaining almost constant with increasing truncated data dimensions. One potential reason is that the parameters have complicated interactions in the parameter-to-observable map while the data is conditional independent in the observation model.}
However, using higher truncate dimensions in both the parameter and data spaces deliver more accurate parametrizations of the random fields. A reasonable strategy is to choose the truncate dimension such that the sum of \cancel{two errors} \retwo{the dimension truncation error and the DIRT approximation error} is minimal.
In this example, this is achieved for data truncation dimension $s = 13$ and parameter truncation dimension $t = 7$. Further numerical tests on the DIRT error with varying TT ranks and tempering steps are shown in \ref{sec:wrench2}.  Realizations of posterior samples generated by the conditional DIRT for different observed data sets are also presented in \ref{sec:wrench2}.

\begin{figure}[h]
\centering
\noindent
\includegraphics[width=0.32\linewidth]{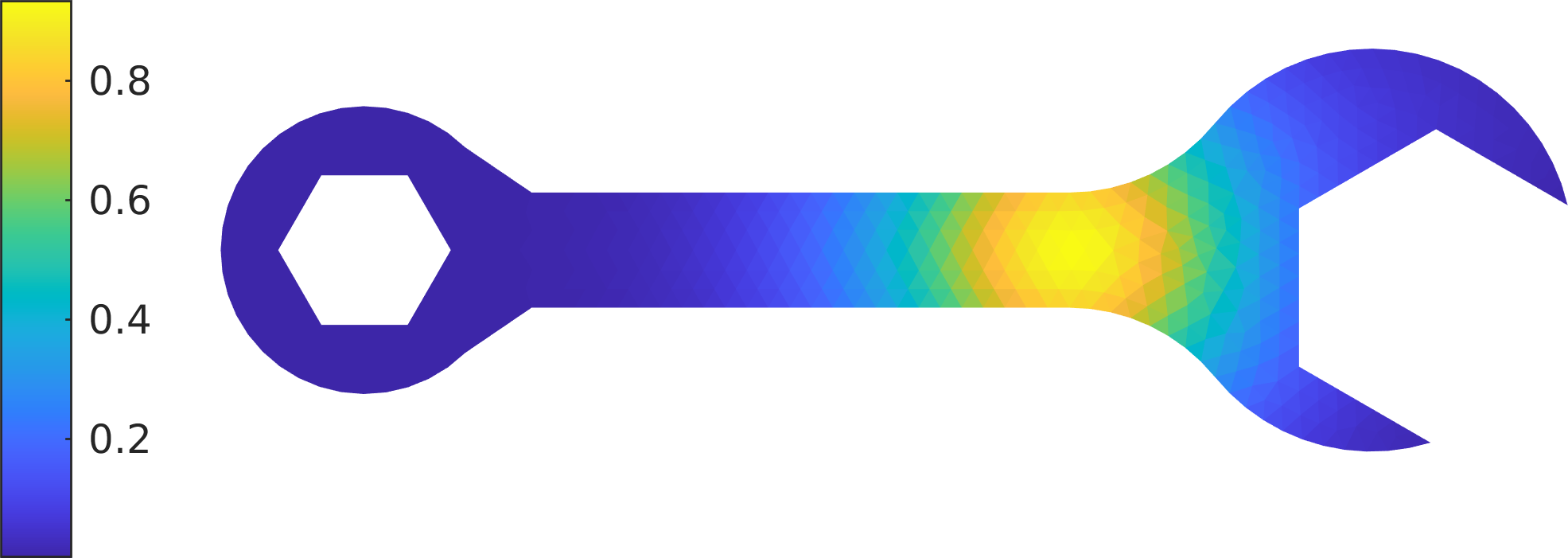}
\includegraphics[width=0.32\linewidth]{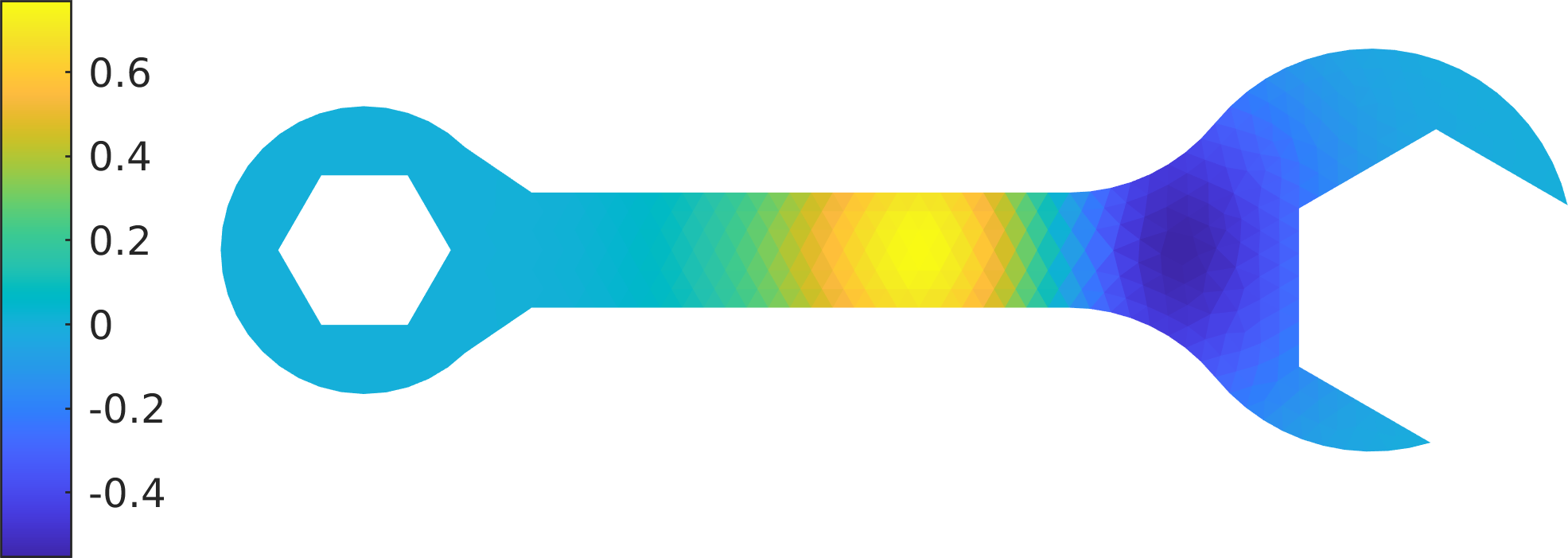}
\includegraphics[width=0.32\linewidth]{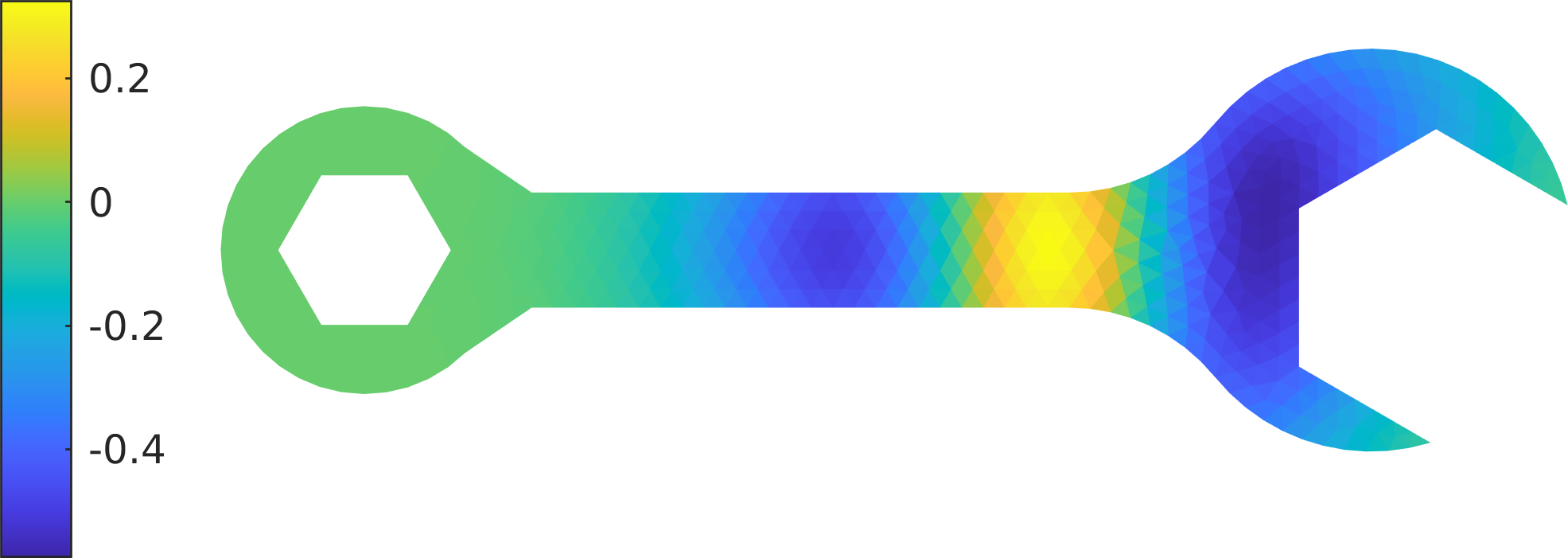}
\caption{Wrench basis vectors in parameters, $(B_t)_1$ (left), $(B_t)_2$ (middle) and $(B_t)_3$ (right).}
\label{fig:wrench-basis}
\end{figure}

\begin{table}[htb]
\caption{\retwo{Linear elasticity model. The setup of sampling methods, the number of density evaluations for generating $N = 8192$ samples, the IACT, the IACT per density evaluation, the estimated square roots of the total variances of $\widehat{\theta}$, and the overhead factors $\tau_{\rm online}(\theta)$. $^\dagger$The ratio of the sample size to the ESS is used here for IS-DIRT.}}
\label{tab:wrench-err}
\begin{center}
\renewcommand{\arraystretch}{1.2}
\small
\begin{tabular}{c|ccc|c|cc}
                     & pCN & pCN-DIRT & pCN-DIRT & IS-DIRT  & NUTS & NUTS-DIRT \\
step size            & $\Delta t=\mathrm{e}^{-1}$ & $\Delta t=10$ & $\Delta t=2$ & -  & auto & auto \\
$N_{\mbox{online}}$                            & 8192      &  8192               &  8192              & 8192 & 127623   &  94024 \\\hline
IACT         & 20.8   & 0.727              &   2.81           & - & 1.85 &  1.52 \\
IACT  per eval.  & 20.8   & 0.727              &   2.81           & 3.53 $^\dagger$ & 28.8 & 17.4 \\\hline
$\surd\mbox{var}(\widehat{\theta})$  & 0.1064    & 0.0232              &  0.0248            & 0.0319 & 0.0147 &  0.0202 \\
$\tau_{\rm online}(\theta)$  & 92.7    & 4.41              &  5.04            & 8.34 & 52.1 &  38.4 \\
\end{tabular}
\end{center}
\end{table}

Next, we compare all exact inference methods for one realization of the data, including NUTS, pCN, NUTS-DIRT, pCN-DIRT, and IS-DIRT. \cancel{In Table~\ref{tab:wrench-err} we show IACT for MCMC (and $N/ESS$ for importance sampling), errors in the posterior expectation of the log-coefficient, estimated using $4$ batches of $2048$ samples,
and the total number of elasticity PDE solves needed to produce those $8192$ samples. In this case, pCN-DIRT with $\Delta t = 10$ (which uses a negatively correlated proposal samples) takes advantage of reduced correlations in $\theta$ samples, and hence in a linear function of them $\log E$.
NUTS can achieve a smaller error in this example but at a much higher number of the posterior density evaluations, requiring the expensive solution of the PDE.} 
\retwo{For each method we produce $N=8192$ samples from the posterior $\Theta|Y=y$. The QoIs are the  $n$-dimensional parameters. We use the estimated total variance, $\mathrm{var}(\widehat{\theta}) = \sum_{i = 1}^n \mathrm{var}(\widehat{\theta}_i)$, to measure the variance-reduction performance over all parameters. In this experiment, variances of estimators are estimated using $4$ batches of $2048$ samples. Table~\ref{tab:wrench-err} summarizes the IACT, IACT per density evaluation (or the ratio of the sample size to the ESS), the estimated square root of the total variance of $\widehat{\theta}$, as well as the overhead factor $\tau_{\rm online}(\theta)$ (see \eqref{eq:overhead}). Similar to the previous example, we observe that all MCMC methods are more efficient when preconditioned by the conditional DIRT, where pCN-DIRT with  $\Delta t=10$ produces the most computationally efficient result.}


\subsection{Elliptic PDE with Besov prior.}\label{sec:Besov}

In problems such as medical imaging and hydrology, practitioners use the PDE
\[
-\nabla (\kappa\nabla u)=f \quad  \text{on} \quad \mathcal{D}=[0,1]^2,
\] 
to model the observable state $u:\mathcal{D}\rightarrow\R$, e.g., pressure field in groundwater or voltage potential field in electrical tomography, given the diffusion property represented by $\kappa:\mathcal{D}\rightarrow\R_{>0}$. The goal is to estimate the unknown random field $\kappa$ from partial observations of the state $u$.
To demonstrate the application of conditional DIRT on problems with non-Gaussian priors, we model $\log(\kappa)$ by a Besov-$\mathcal{B}^r_{pp}$ random field prior~\cite{dashti2012besov,saksman2009discretization}. That is, given a wavelet basis $\{\psi^r_{j,k}\}$, where $j$ controls the scaling, $k$ controls the shift, and $r$ controls the decay of the weighting coefficients, we have $\log \kappa = c_0 + \sum_{\scriptscriptstyle j=0}^\lowsup{\scriptscriptstyle J} \sum_{\scriptscriptstyle k=0}^\lowsup{\scriptscriptstyle 2^\lowsup{\scriptscriptstyle j}-1} b_{j,k} \psi^\lowsup{r}_{j,k}$.
This way, $\log \kappa$ is determined by the vector $\theta=(c_0,b_{0,0},b_{0,1},\ldots) \in \mathbb{R}^\lowsup{n}$, \cancel{$n=2^\lowsup{J}+1$} \reone{$n=2^\lowsup{J+1}$}, which is equipped with a product-form prior
$
 \rho_{\Theta}(\theta) \propto \prod_{i=1}^\lowsup{n} \exp(-\gamma |\theta_i|^p).
$
We choose \reone{$J = 9$}, $r = 2$ and $p = 1$ in this example, so effectively the parameter $\theta$ \reone{is $n = 1024$ dimensional} and follows a Laplace prior. \cancel{We order $\theta_i$ according to Proposition~\ref{prop:bound_Hellinger}.}
The data are noisy observation of $u=u(\theta)$ evaluated at $m=16$ equispaced locations $\{s_i\}_{i=1}^\lowsup{m}\in\mathcal{D}$ perturbed by a centred Gaussian noise with the signal-to-noise-ratio $5$.

Realizations of posterior samples generated by the conditional DIRT for different observed data sets are presented in \ref{sec:besov2}. Here we focus on demonstrating the accuracy of conditional DIRT. 
\cancel{In Figure~7 (top), we let the TT-Cross approximation method truncate the TT ranks from the maximal values of $11$ down to the minimal values that provide the relative maximum error below $0.1$,
and plot $\text{rank}_\text{TT}(\mathcal{T}_{\ell-1}^\sharp \pi_{Y,\Theta}^{\ell})$ for each DIRT layer, as well as the average of TT ranks over $\ell$.}
The tempering parameters are $\beta_0=10^\lowsup{-4}$ and $\beta_{\ell+1} = \sqrt{10} \cdot \beta_{\ell}$.
\cancel{In the top left plot of Figure~7 we vary the order of variables, and in the top right plot we
vary the number of $\theta_i$ variables from $12$ to $22$.
We see that the order following Proposition~\ref{prop:bound_Hellinger} and \eqref{eq:g_TT} gives smallest TT ranks.
Moreover, in this ordering, the intermediate TT ranks are independent of the total number of variables.}
\retwo{We demonstrate the impact of variable reordering on TT ranks by considering three orderings: ({\romannumeral 1}) We first order the parameter $\theta$ and data $y$ according to Proposition~\ref{prop:bound_Hellinger}, and then truncate the parameter dimension to $t=22$. Here we do not truncate the data dimension. ({\romannumeral 2}) We reverse the order of the data given in ({\romannumeral 1}) but keep the order of the truncated parameter. ({\romannumeral 3}) We reverse the order of the truncated parameter given in ({\romannumeral 1}) but keep the same data order. For each ordering, we let the TT-Cross approximation method adapt the TT ranks (with the maximum rank set to $11$) to provide a relative maximum error below $0.1$. The left plot of Figure~\ref{fig:besov1} shows the average TT rank of each coordinate over DIRT layers. It clearly shows that TT ranks are significantly reduced towards the tails of the reordered variables compared to the reverse ordering cases. Then, we use the same ordering as in ({\romannumeral 1}) and use two truncated parameter dimensions $t = \{12,22\}$ to demonstrate the performance of variable ordering with different dimension truncations. The right plot of Figure~\ref{fig:besov1} shows the TT rank of each coordinate for each DIRT layer (thin lines), as well as the  average TT rank of each coordinate over DIRT layers (thick lines). Here we observe a similar rank reduction trend in TT towards the tails of the reordered variables in both cases. Overall, the reordering strategy presented in Section \ref{sec:Variable_ordering} indeed can provide useful insights into the importance of variables in this example. }

\begin{figure}[h!]
\centering
\noindent
\begin{tikzpicture}
\begin{axis}[%
width=0.52\linewidth,
height=0.40\linewidth,
xlabel={coordinate},
ylabel={TT rank},
xmin=0,xmax=38,
legend style={at={(0.99,0.01)},anchor=south east},
]

\addplot+[no marks,line width=1.2pt] table[header=false,x index=0,y index=10]{besov-ex5n-r11-d38.dat} node [pos=0.27,anchor=west,inner sep=0pt] {\footnotesize\;\;$y_{m} \mydots y_1, \theta_1 \mydots \theta_{t}$};
\addplot+[no marks,line width=1.5pt,dashed] table[header=false,x index=0,y index=10]{besov-ex5n-r11-d38-yperm.dat} node[pos=0.25,anchor=south] {\footnotesize$y_{1} \mydots y_m,\theta_1 \mydots \theta_{t}$};
\addplot+[no marks,line width=2.0pt,dotted] table[header=false,x index=0,y index=10]{besov-ex5n-r11-d38-thetaperm.dat} node[pos=0.77,anchor=south east,inner sep=1pt] {\footnotesize$y_{m} \mydots y_1, \theta_t \mydots \theta_{1}$};

\addplot+[black,no marks,solid] coordinates{(16,0) (16,11)} node [very near start] {$y\quad\theta$};
\end{axis}
\end{tikzpicture}
\hfill\noindent
\begin{tikzpicture}
\begin{axis}[%
width=0.52\linewidth,
height=0.40\linewidth,
xlabel={coordinate},
ylabel={TT rank},
xmin=0,xmax=38,
legend style={at={(0.01,0.99)},anchor=north west},
]
\addplot+[no marks,red,solid,opacity=0.5] table[header=false,x index=0,y index=1]{besov-ex5n-r11-d38.dat};
\addplot+[no marks,red,solid,opacity=0.5] table[header=false,x index=0,y index=2]{besov-ex5n-r11-d38.dat}; 
\addplot+[no marks,red,solid,opacity=0.5] table[header=false,x index=0,y index=3]{besov-ex5n-r11-d38.dat}; 
\addplot+[no marks,red,solid,opacity=0.5] table[header=false,x index=0,y index=4]{besov-ex5n-r11-d38.dat}; 
\addplot+[no marks,red,solid,opacity=0.5] table[header=false,x index=0,y index=5]{besov-ex5n-r11-d38.dat}; 
\addplot+[no marks,red,solid,opacity=0.5] table[header=false,x index=0,y index=6]{besov-ex5n-r11-d38.dat}; 
\addplot+[no marks,red,solid,opacity=0.5] table[header=false,x index=0,y index=7]{besov-ex5n-r11-d38.dat}; 
\addplot+[no marks,red,solid,opacity=0.5] table[header=false,x index=0,y index=8]{besov-ex5n-r11-d38.dat}; 
\addplot+[no marks,red,solid,opacity=0.5] table[header=false,x index=0,y index=9]{besov-ex5n-r11-d38.dat};

\addplot+[no marks,red,solid,line width=2pt] table[header=false,x index=0,y index=10]{besov-ex5n-r11-d38.dat} node [pos=0.74,anchor=south west,inner sep=0pt,opacity=1.0] {$t=22$ };

\addplot+[no marks,blue,solid,opacity=0.5] table[header=false,x index=0,y index=1]{besov-ex5n-r11-d28.dat}; 
\addplot+[no marks,blue,solid,opacity=0.5] table[header=false,x index=0,y index=2]{besov-ex5n-r11-d28.dat}; 
\addplot+[no marks,blue,solid,opacity=0.5] table[header=false,x index=0,y index=3]{besov-ex5n-r11-d28.dat}; 
\addplot+[no marks,blue,solid,opacity=0.5] table[header=false,x index=0,y index=4]{besov-ex5n-r11-d28.dat}; 
\addplot+[no marks,blue,solid,opacity=0.5] table[header=false,x index=0,y index=5]{besov-ex5n-r11-d28.dat}; 
\addplot+[no marks,blue,solid,opacity=0.5] table[header=false,x index=0,y index=6]{besov-ex5n-r11-d28.dat}; 
\addplot+[no marks,blue,solid,opacity=0.5] table[header=false,x index=0,y index=7]{besov-ex5n-r11-d28.dat}; 
\addplot+[no marks,blue,solid,opacity=0.5] table[header=false,x index=0,y index=8]{besov-ex5n-r11-d28.dat}; 
\addplot+[no marks,blue,solid,opacity=0.5] table[header=false,x index=0,y index=9]{besov-ex5n-r11-d28.dat}; 

\addplot+[no marks,blue,solid,line width=2pt] table[header=false,x index=0,y index=10]{besov-ex5n-r11-d28.dat} node[pos=0.95,anchor=east] {$t=12$};

\addplot+[black,no marks,solid] coordinates{(16,0) (16,11)} node [very near start] {$y\quad\theta$};
\end{axis}
\end{tikzpicture}
\caption{\retwo{Besov prior example. Left: average TT ranks over all DIRT layers versus variable coordinate for different variable ordering.  A truncated parameter dimension $t=22$ is used in the left plot. Right: individual (thin lines) and average (thick lines) TT ranks over DIRT layers with truncated parameter dimensions $t=12$ and $t=22$. Only the variable ordering suggested by Proposition \ref{prop:bound_Hellinger} is used for the right plot.}}
\label{fig:besov1}
\end{figure}
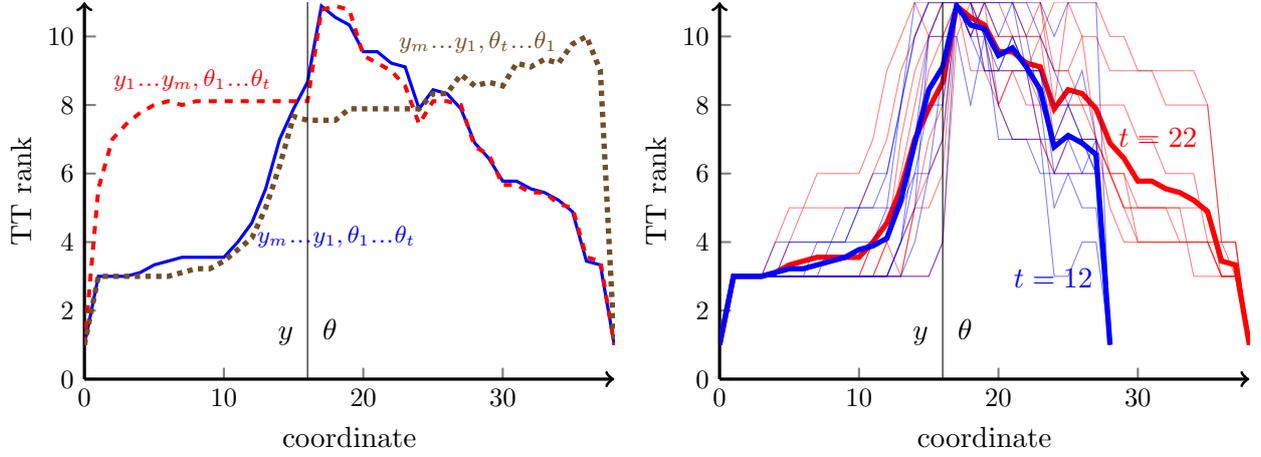

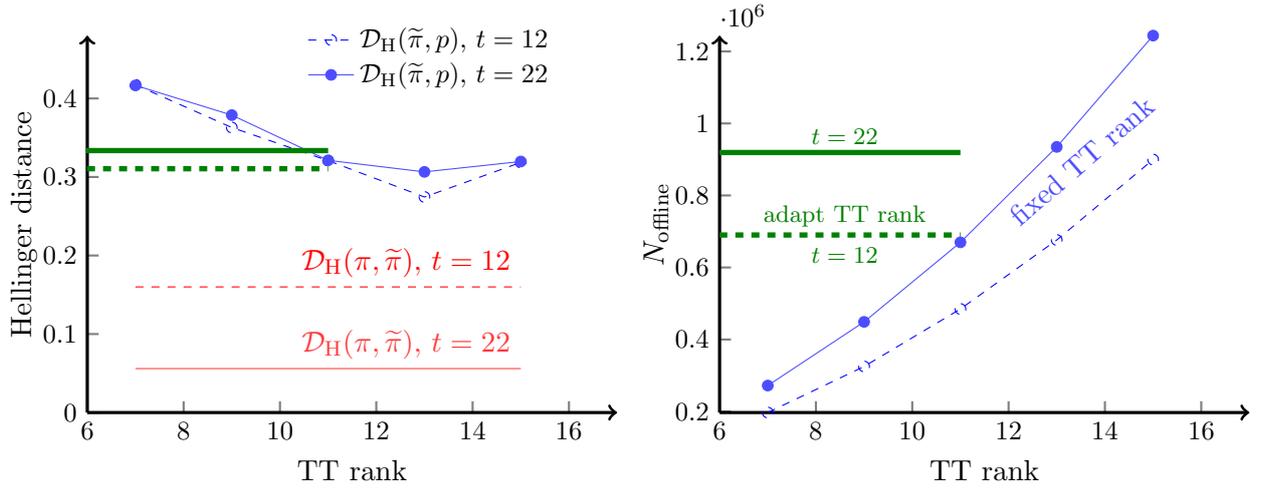
\begin{figure}[h!]
\centering
\noindent
\begin{tikzpicture}
\begin{axis}[%
width=0.52\linewidth,
height=0.40\linewidth,
xlabel={TT rank},
ylabel={Hellinger distance},
legend style={at={(0.4,1.05)},anchor=north west},
ymin=0,ymax=0.48,
xmin=6,xmax=17,
]
\addplot+[blue,dashed,mark=o] coordinates{
(7 ,      0.4174)
(9 ,      0.3630)
(11,      0.3208)
(13,      0.2745)
(15,      0.3185)
};  \addlegendentry{$\DH(\widetilde\pi,p),$ $t=12$};
\addplot+[blue!70!white,solid,mark=*,mark options={fill=blue!70!white}] coordinates{
(7 ,     0.4166)
(9 ,     0.3788)
(11,     0.3212)
(13,     0.3065)
(15,     0.3197)
}; \addlegendentry{$\DH(\widetilde\pi,p),$ $t=22$}; 
\addplot+[no marks,red,dashed,mark=triangle] coordinates{
(7 ,      0.16) %
(9 ,      0.16) %
(11,      0.16) %
(13,      0.16) %
(15,      0.16) %
} node[pos=1,anchor=south east] {$\DH(\pi,\widetilde\pi),$ $t=12$};  
\addplot+[no marks,red!70!pink,solid,mark=triangle*] coordinates{
(7 ,     0.056 )  %
(9 ,     0.056 )  %
(11,     0.056 )  %
(13,     0.056 )  %
(15,     0.056 )  %
} node[pos=1,anchor=south east] {$\DH(\pi,\widetilde\pi),$ $t=22$};  
\addplot+[green!50!black,dashed,no marks,domain=3:11,line width=2pt] {0*x+0.3105}; 
\addplot+[green!50!black,solid,no marks,domain=3:11,line width=2pt]  {0*x+0.3336};
\end{axis}
\end{tikzpicture}
\hfill\noindent
\begin{tikzpicture}
\begin{axis}[%
width=0.52\linewidth,
height=0.40\linewidth,
xlabel={TT rank},
ylabel={$N_{\rm offline}$},
xmin=6,xmax=17,
]
\addplot+[blue,dashed,mark=o] coordinates{
(7 ,        21896*9 )
(9 ,        36108*9 )
(11,        53856*9 )
(13,        75140*9 )
(15,        99960*9 )
};
\addplot+[blue!70!white,solid,mark=*,mark options={fill=blue!70!white}] coordinates{
(7 ,        30226*9 )
(9 ,        49878*9 )
(11,        74426*9 )
(13,       103870*9 )
(15,       138210*9 )
} node[midway,anchor=north west,rotate=41] {fixed TT rank};
\addplot+[green!50!black,dashed,no marks,domain=3:11,line width=2pt] {0*x+690149} node[pos=0.7,inner sep=8pt] {\footnotesize $\overset{\vphantom{\displaystyle\sum}\mbox{adapt TT rank}}{\vphantom{\displaystyle\sum}t=12}$};

\addplot+[green!50!black,solid,no marks,domain=3:11,line width=2pt]  {0*x+919190} node[pos=0.7,anchor=south,inner sep=2pt] {\footnotesize $t=22$};
\end{axis}
\end{tikzpicture}
\caption{\retwo{Besov prior example. Left: Hellinger errors between the DIRT approximation $p$ and the reduced joint density $\widetilde \pi$ with truncated parameter dimensions $t=12$ (dashed lines) and $t=22$ (solid lines), and between the reduced joint density and the original joint density $\pi$. Right: Number of density evaluations in the construction of DIRT. Green and blue lines indicate results of the adaptive TT rank strategy and the fixed TT rank strategy, respectively.}}
\label{fig:besov2}
\end{figure}

In Figure~\ref{fig:besov2} we compare the accuracy and complexity of the DIRT approximation using adaptive TT ranks and fixed TT ranks, where the previous DIRT layer is used as the initial guess for the TT-Cross method on the next layer. \retwo{Here we consider two truncated parameter dimensions of $t=\{12, 22\}$ after reordering (i) and use the full data dimension in both cases.} The adaptive-rank results are shown as green lines (dashed for $t=12$, solid for $t=22$) spanning the TT ranks from 1 to 11.
Fixed-rank results are shown as blue points at the corresponding rank values.
Moreover, red lines show \cancel{the Hellinger distances between the exact densities in reduced and full spaces.} \retwo{the Hellinger distances between the original density and the parameter reduced density, cf. \eqref{eq:api_with_independence_structure}.}
With a more aggressive parameter dimension truncation ($t=12$), the reduced dimensional joint density admits a slightly more accurate DIRT approximation compared to that of $t=22$. This comes at a price of a larger Hellinger distance towards the original density. 
The two strategies for the TT rank selection demonstrate comparable results by mutually compensating two aspects: the adaptive strategy allows one to reduce the ranks between the trailing variables, and hence the number of function evaluations per one TT-Cross iteration, but several cross iterations are needed by the adaptive-rank strategy, which inflates the total number of evaluations.
This suggests that for lower-dimensional problems a simple fixed-rank strategy may be more efficient for constructing DIRT. 
For higher-dimensional problems one may prefer a rank-adaptive strategy to reduce the cost contributed by variables at the tails of the ordering.

\section*{Acknowledgements}
TC acknowledges support from the Australian Research Council under the grant DP210103092. 
SD is thankful for the support from the Engineering and Physical Sciences Research Council New Investigator Award EP/T031255/1.
OZ acknowledges support from the ANR JCJC project MODENA (ANR-21-CE46-0006-01).

\appendix

\section{Proofs and derivations in Section \ref{sec:intro}}

\subsection{Useful lemmas}

The following lemmas will be used to establish the probabilistic error bound of the conditional map in \eqref{eq:Hellinger}.

\begin{lemma}\label{lemma:BoundExpectation}
Let $\pi_1$ and $\pi_2$ be two probability density functions such that $\DH(\pi_1 , \pi_2)\leq e<{1/ \sqrt2}$. For any function $h$ with finite variances $\Var_{\pi_1}(h)<\infty$ and $\Var_{\pi_2}(h)<\infty$, the expectation error satisfies the following inequality
 $$
  \left| \E_{\pi_1}(h) - \E_{\pi_2}(h)  \right| \leq \frac{\sqrt2 e}{1- \sqrt2 e} \left( \sqrt{\Var_{\pi_1}(h)}  + \sqrt{\Var_{\pi_2}(h)}\right) .
 $$
\end{lemma}
\begin{proof}
Defining a function $\tilde h(x) := h(x)-\E_{\pi_1 }(h)$, the mean squared error satisfies 
\begin{align*}
  & \hspace{-.5em}| \E_{\pi_1}(h) - \E_{\pi_2}(h)  | =  | \E_{\pi_1}(\tilde h) - \E_{\pi_2}(\tilde h)  | \\
  &= \left| \int \tilde h(x) \left(\sqrt{\pi_1(x)} + \sqrt{\pi_2(x)} \right)\left(\sqrt{\pi_1(x)} - \sqrt{\pi_2(x)} \right)  \d x\right| \\
  &\leq \left( \int \left(\sqrt{\pi_1(x)} - \sqrt{\pi_2(x)} \right)^2  \d x\right)^\frac12 \left( \int \tilde h(x)^2 \left(\sqrt{\pi_1(x)} + \sqrt{\pi_2(x)} \right)^2  \d x \right)^\frac12 \\
  &= \sqrt{2}\,\DH(\pi_1 , \pi_2) \left( \int \left(\tilde h(x)\sqrt{\pi_1(x)} + \tilde h(x)\sqrt{\pi_2(x)} \right)^2  \d x\right)^\frac12    \\
  &\leq \sqrt{2}\,\DH(\pi_1 , \pi_2) \left( \left( \int \left(  h(x)-\E_{\pi_1}(h) \right)^2 \pi_1(x) \d x\right)^\frac12 + \left( \int  \left(  h(x)-\E_{\pi_1}(h) \right)^2 \pi_2(x)  \d x\right)^\frac12  \right) \\
 & = \sqrt{2}\,\DH(\pi_1 , \pi_2) \left( \sqrt{\Var_{\pi_1}(h)} + \left(\Var_{\pi_2}(h) + |\E_{\pi_2}(h)-\E_{\pi_1}(h)|^2 \right)^\frac12  \right) \\
  &\leq \sqrt{2}\,\DH(\pi_1 , \pi_2) \left( \sqrt{\Var_{\pi_1}(h)} + \sqrt{\Var_{\pi_2}(h)} + |\E_{\pi_1}(h)-\E_{\pi_2}(h)| \right)
\end{align*}
Substituting the condition $\DH(\pi_1 , \pi_2)\leq e$, we obtain
\begin{align*}
  (1-\sqrt{2}e)| \E_{\pi_1}(h) - \E_{\pi_2}(h)  | \leq \sqrt2 e \left( \sqrt{\Var_{\pi_1}(h)}  + \sqrt{\Var_{\pi_2}(h)}\right),
\end{align*}
and, because $\sqrt{2}e<1$, the result follows.
\end{proof}

\begin{lemma}\label{lemma5232}
 For two joint probability densities $\pi_{Y,\Theta}$ and $\widetilde\pi_{Y,\Theta}$ and a given realization $y$, the expected  Hellinger distance between $\pi_{\Theta|Y}$ and $\widetilde\pi_{\Theta|Y}$ satisfies 
  $$
  \E_{Y\sim\pi_{Y}}\left[ \DH( \pi_{\Theta|Y}, \widetilde\pi_{\Theta|Y} ) \right]
  \leq 2 \DH( \pi_{Y,\Theta},\widetilde\pi_{Y,\Theta}  ).
 $$
\end{lemma}

\begin{proof}%
Applying Jensen's inequality, we have 
\[
\E_{Y\sim\pi_{Y}}\left[ \DH( \pi_{\Theta|Y}, \widetilde\pi_{\Theta|Y} ) \right] \leq \sqrt{ \E_{Y\sim\pi_{Y}}\left[ \DH( \pi_{\Theta|Y}, \widetilde\pi_{\Theta|Y} )^2 \right]}.
\]
The expected squared Hellinger distance between $\pi_{\Theta|Y}$ and $\widetilde\pi_{\Theta|Y}$ satisfies 
\begin{align*}
   & \hspace{-0.5em} \E_{Y\sim\pi_{Y}}[ \DH( \pi_{\Theta|Y},  \widetilde\pi_{\Theta|Y} )^2 ]\\
   &= \int \left( \frac12\int \left( \sqrt{\pi_{\Theta|Y}(\theta|y)}-\sqrt{\widetilde\pi_{\Theta|Y}(\theta|y)} \right)^2 \d \theta \right) \pi_{Y}(y) \d y\\
   &= \frac12  \int  \left( \sqrt{\pi_{Y,\Theta}(y,\theta)}-\sqrt{\widetilde\pi_{Y,\Theta}(y,\theta) \frac{\pi_{Y}(y)}{\widetilde\pi_{Y}(y)}} \right)^2 \!\! \d \theta  \d y\\
   &\leq \int  \left( \sqrt{\pi_{Y,\Theta}(y,\theta)}-\sqrt{\widetilde\pi_{Y,\Theta}(y,\theta)} \right)^2 \!\! \d \theta \d y 
   + \int  \left( \sqrt{\widetilde\pi_{Y,\Theta}(y,\theta)}-\sqrt{\widetilde\pi_{Y,\Theta}(y,\theta) \frac{\pi_{Y}(y)}{\widetilde\pi_{Y}(y)}} \right)^2 \!\! \d \theta \d y \\
   &= 2\DH( \pi_{Y,\Theta}, \widetilde\pi_{Y,\Theta} )^2 + \int  \left( \sqrt{\widetilde\pi_{Y}(y)}-\sqrt{\pi_{Y}(y)} \right)^2 \d y \\
   &= 2\DH( \pi_{Y,\Theta}, \widetilde\pi_{Y,\Theta} )^2 + 2\DH( \pi_{Y}, \widetilde\pi_{Y} )^2.
\end{align*}
The squared Hellinger distance between the marginal probability densities satisfies
\begin{align*}
   \DH( \pi_{Y}, \widetilde\pi_{Y} )^2
   &= \frac12 \int  \left( \sqrt{\frac{\widetilde\pi_{Y}(y)}{\pi_{Y}(y)}}-1 \right)^2  \pi_{Y}(y) \d y \\
   &= \frac12 \int  \left( \sqrt{\int \frac{\widetilde\pi_{Y,\Theta}(y,\theta)}{\pi_{Y,\Theta}(y,\theta)} \pi_{\Theta|Y}(\theta|y) \d \theta}- \sqrt{\int \pi_{\Theta|Y}(\theta|y) \d \theta} \right)^2  \pi_{Y}(y) \d y \\
   &\leq \frac12 \int \left(\int  \left( \sqrt{ \frac{\widetilde\pi_{Y,\Theta}(y,\theta)}{\pi_{Y,\Theta}(y,\theta)} }-1 \right)^2  \pi_{\Theta|Y}(\theta|y) \d \theta \right) \pi_{Y}(y) \d y \\
   &= \frac12 \int  \left( \sqrt{ \frac{\widetilde\pi_{Y,\Theta}(y,\theta)}{\pi_{Y,\Theta}(y,\theta)} }-1 \right)^2  \pi_{Y,\Theta}(y,\theta)  \d \theta \d y \\
   &= \DH( \pi_{Y,\Theta}, \widetilde\pi_{Y,\Theta} )^2,
\end{align*}
where the inequality follows from Jensen's inequality. This way, we have 
\[
\E_{Y\sim\pi_{Y}}[ \DH( \pi_{\Theta|Y}, \widetilde\pi_{\Theta|Y} )^2] \leq 4\DH( \pi_{Y,\Theta},\widetilde\pi_{Y,\Theta}  )^2,
\]
and thus the result follows. 
\end{proof}

\subsection{A probabilistic error bound of the conditional transport}\label{proof:MarkovBound}

Suppose the Hellinger distance between the joint probability densities $\pi_{Y,\Theta}$ and $\widetilde\pi_{Y,\Theta}$ satisfies $\DH( \pi_{Y,\Theta}, \widetilde\pi_{Y,\Theta} ) \leq \varepsilon$ for some $\varepsilon < \sqrt2/4$. 
For any $\tau \in [0,1]$, applying Markov's inequality and Lemma \ref{lemma5232} yields
\begin{align*}
 \mathbb{P}_{Y\sim\pi_{Y}}\left[ \DH( \pi_{\Theta|Y}, \widetilde\pi_{\Theta|Y} ) \leq \tau \right]
 &= 1 - \mathbb{P}_{Y\sim\pi_{Y}}\left[ \DH( \pi_{\Theta|Y}, \widetilde\pi_{\Theta|Y} ) > \tau \right] \\
 &> 1 - \frac{\E_{Y\sim\pi_{Y}}[ \DH( \pi_{\Theta|Y}, \widetilde\pi_{\Theta|Y} ) ] }{\tau}\\
 &\geq 1 - \frac{2\,\DH( \pi_{Y,\Theta},\widetilde\pi_{Y,\Theta}  )}{\tau} 
 \geq 1-\frac{2\,\varepsilon}{\tau}.
\end{align*}
Let $\delta = 2\varepsilon/\tau$, we deduce that $ \DH( \pi_{\Theta|Y}, \widetilde\pi_{\Theta|Y} ) \leq 2\varepsilon / \delta$ holds with probability greater than $1-\delta$. 
Applying Lemma \ref{lemma:BoundExpectation}, with probability greater than $1-\delta$, we have the event 
$$
  \frac{| \E_{\pi_{\Theta|Y}}(h) - \E_{\widetilde\pi_{\Theta|Y}}(h)|}{\sqrt{ \Var_{\pi_{\Theta|Y}}(h)} + \sqrt{\Var_{\widetilde\pi_{\Theta|Y}}(h)}}
  \leq \frac{2 \sqrt2 \,\varepsilon/\delta}{ 1-2 \sqrt2 \,\varepsilon/\delta}
  =\frac{4 \varepsilon}{ \sqrt2 \delta -4 \varepsilon}.
$$
This concludes the result. \hfill $\square$


\section{Further details of Section \ref{sec:background}}

\subsection{TT-Cross approximation of functions}\label{sec:cross}
To simplify the notation, we aggregate all random variables into one random vector
\begin{equation}
X := (X_{1},\ldots,X_m,X_{m+1},\ldots,X_{m+n}) \equiv (Y_m,\ldots,Y_1,\Theta_1,\ldots,\Theta_n),
\end{equation}
and denote its corresponding realization $x = (x_1,\ldots,x_{m+n})$.
The problem of constructing a TT approximation of an arbitrary function $q(x)$ (keeping in mind the equivalence $q(x) = \sqrt{\pi_{Y,\Theta}(y,\theta)}$) can be written as
\begin{equation}
 q(x) \approx g(x)  := \mG_{1}(x_{1})   \hdots  \mG_{m+n}(x_{m+n}),
\end{equation}
where $\mG_{k}(x_k) \in \mathbb{R}^{r_{k-1} \times r_k}$, with $r_0=r_{m+n}=1$.
For practical computations, we introduce univariate bases $\{\phi_k^{(i)}(x_{k})\}_{i=1}^{I_k}$, $k=1,\ldots,m+n$, in which we expand the corresponding TT cores,
\begin{align}
 \mG_{k}^{(\alpha_{k-1},\alpha_k)}(x_{k}) &= \sum_{i=1}^{I_k} \phi_k^{(i)}(x_{k}) \tA_k[\alpha_{k-1},i,\alpha_k].
\end{align}
Here, $I_k \in \mathbb{N}$ is a dimension of the basis in the $k$th variable.
Note that one only needs to store the coefficient tensors $\{\tA_k\}$.
Counting their cardinalities,
we obtain the TT storage complexity $\sum_{k=1}^{m+n} r_{k-1} I_k r_k$.
Introducing upper bounds $r:=\max_k r_k$, $I:=\max_k I_k$,
we arrive at a short estimate $\mathcal{O}((n+m)Ir^2)$, linear in the number of variables.

To compute the coefficients from interpolation, we consider sets of {\bf collocation points} $\mathtt{X}_{k}:=\{x_{k}^{(i)}\}_{i=1}^{I_k}$ such that the Vandermonde matrices
$
\Phi_k(\mathtt{X}_{k}):=\left[\phi_k^{(j)}(x_{k}^{(i)})\right]_{i,j=1}^{I_k}
$
are invertible.

The TT-Cross is an alternating direction algorithm, which iterates over the TT cores, $k=1,\ldots, m+n$, and in each step it computes the coefficients $\tA_k$ by solving a system of interpolation equations
$$
g(x^{(j)}) = \tilde g(x^{(j)}) \qquad \forall x^{(j)} \in \mathtt{X}_{<k>} \subset \mathbb{R}^{m+n},
$$
where $\mathtt{X}_{<k>}$ is a set of samples of cardinality $r_{k-1}I_k r_k$ and structure described below,
ensuring the unique resolution of those equations.
The alternating iteration is beneficial, since the TT decomposition is {\bf linear} with respect to each individual tensor $\tA_k$.
Indeed, we can write $\tilde g(x^{(j)})$ in the form
\canceleq{\begin{equation*}
\msout{
  \tilde g(x^{(j)}) \equiv \sum_{\alpha_{k-1},i,\alpha_k} \left[\mG_{\le k-1}^{(\alpha_{k-1})}(x^{(j)}_{\le k-1}) \phi_k^{(i)}(x_k^{(j)}) \mG_{>k}^{(\alpha_k)}(x^{(j)}_{>k})\right] \cdot \tA_k[\alpha_{k-1},i,\alpha_k] = g(x^{(j)}),
}
\end{equation*}}
\begin{equation}\label{eq:frame-linear}
\reone{
 \tilde g(x^{(j)}) \equiv \sum_{\alpha_{k-1},i,\alpha_k} \left[\mG_{1: k-1}^{(\alpha_{k-1})}(x^{(j)}_{1:k-1}) \phi_k^{(i)}(x_k^{(j)}) \mG_{k+1:m+n}^{(\alpha_k)}(x^{(j)}_{k+1:m+n})\right] \cdot \tA_k[\alpha_{k-1},i,\alpha_k] = g(x^{(j)}),
}
\end{equation}
where we expand the notation to let $x_{1:0} = x_{m+n+1:m+n} = \emptyset$.
Note that \eqref{eq:frame-linear} resembles linear equations.
We can {\bf reshape} the tensor $\tA_k$ into a vector $a_k\in\mathbb{R}^{r_{k-1}I_k r_k}$ with the same elements,
\begin{equation}\label{eq:core-vec}
a_k[\alpha_{k-1} + (i-1) r_{k-1} + (\alpha_k-1) r_{k-1} I_k] = \tA_k[\alpha_{k-1},i,\alpha_k],
\end{equation}
and introduce a matrix $G_{\neq k} \in \mathbb{R}^{r_{k-1}I_k r_k \times r_{k-1}I_k r_k}$ with elements
\canceleq{\begin{equation*}
\msout{
  G_{\neq k}[j, \alpha_{k-1} + (i-1) r_{k-1} + (\alpha_k-1) r_{k-1} I_k] = \mG_{\le k-1}^{(\alpha_{k-1})}(x^{(j)}_{\le k-1}) \phi_k^{(i)}(x_k^{(j)}) \mG_{>k}^{(\alpha_k)}(x^{(j)}_{>k}).
}
\end{equation*}}
\begin{equation}\label{eq:frame}
\reone{
  G_{\neq k}[j, \alpha_{k-1} + (i-1) r_{k-1} + (\alpha_k-1) r_{k-1} I_k] = \mG_{1:k-1}^{(\alpha_{k-1})}(x^{(j)}_{1:k-1}) \phi_k^{(i)}(x_k^{(j)}) \mG_{k+1:m+n}^{(\alpha_k)}(x^{(j)}_{k+1:m+n}).
}
\end{equation}
Now the square linear system
$
G_{\neq k} a_k = g(\mathtt{X}_{<k>})
$
becomes more apparent.
We can also notice that the right hand side requires $\mathcal{O}(Ir^2)$ samples of the desired density function.

The construction and conditioning of \eqref{eq:frame} can be simplified dramatically if we endow the set $\mathtt{X}_{<k>}$ with a special structure, and recall that equations \eqref{eq:frame-linear} are solved in the course of consecutive iterations over $k$.
First, let us restrict $\mathtt{X}_{<k>}$ to a {\bf Cartesian} form
\begin{equation}\label{eq:Cart-set}
\mathtt{X}_{<k>} = \tilde{\mathtt{X}}_{1:k-1} \times \mathtt{X}_{k} \times \tilde{\mathtt{X}}_{k+1:m+n},
\end{equation}
where in turn $\tilde{\mathtt{X}}_{1:k-1},\tilde{\mathtt{X}}_{k+1:m+n}$ contain samples of fewer coordinates,
\begin{equation}\label{eq:Xiface}
  \tilde{\mathtt{X}}_{1:k-1}:=\left\{(x_{1}^{(\alpha_{k-1})},\ldots,x_{k-1}^{(\alpha_{k-1})})\right\}_{\alpha_{k-1}=1}^{r_{k-1}}, \qquad
  \tilde{\mathtt{X}}_{k+1:m+n}:=\left\{(x_{k+1}^{(\alpha_{k})},\ldots,x_{m+n}^{(\alpha_{k})})\right\}_{\alpha_{k}=1}^{r_{k}}.
\end{equation}
We can say that a union of the sets $\mathtt{X}_{<k>}$ has a shape of a {\bf cross}, since each set $\mathtt{X}_{<k>}$ samples the entire collocation set in the $k$th variable, and only selected points in the other coordinates.
Hence the name of the algorithm.
Examples for two and three variables are shown in Figure \ref{fig:cross}.

\begin{figure}
\centering
\begin{tikzpicture}[x=0.1\linewidth,y=-0.1\linewidth]
\draw[line width=1pt] (0,0)--(2,0)--(2,2)--(0,2)--(0,0);

\draw[fill=blue] (0.5,0)--(0.8,0)--(0.8,2)--(0.5,2)--(0.5,0);
\node[anchor=north west] at (0.8, 0) {$\mathtt{X}_{<1>}$};
\draw[] (0.6,0)--(0.6,2);
\draw[] (0.7,0)--(0.7,2);

\draw[fill=red,opacity=0.5] (0,1)--(2,1)--(2,1.3)--(0,1.3)--(0,1);
\node[anchor=north east] at (2, 1.3) {$\mathtt{X}_{<2>}$};
\draw[] (0,1.1)--(2,1.1);
\draw[] (0,1.2)--(2,1.2);
\end{tikzpicture}
\hspace{4em}
\begin{tikzpicture}[x=0.1\linewidth,y=-0.1\linewidth,z=0.05\linewidth]


\draw[draw=black,fill=green!50!black,opacity=0.5] (0.5,1.1,0.2)--(0.5,1.2,0.2)--(0.5,1.2,1.5)--(0.5,1.1,1.5)--(0.5,1.1,0.2);
\draw[draw=black,fill=green!50!black,opacity=0.5] (0.5,1.2,0.2)--(0.6,1.2,0.2)--(0.6,1.2,1.5)--(0.5,1.2,1.5)--(0.5,1.2,0.2);
\draw[draw=black,fill=green!50!black,opacity=0.5] (0.5,1.1,0.2)--(0.6,1.1,0.2)--(0.6,1.1,1.5)--(0.5,1.1,1.5)--(0.5,1.1,0.2);
\draw[draw=black,fill=green!50!black,opacity=0.5] (0.6,1.1,0.2)--(0.6,1.2,0.2)--(0.6,1.2,1.5)--(0.6,1.1,1.5)--(0.6,1.1,0.2);
\draw[draw=black,fill=green!50!black,opacity=0.5] (0.5,1,0.2)--(0.5,1.1,0.2)--(0.5,1.1,1.5)--(0.5,1,1.5)--(0.5,1,0.2);
\draw[draw=black,fill=green!50!black,opacity=0.5] (0.5,1.1,0.2)--(0.6,1.1,0.2)--(0.6,1.1,1.5)--(0.5,1.1,1.5)--(0.5,1.1,0.2);
\draw[draw=black,fill=green!50!black,opacity=0.5] (0.5,1,0.2)--(0.6,1,0.2)--(0.6,1,1.5)--(0.5,1,1.5)--(0.5,1,0.2);
\draw[draw=black,fill=green!50!black,opacity=0.5] (0.6,1,0.2)--(0.6,1.1,0.2)--(0.6,1.1,1.5)--(0.6,1,1.5)--(0.6,1,0.2);
\draw[draw=black,fill=green!50!black,opacity=0.5] (0.6,1.1,0.2)--(0.6,1.2,0.2)--(0.6,1.2,1.5)--(0.6,1.1,1.5)--(0.6,1.1,0.2);
\draw[draw=black,fill=green!50!black,opacity=0.5] (0.6,1.2,0.2)--(0.7,1.2,0.2)--(0.7,1.2,1.5)--(0.6,1.2,1.5)--(0.6,1.2,0.2);
\draw[draw=black,fill=green!50!black,opacity=0.5] (0.6,1.1,0.2)--(0.7,1.1,0.2)--(0.7,1.1,1.5)--(0.6,1.1,1.5)--(0.6,1.1,0.2);
\draw[draw=black,fill=green!50!black,opacity=0.5] (0.7,1.1,0.2)--(0.7,1.2,0.2)--(0.7,1.2,1.5)--(0.7,1.1,1.5)--(0.7,1.1,0.2);
\draw[draw=black,fill=green!50!black,opacity=0.5] (0.6,1,0.2)--(0.6,1.1,0.2)--(0.6,1.1,1.5)--(0.6,1,1.5)--(0.6,1,0.2);
\draw[draw=black,fill=green!50!black,opacity=0.5] (0.6,1.1,0.2)--(0.7,1.1,0.2)--(0.7,1.1,1.5)--(0.6,1.1,1.5)--(0.6,1.1,0.2);
\draw[draw=black,fill=green!50!black,opacity=0.5] (0.6,1,0.2)--(0.7,1,0.2)--(0.7,1,1.5)--(0.6,1,1.5)--(0.6,1,0.2);
\draw[draw=black,fill=green!50!black,opacity=0.5] (0.7,1,0.2)--(0.7,1.1,0.2)--(0.7,1.1,1.5)--(0.7,1,1.5)--(0.7,1,0.2);

\node[anchor=north west] at (0.7, 1, 1.5) {$\mathtt{X}_{<3>}$};

\draw[draw=black,fill=blue,opacity=0.5] (0.5,1.2,0.1)--(0.5,1.2,0.2)--(0.5,2,0.2)--(0.5,2,0.1)--(0.5,1.2,0.1);
\draw[draw=black,fill=blue,opacity=0.5] (0.5,1.2,0.2)--(0.6,1.2,0.2)--(0.6,2,0.2)--(0.5,2,0.2)--(0.5,1.2,0.2);
\draw[draw=black,fill=blue,opacity=0.5] (0.5,1.2,0.1)--(0.6,1.2,0.1)--(0.6,2,0.1)--(0.5,2,0.1)--(0.5,1.2,0.1);
\draw[draw=black,fill=blue,opacity=0.5] (0.6,0,0.1)--(0.6,1.2,0.2)--(0.6,2,0.2)--(0.6,2,0.1)--(0.6,1.2,0.1);
\draw[draw=black,fill=blue,opacity=0.5] (0.5,1.2,0)--(0.5,1.2,0.1)--(0.5,2,0.1)--(0.5,2,0)--(0.5,1.2,0);
\draw[draw=black,fill=blue,opacity=0.5] (0.5,1.2,0.1)--(0.6,1.2,0.1)--(0.6,2,0.1)--(0.5,2,0.1)--(0.5,1.2,0.1);
\draw[draw=black,fill=blue,opacity=0.5] (0.5,1.2,0)--(0.6,1.2,0)--(0.6,2,0)--(0.5,2,0)--(0.5,1.2,0);
\draw[draw=black,fill=blue,opacity=0.5] (0.6,1.2,0)--(0.6,1.2,0.1)--(0.6,2,0.1)--(0.6,2,0)--(0.6,1.2,0);
\draw[draw=black,fill=blue,opacity=0.5] (0.6,1.2,0)--(0.6,1.2,0.1)--(0.6,2,0.1)--(0.6,2,0)--(0.6,1.2,0);
\draw[draw=black,fill=blue,opacity=0.5] (0.6,1.2,0.1)--(0.7,1.2,0.1)--(0.7,2,0.1)--(0.6,2,0.1)--(0.6,1.2,0.1);
\draw[draw=black,fill=blue,opacity=0.5] (0.6,1.2,0)--(0.7,1.2,0)--(0.7,2,0)--(0.6,2,0)--(0.6,1.2,0);
\draw[draw=black,fill=blue,opacity=0.5] (0.7,1.2,0)--(0.7,1.2,0.1)--(0.7,2,0.1)--(0.7,2,0)--(0.7,1.2,0);
\draw[draw=black,fill=blue,opacity=0.5] (0.6,1.2,0.1)--(0.6,1.2,0.2)--(0.6,2,0.2)--(0.6,2,0.1)--(0.6,1.2,0.1);
\draw[draw=black,fill=blue,opacity=0.5] (0.6,1.2,0.2)--(0.7,1.2,0.2)--(0.7,2,0.2)--(0.6,2,0.2)--(0.6,1.2,0.2);
\draw[draw=black,fill=blue,opacity=0.5] (0.6,1.2,0.1)--(0.7,1.2,0.1)--(0.7,2,0.1)--(0.6,2,0.1)--(0.6,1.2,0.1);
\draw[draw=black,fill=blue,opacity=0.5] (0.7,1.2,0.1)--(0.7,1.2,0.2)--(0.7,2,0.2)--(0.7,2,0.1)--(0.7,1.2,0.1);

\draw[draw=black,fill=red,opacity=0.5] (0,1.1,0.2)--(2,1.1,0.2)--(2,1.2,0.2)--(0,1.2,0.2)--(0,1.1,0.2);
\draw[draw=black,fill=red,opacity=0.5] (0,1.2,0.1)--(2,1.2,0.1)--(2,1.2,0.2)--(0,1.2,0.2)--(0,1.2,0.1);
\draw[draw=black,fill=red,opacity=0.5] (0,1.1,0.1)--(2,1.1,0.1)--(2,1.1,0.2)--(0,1.1,0.2)--(0,1.1,0.1);
\draw[draw=black,fill=red,opacity=0.5] (0,1.1,0.1)--(2,1.1,0.1)--(2,1.2,0.1)--(0,1.2,0.1)--(0,1.1,0.1);
\draw[draw=black,fill=red,opacity=0.5] (0,1.1,0.1)--(2,1.1,0.1)--(2,1.2,0.1)--(0,1.2,0.1)--(0,1.1,0.1);
\draw[draw=black,fill=red,opacity=0.5] (0,1.2,0)--(2,1.2,0)--(2,1.2,0.1)--(0,1.2,0.1)--(0,1.2,0);
\draw[draw=black,fill=red,opacity=0.5] (0,1.1,0)--(2,1.1,0)--(2,1.1,0.1)--(0,1.1,0.1)--(0,1.1,0);
\draw[draw=black,fill=red,opacity=0.5] (0,1.1,0)--(2,1.1,0)--(2,1.2,0)--(0,1.2,0)--(0,1.1,0);
\draw[draw=black,fill=red,opacity=0.5] (0,1,0.2)--(2,1,0.2)--(2,1.1,0.2)--(0,1.1,0.2)--(0,1,0.2);
\draw[draw=black,fill=red,opacity=0.5] (0,1.1,0.1)--(2,1.1,0.1)--(2,1.1,0.2)--(0,1.1,0.2)--(0,1.1,0.1);
\draw[draw=black,fill=red,opacity=0.5] (0,1,0.1)--(2,1,0.1)--(2,1,0.2)--(0,1,0.2)--(0,1,0.1);
\draw[draw=black,fill=red,opacity=0.5] (0,1,0.1)--(2,1,0.1)--(2,1.1,0.1)--(0,1.1,0.1)--(0,1,0.1);
\draw[draw=black,fill=red,opacity=0.5] (0,1,0.1)--(2,1,0.1)--(2,1.1,0.1)--(0,1.1,0.1)--(0,1,0.1);
\draw[draw=black,fill=red,opacity=0.5] (0,1.1,0)--(2,1.1,0)--(2,1.1,0.1)--(0,1.1,0.1)--(0,1.1,0);
\draw[draw=black,fill=red,opacity=0.5] (0,1,0)--(2,1,0)--(2,1,0.1)--(0,1,0.1)--(0,1,0);
\draw[draw=black,fill=red,opacity=0.5] (0,1,0)--(2,1,0)--(2,1.1,0)--(0,1.1,0)--(0,1,0);

\node[anchor=north east] at (2, 1.2, 0) {$\mathtt{X}_{<2>}$};

\draw[draw=black,fill=blue,opacity=0.5] (0.5,0,0.1)--(0.5,0,0.2)--(0.5,1,0.2)--(0.5,1,0.1)--(0.5,0,0.1);
\draw[draw=black,fill=blue,opacity=0.5] (0.5,0,0.2)--(0.6,0,0.2)--(0.6,1,0.2)--(0.5,1,0.2)--(0.5,0,0.2);
\draw[draw=black,fill=blue,opacity=0.5] (0.5,0,0.1)--(0.6,0,0.1)--(0.6,1,0.1)--(0.5,1,0.1)--(0.5,0,0.1);
\draw[draw=black,fill=blue,opacity=0.5] (0.6,0,0.1)--(0.6,0,0.2)--(0.6,1,0.2)--(0.6,1,0.1)--(0.6,0,0.1);
\draw[draw=black,fill=blue,opacity=0.5] (0.5,0,0)--(0.5,0,0.1)--(0.5,1,0.1)--(0.5,1,0)--(0.5,0,0);
\draw[draw=black,fill=blue,opacity=0.5] (0.5,0,0.1)--(0.6,0,0.1)--(0.6,1,0.1)--(0.5,1,0.1)--(0.5,0,0.1);
\draw[draw=black,fill=blue,opacity=0.5] (0.5,0,0)--(0.6,0,0)--(0.6,1,0)--(0.5,1,0)--(0.5,0,0);
\draw[draw=black,fill=blue,opacity=0.5] (0.6,0,0)--(0.6,0,0.1)--(0.6,1,0.1)--(0.6,1,0)--(0.6,0,0);
\draw[draw=black,fill=blue,opacity=0.5] (0.6,0,0)--(0.6,0,0.1)--(0.6,1,0.1)--(0.6,1,0)--(0.6,0,0);
\draw[draw=black,fill=blue,opacity=0.5] (0.6,0,0.1)--(0.7,0,0.1)--(0.7,1,0.1)--(0.6,1,0.1)--(0.6,0,0.1);
\draw[draw=black,fill=blue,opacity=0.5] (0.6,0,0)--(0.7,0,0)--(0.7,1,0)--(0.6,1,0)--(0.6,0,0);
\draw[draw=black,fill=blue,opacity=0.5] (0.7,0,0)--(0.7,0,0.1)--(0.7,1,0.1)--(0.7,1,0)--(0.7,0,0);
\draw[draw=black,fill=blue,opacity=0.5] (0.6,0,0.1)--(0.6,0,0.2)--(0.6,1,0.2)--(0.6,1,0.1)--(0.6,0,0.1);
\draw[draw=black,fill=blue,opacity=0.5] (0.6,0,0.2)--(0.7,0,0.2)--(0.7,1,0.2)--(0.6,1,0.2)--(0.6,0,0.2);
\draw[draw=black,fill=blue,opacity=0.5] (0.6,0,0.1)--(0.7,0,0.1)--(0.7,1,0.1)--(0.6,1,0.1)--(0.6,0,0.1);
\draw[draw=black,fill=blue,opacity=0.5] (0.7,0,0.1)--(0.7,0,0.2)--(0.7,1,0.2)--(0.7,1,0.1)--(0.7,0,0.1);

\node[anchor=north east] at (0.6, 0, 0) {$\mathtt{X}_{<1>}$};

\draw[draw=black,fill=green!50!black,opacity=0.5] (0.5,1.1,-0.7)--(0.5,1.2,-0.7)--(0.5,1.2,0.0)--(0.5,1.1,0.0)--(0.5,1.1,-0.7);  
\draw[draw=black,fill=green!50!black,opacity=0.5] (0.5,1.2,-0.7)--(0.6,1.2,-0.7)--(0.6,1.2,0.0)--(0.5,1.2,0.0)--(0.5,1.2,-0.7);
\draw[draw=black,fill=green!50!black,opacity=0.5] (0.5,1.1,-0.7)--(0.6,1.1,-0.7)--(0.6,1.1,0.0)--(0.5,1.1,0.0)--(0.5,1.1,-0.7);
\draw[draw=black,fill=green!50!black,opacity=0.5] (0.6,1.1,-0.7)--(0.6,1.2,-0.7)--(0.6,1.2,0.0)--(0.6,1.1,0.0)--(0.6,1.1,-0.7);
\draw[draw=black,fill=green!50!black,opacity=0.5] (0.5,1,-0.7)--(0.5,1.1,-0.7)--(0.5,1.1,0.0)--(0.5,1,0.0)--(0.5,1,-0.7);
\draw[draw=black,fill=green!50!black,opacity=0.5] (0.5,1.1,-0.7)--(0.6,1.1,-0.7)--(0.6,1.1,0.0)--(0.5,1.1,0.0)--(0.5,1.1,-0.7);
\draw[draw=black,fill=green!50!black,opacity=0.5] (0.5,1,-0.7)--(0.6,1,-0.7)--(0.6,1,0.0)--(0.5,1,0.0)--(0.5,1,-0.7);
\draw[draw=black,fill=green!50!black,opacity=0.5] (0.6,1,-0.7)--(0.6,1.1,-0.7)--(0.6,1.1,0.0)--(0.6,1,0.0)--(0.6,1,-0.7);
\draw[draw=black,fill=green!50!black,opacity=0.5] (0.6,1.1,-0.7)--(0.6,1.2,-0.7)--(0.6,1.2,0.0)--(0.6,1.1,0.0)--(0.6,1.1,-0.7);
\draw[draw=black,fill=green!50!black,opacity=0.5] (0.6,1.2,-0.7)--(0.7,1.2,-0.7)--(0.7,1.2,0.0)--(0.6,1.2,0.0)--(0.6,1.2,-0.7);
\draw[draw=black,fill=green!50!black,opacity=0.5] (0.6,1.1,-0.7)--(0.7,1.1,-0.7)--(0.7,1.1,0.0)--(0.6,1.1,0.0)--(0.6,1.1,-0.7);
\draw[draw=black,fill=green!50!black,opacity=0.5] (0.7,1.1,-0.7)--(0.7,1.2,-0.7)--(0.7,1.2,0.0)--(0.7,1.1,0.0)--(0.7,1.1,-0.7);
\draw[draw=black,fill=green!50!black,opacity=0.5] (0.6,1,-0.7)--(0.6,1.1,-0.7)--(0.6,1.1,0.0)--(0.6,1,0.0)--(0.6,1,-0.7);
\draw[draw=black,fill=green!50!black,opacity=0.5] (0.6,1.1,-0.7)--(0.7,1.1,-0.7)--(0.7,1.1,0.0)--(0.6,1.1,0.0)--(0.6,1.1,-0.7);
\draw[draw=black,fill=green!50!black,opacity=0.5] (0.6,1,-0.7)--(0.7,1,-0.7)--(0.7,1,0.0)--(0.6,1,0.0)--(0.6,1,-0.7);
\draw[draw=black,fill=green!50!black,opacity=0.5] (0.7,1,-0.7)--(0.7,1.1,-0.7)--(0.7,1.1,0.0)--(0.7,1,0.0)--(0.7,1,-0.7);

\end{tikzpicture}
\caption{Example of sample sets of TT-Cross in two (left) and three (right) dimensions. }\label{fig:cross}
\end{figure}
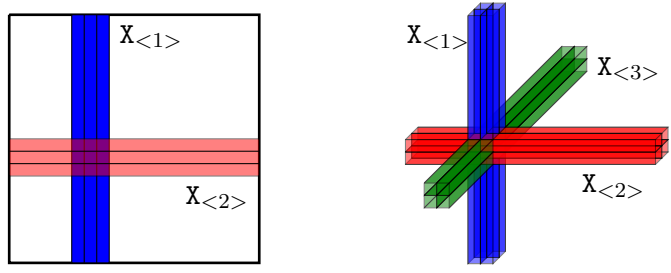

The Cartesian set \eqref{eq:Cart-set} allows us to factorize \eqref{eq:frame} in the Kronecker form,
$$
G_{\neq k} = \mG_{1:k-1}(\tilde{\mathtt{X}}_{1:k-1}) \otimes \Phi_k \otimes  \mG_{k+1:m+n}(\tilde{\mathtt{X}}_{k+1:m+n}).
$$
Moreover, we construct Lagrangian bases ensuring that $\mG_{1:k-1}(\tilde{\mathtt{X}}_{1:k-1})$ and $\mG_{k+1:m+n}(\tilde{\mathtt{X}}_{k+1:m+n})$ are identity matrices.

\begin{lemma}[\cite{oseledets2010tt}]
Assume that the sets \eqref{eq:Xiface} are {\bf nested},
$$
\tilde{\mathtt{X}}_{1:k} \subset \tilde{\mathtt{X}}_{1:k-1} \times \mathtt{X}_k, \quad \mbox{and} \quad
\tilde{\mathtt{X}}_{k:m+n} \subset \mathtt{X}_k \times \tilde{\mathtt{X}}_{k+1:m+n}.
$$
Then one can choose them such that $\mG_{1:k}(\mathtt{X}_{1:k}) = \mathtt{Id}_{r_{k}}$ and $\mG_{k:m+n}(\mathtt{X}_{k:m+n}) = \mathtt{Id}_{r_{k-1}}$.
\end{lemma}
\begin{proof}
To set up the base of induction, consider $\tilde{\mathtt{X}}_{1} \subset \mathtt{X}_{1}$.
Any given TT core $\mG_{1}(x_{1})$ can be evaluated on the set $\mathtt{X}_{1}$, giving a matrix $\mG_{1}(\mathtt{X}_{1}) \in \mathbb{R}^{I_{1} \times r_{1}}$.
Moreover, for a non-redundant TT decomposition it holds $r_{1} \le I_{1}$.
We can compute a thin QR decomposition $\mQ\mR = \mG_{1}(\mathtt{X}_{1})$ with a tall matrix $\mQ \in \mathbb{R}^{I_{1} \times r_{1}}$,
and search for an index set $\mathtt{J}_{1}\subset \{1,\ldots,I_{1}\}$ of $r_{1}$ linearly independent rows $\mQ[\mathtt{J}_{1}, :]$.
In fact, the Maximum Volume (\texttt{maxvol}) algorithm \cite{goreinov1997pseudo} allows one to select $\mQ[\mathtt{J}_{1}, :]$ close to the most well-conditioned submatrix.
Now we can replace the first two TT cores with
\begin{equation}
 \tA_{1} := \Phi_{1}(\mathtt{X}_{1})^{-1} \,\mQ\, \mQ[\mathtt{J}_{1}, :]^{-1}, \quad \mG_{2}(x_{2}) := \mQ[\mathtt{J}_{1}, :] \, \mR \, \mG_{2}(x_{2}).
\end{equation}
Note that both inversions in $\tA_{1}$ are well defined by assumptions (and are in fact well conditioned),
the product $\mG_{1}(x_{1}) \mG_{2}(x_{2})$ (and hence the entire tensor train) is unchanged,
and if we take $\tilde{\mathtt{X}}_{1} = \mathtt{X}_{1}[\mathtt{J}_{1}]$ it holds that $\mG_{1}(\tilde{\mathtt{X}}_{1}) = \mathtt{Id}_{r_{1}}$.

Assuming that $\mG_{1:k-1}(\tilde{\mathtt{X}}_{1:k-1})=\mathtt{Id}_{r_{k-1}}$, we proceed as follows.
Evaluating $\mG_{k}(x_k)$ on the corresponding collocation set,
we obtain a three-dimensional tensor with elements
\begin{equation}\label{eq:CrossTensor}
 \tH_k[\alpha_{k-1},i,\alpha_k] := \mG_{k}^{(\alpha_{k-1},\alpha_k)}(x_k^{(i)}) = \sum_{j} \phi_k^{(j)}(x_k^{(i)}) \tA_k[\alpha_{k-1},j,\alpha_k],
\end{equation}
which can be reshaped into a matrix $\mH_k^{(\mathrm{L})} \in \mathbb{R}^{r_{k-1} I_k \times r_k}$ with elements
\begin{equation}
\mH_k^{(\mathrm{L})}[\alpha_{k-1} + (i-1)r_{k-1}, \alpha_k] := \tH_k[\alpha_{k-1},i,\alpha_k].
\end{equation}
Now we can compute again a thin QR decomposition $\mQ\mR = \mH_k^{(\mathrm{L})}$, apply the \texttt{maxvol} algorithm to $\mQ$ to find the index set $\mathtt{J}_{k} \subset \{1,\ldots,r_{k-1}I_k\}$,
and update
\begin{equation}\label{eq:core-updated}
 \tA_k[\alpha_{k-1},i,\alpha_k] := \mA_k^{(\mathrm{L})}[\alpha_{k-1}+(i-1)r_{k-1}, \alpha_k], \quad  \mA_k^{(\mathrm{L})} := (\Phi_k(\mathtt{X}_k)^{-1} \otimes \mathtt{Id}_{r_{k-1}}) \mQ \mQ[\mathtt{J}_{k},:]^{-1}.
\end{equation}
Similarly, we continue the recursion
\begin{equation}\label{eq:left-set}
\tilde{\mathtt{X}}_{1:k} = (\tilde{\mathtt{X}}_{1:k-1} \times \mathtt{X}_k)[\mathtt{J}_{k}].
\end{equation}
This gives
\begin{align*}
\mG_{1:k}(\tilde{\mathtt{X}}_{1:k}) & = \left(\sum_{\alpha_{k-1}}
\mG_k^{(\alpha_{k-1},:)}(\mathtt{X}_k)
\otimes
\mG_{1:k-1}^{(\alpha_{k-1})}(\tilde{\mathtt{X}}_{1:k-1})
\right)[\mathtt{J}_{k}] \\
& = \left((\Phi_k(\mathtt{X}_k) \otimes \mathtt{Id}_{r_{k-1}}) \mA_k^{(\mathrm{L})}\right)[\mathtt{J}_{k}] = \mathtt{Id}_{r_k}.
\end{align*}
The backward recursion from $\mG_{k+1:m+n}$ to $\mG_{k:m+n}$ can be established similarly.
In a $k$th step, we compute
\begin{equation}\label{eq:core-right}
 \mH_k^{(\mathrm{R})}[\alpha_{k-1}, i + (\alpha_k-1) I_k] := \tH_k[\alpha_{k-1},i,\alpha_k] := g(\tilde{\mathtt{X}}_{1:k-1}^{(\alpha_{k-1})}, x_k^{(i)}, \tilde{\mathtt{X}}_{k+1:m+n}^{(\alpha_k)}).
\end{equation}
Then, we apply the QR and \texttt{maxvol} algorithms to $(H_k^{(\mathrm{R})})^\top$ to obtain an index set $\mathtt{J}_{k}$, and select the new right interpolation set
\begin{equation}\label{eq:right-set}
 \tilde{\mathtt{X}}_{k:m+n} = (\mathtt{X}_k \times \tilde{\mathtt{X}}_{k+1:m+n})[\mathtt{J}_{k}].
\end{equation}
The rest of the result follows as the forward case. 
\end{proof}
Now we can resolve the interpolation equation \eqref{eq:frame-linear} by simply inverting the small Vandermonde matrix $\Phi_k(\mathtt{X}_k)$.
Unfolding the solution back through \eqref{eq:core-vec} gives an updated TT core,
which can be passed through the QR factorization and Maximum Volume steps \eqref{eq:CrossTensor} -- \eqref{eq:core-updated} to continue the TT-Cross iteration.
We call this method a {\bf fixed-rank} TT-Cross.

For complicated functions it is beneficial to {\bf adapt} TT ranks towards a desired error threshold.
To reduce overestimated ranks, we replace the QR factorizations by a truncated SVD with an error threshold $\delta>0$, for example
$$
\mQ_{\hat r_k} \Sigma_{\hat r_k} \mV_{\hat r_k}^\top \approx \mH_k^{(\mathrm{L})},
$$
where $\hat r_k \le r_k$ is a new truncated rank.
The rest of the procedure is identical by using $\mQ_{\hat r_k}\in\mathbb{R}^{r_{k-1}I_k \times \hat r_k}$.

To increase underestimated ranks, we can notice that the interpolation equation \eqref{eq:frame-linear} can now be written as a direct assignment
$$
\mG_k(x_k^{(i)}) = g(\tilde{\mathtt{X}}_{1:k-1},x_k^{(i)},\tilde{\mathtt{X}}_{k+1:m+n}).
$$
We can expand one of the sets (for example $\tilde{\mathtt{X}}_{k+1:m+n}$) with additional (e.g. random) points $\hat{\mathtt{X}}_{k+1:m+n}$ of a chosen cardinality $\tilde r_k$, and compute
$$
\mG_k(x_k^{(i)}) = g\big(\tilde{\mathtt{X}}_{1:k-1},x_k^{(i)},\tilde{\mathtt{X}}_{k+1:m+n}\cup \hat{\mathtt{X}}_{k+1:m+n}\big)  \in \mathbb{R}^{r_{k-1} \times (r_k + \hat r_k)},
$$
obtaining eventually an expanded coefficient tensor $\tA_k \in \mathbb{R}^{r_{k-1} \times  I_k \times (r_k + \tilde r_k)}$.
Now the Maximum Volume algorithm can select up to $r_k + \hat r_k$ indices in $\mathtt{J}_{k}$.
This interplay of reduction and expansion of TT ranks allows the method to converge them to nearly optimal ranks for the chosen truncation error threshold.
The entire iteration is outlined in Algorithm~\ref{alg:tt-cross}.
Each full iteration needs $\mathcal{O}((n+m)Ir^2)$ samples from $\pi_{Y,\Theta}$ and $\mathcal{O}((n+m)Ir^3)$ additional floating point operations for QR/SVD and \texttt{maxvol} computations.
\begin{algorithm}[htb]
\centering
\caption{TT-Cross (fixed-rank if $\hat r_k=0$, adaptive-rank if $\hat r_k>0$.)}
\label{alg:tt-cross}
\begin{algorithmic}[1]
 \State Choose initial sets $\tilde{\mathtt{X}}_{1:k}$, $k=1,\ldots,m+n$, stopping threshold $\delta>0$, and expansion ranks $\hat r_k$.
 \While{first iteration or $\|\tilde g - \tilde g_{\rm prev} \|>\delta \|\tilde g\|$}
   \For{$k=m+n,\ldots,3,2$} \Comment{backward iteration}
     \State Sample $\tH_k[\alpha_{k-1},i,\alpha_k] = g((\tilde{\mathtt{X}}_{1:k-1} \cup \hat{\mathtt{X}}_{1:k-1})^{(\alpha_{k-1})}, x_k^{(i)}, \mathtt{X}_{k+1:m+n}^{(\alpha_k)})$.
     \State Compute $\mathtt{J}_{k}$ from SVD and \texttt{maxvol} of $(\mH_k^{(\rm R)})^\top$, compute new set \eqref{eq:right-set}.
   \EndFor
   \For{$k=1,2,\ldots,m+n-1$} \Comment{forward iteration}
     \State Sample $\tH_k[\alpha_{k-1},i,\alpha_k] = g(\tilde{\mathtt{X}}_{1:k-1}^{(\alpha_{k-1})}, x_k^{(i)}, (\tilde{\mathtt{X}}_{k+1:m+n} \cup \hat{\mathtt{X}}_{k+1:m+n})^{(\alpha_k)})$.
     \State Compute $\mathtt{J}_{<k+1}$ from SVD and \texttt{maxvol} of $\mH_k^{(\rm L)}$, compute new set \eqref{eq:left-set}.
     \State Reconstruct TT cores as shown in~\eqref{eq:core-updated}.
   \EndFor
   \State Sample $\tH_{m+n}[\alpha_{m+n-1}, i] = g(\tilde{\mathtt{X}}^{(\alpha_{m+n-1})}_{1:m+n-1},x_{m+n}^{(i)})$ and compute the last TT core.
 \EndWhile
\end{algorithmic}
\end{algorithm}

\subsection{SIRT error}\label{sec:sirt_error}

Suppose we have an approximation $g$ to $\sqrt{\pi_{Y,\Theta}}$ such that the $L^\lowsup{2}$ error  satisfies $\|\sqrt{\pi_{Y,\Theta}}-g\|_{2} \leq \varepsilon / \sqrt2$. Let $z = \|g\|_2^2$ denote the normalizing constant of the approximate density $p_{Y,\Theta}(y,\theta) = \frac1z g(y,\theta)^2$. Then, the normalizing constant satisfies
\begin{align*}
|z - 1| & = \left| \int \left(g^2 - \pi_{Y,\Theta} \right) \d y \d \theta \right| \\
& = \left| \int \left(g - \sqrt{\pi_{Y,\Theta}}\right)\left(g + \sqrt{\pi_{Y,\Theta}} \right) \d y \d \theta \right|\\
& \leq \left(\int \left(g - \sqrt{\pi_{Y,\Theta}}\right)^2 \d y \d \theta \right)^\frac12 \left(\int \left(g + \sqrt{\pi_{Y,\Theta}} \right)^2 \d y \d \theta \right)^\frac12 \\
& \leq \|\sqrt{\pi_{Y,\Theta}}-g\|_2 \left( \|g\|_2 + \|\sqrt{\pi_{Y,\Theta}}\|_2 \right) \\
& = \|\sqrt{\pi_{Y,\Theta}}-g\|_2 \left( \sqrt{z} + 1 \right).
\end{align*}
Since $|z - 1| = |\sqrt{z} - 1|(\sqrt{z} + 1)$, the above inequality leads to $|\sqrt{z} - 1| \leq \|\sqrt{\pi_{Y,\Theta}}-g\|_2$.
Next, we write 
\begin{align*}
\DH( \pi_{Y,\Theta}, p_{Y,\Theta} ) & = \frac1{\sqrt2} \left( \int \left( \sqrt{\pi_{Y,\Theta}} - \sqrt{p_{Y,\Theta}}\right)^2 \d y\d \theta \right)^\frac12 \nonumber \\
& = \frac1{\sqrt2} \left( \int \left( \sqrt{\pi_{Y,\Theta}} - g + g - \frac{1}{\sqrt{z}}g\right)^2 \d y\d \theta \right)^\frac12 \nonumber \\
& \leq \frac1{\sqrt2} \left( \int \left( \sqrt{\pi_{Y,\Theta}} - g\right)^2 \d y\d \theta \right)^\frac12 + \frac1{\sqrt2} \left( \int  g^2 \left( 1-\frac{1}{\sqrt{z}} \right)^2 \d y\d \theta \right)^\frac12 \nonumber \\
& = \frac1{\sqrt2} \|\sqrt{\pi_{Y,\Theta}}-g\|_2 + \frac{| \sqrt{z} - 1 |}{\sqrt2} \nonumber \\
& \leq \sqrt2 \, \|\sqrt{\pi_{Y,\Theta}}-g\|_2 \leq \varepsilon,
\end{align*}
where we used $|\sqrt{z} - 1| \leq \|\sqrt{\pi_{Y,\Theta}}-g\|_2$ in the last inequality.
This concludes the result. \hfill $\square$

\subsection{Marginalizing squared TT factorization}\label{sec:tt_marginal}
Recall that we are computing $\bar\mM_{k:n} = \int \mG_{k:n}(x_{k:n})\mG_{k:n}(x_{k:n}) \d x_{k:n}$.
For the $k$th set of basis functions, we define the mass matrix $\mM_k \in \R^{I_k \times I_k}$ by
\begin{equation}\label{eq:mass-k}
\mM_k[i,j] = \int \phi_{k}^{(i)}(x_k)\phi_{k}^{(j)}(x_k) \,\d x_k, \quad \text{for}\quad  i = 1,\ldots, I_k,\,j = 1,\ldots, I_k.
\end{equation}
Starting with the last coordinate $k = n$, we initialize a tensor $\tB_n = \tA_n$, and a tensor $\tC_n \in \R^{r_{n-1}\times I_n}$,
\begin{align}
\tC_n[\alpha_{n-1}, \tau] = \sum_{i = 1}^{I_n} \tB_n[\alpha_{k-1}, i] \, \chol_n[i, \tau],
\end{align}
where $\chol_n^{} \chol_n^\top = \mM_n^{}$ is the Cholesky factorization of the univariate mass matrix.
Note that $\bar\mM_{n}^{} = \tC_n^{} \tC_n^\top$.

The following procedure can be used to obtain the coefficient tensor $\tB_{k-1}\in \R^{r_{k-2}\times I_{k-1} \times r_{k-1}}$ for defining the next marginal mass matrix $\bar\mM_{k:n}$:

\begin{enumerate}[leftmargin=14pt]
\item Use the Cholesky factorisation of the mass matrix, $\chol_k^{} \chol_k^\top = \mM_k^{} \in \R^{I_k \times I_k}$, to construct a tensor $\tC_k \in \R^{r_{k-1}\times I_k \times r_k}$:
\begin{align}
\tC_k[\alpha_{k-1}, \tau, \ell_{k}] = \sum_{i = 1}^{I_k} \tB_k[\alpha_{k-1}, i, \ell_{k}] \, \chol_k[i, \tau].
\end{align}
\item  Unfold $\tC_k$ along the first coordinate \cite{kolda2009tensor} to obtain a matrix $\mC_k^{(\rm R)} \in \R^{r_{k-1} \times (I_k  r_k)} $ and compute the thin QR factorisation
\begin{align}\label{eq:SIRT-QR}
\mQ_k \mR_k = \big( \mC_k^{(\rm R)} \big)^\top,
\end{align}
where $\mQ_k \in \R^{(I_k  r_k) \times r_{k-1}} $ is semi--orthogonal and $\mR_k \in \R^{r_{k-1} \times r_{k-1}}$ is upper--triangular.
\item The current marginal mass matrix is recovered as $\bar\mM_{k:n} = \mR_k^\top \mR_k^{}$.
\item Compute the new coefficient tensor
\begin{align}\label{eq:B_recur}
\tB_{k-1}[\alpha_{k-2},i, \ell_{k-1}] =  \sum_{\alpha_{k-1} = 1}^{r_{k-1}} \tA_{k-1}[\alpha_{k-2},i, \alpha_{k-1}]\, \mR_k[\ell_{k-1},\alpha_{k-1}].
\end{align}
\end{enumerate}

%
\section{Proofs of Lemmas \ref{prop:defG_optimal} and \ref{prop:bound_Hellinger_Poincare} in Section \ref{sec:Variable_ordering}}

\subsection{Proofs of Lemma \ref{prop:defG_optimal}}\label{proof:defG_optimal}

We denote the norm and the scalar product of $L^2:=\{u:\R^d\rightarrow\R:\int u^2\d x<\infty\}$ by $\|\cdot\|$ and $\langle\cdot,\cdot\rangle$, respectively.
Let $\widetilde f(x) = \widetilde g(x_{\gamma}) \sqrt{\rho(x)}$ and $f(x) = g(x_{\gamma}) \sqrt{\rho(x)}$. We have
\begin{align}
  \langle \widetilde f,\sqrt{\pi}\rangle
  &= \int \widetilde g(x_{\gamma}) \sqrt{\pi(x)/\rho(x)}\rho(x)\d x \nonumber\\
  &= \int \widetilde g(x_{\gamma}) \underbrace{\left(\int \left( \frac{\pi(x_{\gamma},x_{\gamma^c})}{\rho(x_{\gamma},x_{\gamma^c})}\right)^\frac12\rho_{\gamma^c|\gamma}(x_{\gamma^c}|x_{\gamma}) \d x_{\gamma^c} \right)}_{=g(x_{\gamma})} \rho_{\gamma}(x_{\gamma}) \d x_{\gamma} \nonumber\\
  &= \int \widetilde g(x_{\gamma})g(x_{\gamma})\rho(x) \d x
  = \langle \widetilde f,f\rangle .\label{eq:tmp26899}
\end{align}
Similarly we have $\langle f,\sqrt{\pi}\rangle= \langle  f,f\rangle$.
This yields
\begin{align*}
  \DH(\pi,\widetilde\pi)^2
  =& \frac12\left\| \left(\sqrt{\pi} -\frac{f}{\|f\|} \right) + \left(\frac{f}{\|f\|}- \frac{\widetilde f}{\|\widetilde f\|}\right) \right\|^2 \\
  =& \DH(\pi,\widetilde\pi^*)^2  + \left\langle \sqrt{\pi} -\frac{f}{\|f\|} , \frac{f}{\|f\|}- \frac{\widetilde f}{\|\widetilde f\|} \right\rangle + \DH(\widetilde\pi^*,\widetilde\pi)^2 \\
  \overset{\eqref{eq:tmp26899}}{=}& \DH(\pi,\widetilde\pi^*)^2  + \left\langle f -\frac{f}{\|f\|} , \frac{f}{\|f\|}- \frac{\widetilde f}{\|\widetilde f\|} \right\rangle + \DH(\widetilde\pi^*,\widetilde\pi)^2 \\
  =& \DH(\pi,\widetilde\pi^*)^2  + (\|f\|-1)\left(1- \frac{\langle f,\widetilde f\rangle}{\|f\|\|\widetilde f\|}\right) + \DH(\widetilde\pi^*,\widetilde\pi)^2 \\
  =& \DH(\pi,\widetilde\pi^*)^2  + (\|f\|-1)\,\DH(\widetilde\pi^*,\widetilde\pi)^2 + \DH(\widetilde\pi^*,\widetilde\pi)^2 \\
  =& \DH(\pi,\widetilde\pi^*)^2  + \|f\|\DH(\widetilde\pi^*,\widetilde\pi)^2,
\end{align*}
and thus the results follow.\hfill $\square$

\subsection{Proofs of Lemma \ref{prop:bound_Hellinger_Poincare}}\label{proof:bound_Hellinger_Poincare}

 We denote the norm and the scalar product of $L^2:=\{u:\R^d\rightarrow\R:\int u^2\d x<\infty\}$ by $\|\cdot\|$ and $\langle\cdot,\cdot\rangle$, respectively. We also define $\rho_{\gamma}(x_{\gamma}) = \prod_{i \in \gamma}\rho_{i}(x_{i})$ and $\rho_{\gamma^c}(x_{\gamma^c}) = \prod_{i \in \gamma^c}\rho_{i}(x_{i})$.  Letting $ f(x) = g(x_{\gamma}) \sqrt{\rho(x)}$, we have the identity
 \begin{align}
  \langle f,\sqrt{\pi}\rangle
  &= \int g(x_{\gamma}) \sqrt{\pi(x)/\rho(x)}\rho(x)\d x \nonumber\\
  &= \int g(x_{\gamma}) \underbrace{\left(\int \left(\frac{\pi(x_{\gamma},x_{\gamma^c})}{\rho(x_{\gamma},x_{\gamma^c})}\right)^\frac12\rho_{\gamma^c}(x_{\gamma^c})\d x_{\gamma^c} \right)}_{=g(x_{\gamma})} \rho_{\gamma}(x_{\gamma}) \d x_{\gamma} \nonumber\\
  &= \int g(x_{\gamma})^2\rho(x) \d x
  = \|f\|^2 .\label{eq:tmp2687}
 \end{align}
 Furthermore, $\widetilde\pi^*(x) \propto g(x_{\gamma})^2  \rho(x)=f(x)^2$ writes $\widetilde\pi^*(x) = \frac{f(x)^2}{\|f\|^2}$ and satisfies
\begin{equation}\label{eq:tmp2688}
  \DH(\pi,\widetilde\pi^*)^2
  = \frac12\left\| \sqrt{\pi}-\frac{f}{\|f\|} \right\|^2
  = 1- \left\langle \sqrt{\pi},\frac{f}{\|f\|} \right\rangle
  \overset{\eqref{eq:tmp2687}}{=} 1-\|f\|.
\end{equation}
Because $\DH(\pi,\widetilde\pi^*)^2\geq0$, we have $0\leq \|f\| \leq 1$ so that
\[
1-\|f\| \leq 1 - \|f\|^2 \overset{\eqref{eq:tmp2687}}{=} \|\sqrt{\pi}-f\|^2.
\]
Thus, we have $\DH(\pi,\widetilde\pi^*)^2 \leq \|\sqrt{\pi}-f\|^2$.
Letting $h(x) = \sqrt{\pi(x)/\rho(x)}$, we obtain
 \begin{align}
  \DH(\pi,\pi^*)^2 & \leq \; \int \left( \sqrt{\pi(x)}-f(x) \right)^2\d x = \int \left( h(x)-g(x_{\gamma}) \right)^2\rho(x)\d x \nonumber\\
  & \!\!\! \overset{\eqref{eq:defG_proof}}{=} \int \bigg(\int \bigg( h(x_{\gamma},x_{\gamma^c})-\int h(x_{\gamma},x_{\gamma^c}') \rho_{\gamma^c}(x_{\gamma^c}') \d x_{\gamma^c}' \bigg)^2 \rho_{\gamma^c}(x_{\gamma^c}) \d x_{\gamma^c}\bigg)\rho_{\gamma}(x_{\gamma})\d x_{\gamma} \nonumber\\
  & \leq \kappa \int \bigg( \int \left\| \nabla_{{\gamma^c}} \, h(x_{\gamma},x_{\gamma^c}) \right\|_2^2 \rho_{\gamma^c}(x_{\gamma^c}) \d x_{\gamma^c}\bigg)\rho_{\gamma}(x_{\gamma})\d x_{\gamma}  \label{ineq:tmp6893204} \\
  & = \kappa \int \left\| \nabla_{{\gamma^c}} \, h(x_{\gamma},x_{\gamma^c}) \right\|_2^2  \rho(x)\d x = \kappa \sum_{i\in\gamma^c} \int \left( \partial_i h(x) \right)^2  \rho(x)\d x  .\nonumber
 \end{align}
 The previous inequality \eqref{ineq:tmp6893204} is obtained by applying the Poincar\'{e} inequality \eqref{eq:PoincareRho} to the function $h_{\gamma^c}:x_{\gamma^c}\mapsto h(x_{\gamma},x_{\gamma^c})$.
 To conclude the proof, it remains to show that $\int \left( \partial_i h(x) \right)^2  \rho(x)\d x  = \frac14 H_{ii}$, which can be given as follows
 \begin{align*}
  \int \left( \partial_i h(x) \right)^2  \rho(x)\d x
  &= \int \left( \partial_i \sqrt{\pi(x)/\rho(x)} \right)^2  \rho(x)\d x  \\
  &= \frac14 \int \left( \frac{\partial_i (\pi(x)/\rho(x))}{\sqrt{\pi(x)/\rho(x)}} \right)^2  \rho(x)\d x  \\
  &= \frac14 \int \left( \sqrt{\pi(x)/\rho(x)} \partial_i \log\Big(\pi(x)/\rho(x)\Big)   \right)^2  \rho(x)\d x \\
  &= \frac14 \int \left( \partial_i \log\Big(\pi(x)/\rho(x)\Big)  \right)^2  \pi(x)\d x
  = \frac14 H_{ii}.
 \end{align*}
This concludes the proof.\hfill $\square$

\section{Proofs and additional details of Section \ref{sec:DIRT}}

\subsection{Proof of Proposition \ref{prop:DIRT_CV}}\label{proof:DIRT_CV}
Because the Hellinger distance satisfies the triangle inequality, we can write 
\begin{align*}
 \DH\big( p_{Y,\Theta}^{\ell+1}, \pi_{Y,\Theta}^{\ell+1} \big) 
 &=\DH\big( (\mathcal{T}_{\ell}\circ \mathcal{Q}_{\ell+1} )_\sharp\rho_{Y,\Theta}, \pi_{Y,\Theta}^{\ell+1} \big) \\
 &=\DH\big( (\mathcal{Q}_{\ell+1} )_\sharp\rho_{Y,\Theta}, \mathcal{T}_{\ell}^\sharp \pi_{Y,\Theta}^{\ell+1} \big) \\
 &\overset{\eqref{eq:Greedy_DIRT}}{\leq} \omega ~\DH\big( p_{Y,\Theta}^{\ell}, \pi_{Y,\Theta}^{\ell+1} \big) \\
 &\leq \;\, \omega \Big( \DH\big(  p_{Y,\Theta}^{\ell} , \pi_{Y,\Theta}^{\ell} \big) +  \DH\big(  \pi_{Y,\Theta}^{\ell}, \pi_{Y,\Theta}^{\ell+1} \big) \Big)\\
 &\overset{\eqref{eq:BoundOnBridgingPdfs}}{\leq} \omega \Big( \DH\big(  p_{Y,\Theta}^{\ell} , \pi_{Y,\Theta}^{\ell} \big) +  \eta(L) \Big) ,
\end{align*}
for any $0\leq\ell<L$. A direct recurrence yields
\begin{align*}
 \DH\big( p_{Y,\Theta}^{L} , \pi_{Y,\Theta}^{L} \big)
 &\leq \omega^{L} \DH\big(  p_{Y,\Theta}^{0} , \pi_{Y,\Theta}^{0} \big) + \omega^{L}\eta(L)+\hdots+  \omega \eta(L) \\
 &\leq \omega^{L+1} \eta(L) + \omega^{L}\eta(L)+\hdots+  \omega \eta(L) \\
 &\leq \frac{\omega}{1-\omega} \eta(L),
\end{align*}
which concludes the proof. \hfill $\square$

\subsection{Bridging densities with uniformly spaced temperatures}\label{sec:dirt_fix}

Proposition \ref{prop:DIRT_CV} reveals that controlling the Hellinger distances between the bridging densities is one of the keys to ensure the accuracy of DIRT. 
Using the tempered density $\pi_{Y,\Theta}^\lowsup{\ell} \propto \varphi_{Y,\Theta}^\lowsup{\beta_\ell}\, \rho_{Y,\Theta}$ as the example, we demonstrate in Lemma \ref{lemma:const_temp} that under mild technical assumptions, the Hellinger distance between $\pi_{Y,\Theta}^\lowsup{\ell}$ and $\pi_{Y,\Theta}^\lowsup{\ell{+}1}$ is bounded by $\mathcal{O}(\beta_{\ell{+}1} {-} \beta_\ell$). 
This ensures that the uniformly spaced temperatures $\beta_\ell = \ell/L$ yields controlled Hellinger distances of the form of $\DH\big( \pi_{Y,\Theta}^{\ell} ,  \pi_{Y,\Theta}^{\ell+1} \big) = \mathcal{O}(1/L)$.

\begin{lemma}\label{lemma:const_temp}
We denote the supremum of $\log \varphi_{Y,\Theta}$, the mean and the second moment of $\log \rho_{Y,\Theta}$ with respect to $\rho_{Y,\Theta}$ by
\[
c_\varphi \coloneqq \sup_{y,\theta} \log \varphi_{Y,\Theta},  \quad m_\varphi \coloneqq \int \log \varphi_{Y,\Theta} \d \rho_{Y,\Theta}, \quad {\rm and} \quad V_\varphi \coloneqq \int \left(\log \varphi_{Y,\Theta}\right)^2 \d \rho_{Y,\Theta} ,
\]
respectively. Suppose $c_\varphi < \infty$ and $V_\varphi < \infty$. Considering the temperatures of adjacent levels satisfy $\beta_{\ell{+}1} = \beta_\ell + \Delta$ where $\Delta > 0$, the Hellinger distance between $\pi_{Y,\Theta}^\lowsup{\ell} \propto \varphi_{Y,\Theta}^\lowsup{\beta_\ell}\, \rho_{Y,\Theta}$ and $\pi_{Y,\Theta}^\lowsup{\ell{+}1} \propto \varphi_{Y,\Theta}^\lowsup{\beta_\ell {+} \Delta}\, \rho_{Y,\Theta}$ satisfies
\[
\DH(\pi_{Y,\Theta}^\lowsup{\ell}, \pi_{Y,\Theta}^\lowsup{\ell{+}1}) \leq \Delta C,
\]
for any sequence of finite temperatures, where $C=C(c_\varphi,m_\varphi,V_\varphi)$ is a constant independent of $\Delta$.
\end{lemma}
\begin{proof}
We define the normalized probability densities
\begin{align*}
\pi_{Y,\Theta}^\lowsup{\ell} = \frac{1}{z_{\ell}} \varphi_{Y,\Theta}^\lowsup{\beta_\ell}\, \rho_{Y,\Theta} \quad {\rm and} \quad \pi_{Y,\Theta}^\lowsup{\ell{+}1} = \frac{1}{z_{\ell{+}1}} \varphi_{Y,\Theta}^\lowsup{\beta_\ell+\Delta}\, \rho_{Y,\Theta} ,
\end{align*}
where
\(
z_{\ell} = \int \varphi_{Y,\Theta}^\lowsup{\beta_\ell}\, \d \rho_{Y,\Theta}
\)
and
\(
z_{\ell{+}1} = \int \varphi_{Y,\Theta}^\lowsup{\beta_\ell{+}\Delta}\, \d \rho_{Y,\Theta}.
\)
Applying Jensen's inequality, the mean $m_\varphi$ satisfies $|m_\varphi| \leq \sqrt{V_\varphi} < \infty$, and we have the following lower bound on the normalizing constant
\begin{align}
z_\ell & = \exp\bigg( \log \bigg( \int \varphi_{Y,\Theta}^\lowsup{\beta_\ell} \d \rho_{Y,\Theta} \bigg)\bigg) \geq  \exp\bigg( \int \log \varphi_{Y,\Theta}^\lowsup{\beta_\ell} \d \rho_{Y,\Theta} \bigg)= \exp( m_\varphi \beta_\ell) > 0, \label{eq:lower_bound_z}
\end{align}
for any finite $\beta_\ell$. The difference between the normalizing constants has the following upper bound
\begin{align}
\left|z_{\ell} -z_{\ell{+}1} \right| 
& = \left| \int \left(  \sqrt{\varphi_{Y,\Theta}^\lowsup{\beta_\ell}} - \sqrt{\varphi_{Y,\Theta}^\lowsup{\beta_\ell{+}\Delta} } \right) \,\left(  \sqrt{\varphi_{Y,\Theta}^\lowsup{\beta_\ell}} + \sqrt{\varphi_{Y,\Theta}^\lowsup{\beta_\ell{+}\Delta} } \right)\, \d \rho_{Y,\Theta} \right|\nonumber \\
& \leq \bigg( \int \left(  \sqrt{\varphi_{Y,\Theta}^\lowsup{\beta_\ell}} - \sqrt{\varphi_{Y,\Theta}^\lowsup{\beta_\ell{+}\Delta} } \right)^2\, \d \rho_{Y,\Theta} \bigg)^\frac12 \bigg( \int \left(  \sqrt{\varphi_{Y,\Theta}^\lowsup{\beta_\ell}} + \sqrt{\varphi_{Y,\Theta}^\lowsup{\beta_\ell{+}\Delta} } \right)^2\, \d \rho_{Y,\Theta} \bigg)^\frac12 \nonumber \\
& \leq \bigg( \int \left(  \sqrt{\varphi_{Y,\Theta}^\lowsup{\beta_\ell}} - \sqrt{\varphi_{Y,\Theta}^\lowsup{\beta_\ell{+}\Delta} } \right)^2\, \d \rho_{Y,\Theta} \bigg)^\frac12 \bigg( \sqrt{\int \varphi_{Y,\Theta}^\lowsup{\beta_\ell}\, \d \rho_{Y,\Theta} } + \sqrt{\int \varphi_{Y,\Theta}^\lowsup{\beta_\ell{+}\Delta} \, \d \rho_{Y,\Theta} } \bigg) \nonumber \\
& = \bigg( \int \left(  \sqrt{\varphi_{Y,\Theta}^\lowsup{\beta_\ell}} - \sqrt{\varphi_{Y,\Theta}^\lowsup{\beta_\ell{+}\Delta} } \right)^2\, \d \rho_{Y,\Theta} \bigg)^\frac12 \big( \sqrt{z_{\ell}} + \sqrt{z_{\ell{+}1}} \big) \label{eq:diff_z}.
\end{align}
Applying the identity $|z_{\ell} -z_{\ell{+}1}| = |\sqrt{z_{\ell}} - \sqrt{z_{\ell{+}1}}| \, (\sqrt{z_{\ell}} + \sqrt{z_{\ell{+}1}})$, the inequality in \eqref{eq:diff_z} yields
\begin{align}
|\sqrt{z_{\ell}} - \sqrt{z_{\ell{+}1}}| \leq \bigg( \int \bigg(  \sqrt{\varphi_{Y,\Theta}^\lowsup{\beta_\ell}} - \sqrt{\varphi_{Y,\Theta}^\lowsup{\beta_\ell{+}\Delta} } \bigg)^2\, \d \rho_{Y,\Theta} \bigg)^\frac12. \label{eq:diff_sqrt_z}
\end{align}
This way, the Hellinger distance satisfies 
\begin{align}
& \hspace{-1.5em}\DH(\pi_{Y,\Theta}^\lowsup{\ell}, \pi_{Y,\Theta}^\lowsup{\ell{+}1}) \nonumber \\
= &  \frac{1}{\sqrt2} \Bigg( \int \Bigg(\sqrt{\frac{\varphi_{Y,\Theta}^\lowsup{\beta_\ell}}{z_{\ell}} }- \sqrt{ \frac{\varphi_{Y,\Theta}^\lowsup{\beta_\ell{+}\Delta}}{z_{\ell{+}1}} }\Bigg)^2 \d \rho_{Y,\Theta} \Bigg)^\frac12  \nonumber \\
= & \frac{1}{\sqrt{2 z_\ell} } \left( \int \bigg(\sqrt{\varphi_{Y,\Theta}^\lowsup{\beta_\ell} } - \sqrt{\varphi_{Y,\Theta}^\lowsup{\beta_\ell{+}\Delta} } \sqrt{\frac{z_{\ell}}{z_{\ell{+}1}} }\bigg)^2 \d \rho_{Y,\Theta} \right)^\frac12 \nonumber \\
= & \frac{1}{\sqrt{2 z_\ell} } \left( \int \bigg(\sqrt{\varphi_{Y,\Theta}^\lowsup{\beta_\ell} } - \sqrt{\varphi_{Y,\Theta}^\lowsup{\beta_\ell{+}\Delta} } + \sqrt{\varphi_{Y,\Theta}^\lowsup{\beta_\ell{+}\Delta} } -  \sqrt{\varphi_{Y,\Theta}^\lowsup{\beta_\ell{+}\Delta} } \sqrt{\frac{z_{\ell}}{z_{\ell{+}1}} }\bigg)^2 \d \rho_{Y,\Theta} \right)^\frac12 \nonumber \\
\leq & \frac{1}{\sqrt{2 z_\ell} } \left( \int \bigg(\sqrt{\varphi_{Y,\Theta}^\lowsup{\beta_\ell} } - \sqrt{\varphi_{Y,\Theta}^\lowsup{\beta_\ell{+}\Delta} } \bigg)^2 \d \rho_{Y,\Theta}\right)^\frac12 + \frac{1}{\sqrt{2 z_\ell} } \left( \int \varphi_{Y,\Theta}^\lowsup{\beta_\ell{+}\Delta}   \bigg(1 -\sqrt{\frac{z_{\ell}}{z_{\ell{+}1}} }\bigg)^2 \d \rho_{Y,\Theta} \right)^\frac12 \nonumber \\
= & \frac{1}{\sqrt{2 z_\ell} } \left( \int \bigg(\sqrt{\varphi_{Y,\Theta}^\lowsup{\beta_\ell} } - \sqrt{\varphi_{Y,\Theta}^\lowsup{\beta_\ell{+}\Delta} } \bigg)^2 \d \rho_{Y,\Theta}\right)^\frac12 + \frac{\left|\sqrt{z_{\ell{+}1}} -\sqrt{z_{\ell}} \right| }{\sqrt{2 z_\ell} } \nonumber \\
\overset{\eqref{eq:diff_sqrt_z}}{\leq} & \sqrt{ \frac{2}{z_\ell} } \left( \int \bigg(\sqrt{\varphi_{Y,\Theta}^\lowsup{\beta_\ell} } - \sqrt{\varphi_{Y,\Theta}^\lowsup{\beta_\ell{+}\Delta} } \bigg)^2 \d \rho_{Y,\Theta}\right)^\frac12 \label{eq:hell_gap_1}.
\end{align}
Applying H\"{o}lder's inequality, we have
\begin{align}
\DH(\pi_{Y,\Theta}^\lowsup{\ell}, \pi_{Y,\Theta}^\lowsup{\ell{+}1}) 
&\leq  \left( \frac{2 \, \sup_{y,\theta} \big(\varphi_{Y,\Theta}^\lowsup{\beta_\ell} \big) }{z_\ell} \int \bigg( 1 - \varphi_{Y,\Theta}^\lowsup{\Delta/2}  \bigg)^2 \d \rho_{Y,\Theta} \right)^\frac12 \\
&= \left( \frac{2 \, \exp(\beta_\ell c_\varphi) }{z_\ell} \int \bigg( 1 - \varphi_{Y,\Theta}^\lowsup{\Delta/2}  \bigg)^2 \d \rho_{Y,\Theta} \right)^\frac12. \label{eq:hell_gap_2}
\end{align}
By the Lipschitz continuity of the exponential function in the domain $(-\infty, a]$ for $a < \infty$, we have
\[
\left|1 - \varphi_{Y,\Theta}^\lowsup{\Delta/2} \right|  = \left|\exp(0) - \exp\left( \frac\Delta2 \log \varphi_{Y,\Theta}\right) \right|  \leq \exp \Big(\max\Big( 0, \frac{c_\varphi\Delta}2  \Big)\Big)  \left| \frac\Delta2 \log \varphi_{Y,\Theta}  - 0\right|,
\]
which leads to 
\[
\left(1 - \varphi_{Y,\Theta}^\lowsup{\Delta/2} \right)^2 \leq \frac{\Delta^2\exp \big(\max( 0, c_\varphi \Delta)\big)}4  \left(\log \varphi_{Y,\Theta} \right)^2 .
\]
Substituting the above inequality into \eqref{eq:hell_gap_2}, we have
\begin{align}
\DH(\pi_{Y,\Theta}^\lowsup{\ell}, \pi_{Y,\Theta}^\lowsup{\ell{+}1}) \leq & \bigg( \frac{ \Delta^2 \exp \big(  \max( c_\varphi \beta_\ell, c_\varphi \beta_{\ell{+}1}) \big)}{2 z_\ell} \int \left(\log \varphi_{Y,\Theta} \right)^2 \d \rho_{Y,\Theta} \bigg)^\frac12 \nonumber \\
\overset{\eqref{eq:lower_bound_z}}{\leq} & \bigg( \frac{\Delta^2 \exp \big(  \max( c_\varphi \beta_\ell, c_\varphi \beta_{\ell{+}1} ) \big)}{2 \exp( m_\varphi \beta_\ell)}  \bigg)^\frac12  \bigg(\int \left(\log \varphi_{Y,\Theta} \right)^2 \d \rho_{Y,\Theta} \bigg)^\frac12 \nonumber \\
= & \Delta  \bigg( \frac12 \exp\Big( - m_\varphi \beta_\ell + \max\big( c_\varphi \beta_\ell, c_\varphi \beta_{\ell{+}1}\big) \Big) V_\varphi \bigg)^\frac12 .\nonumber 
\end{align}
Thus, the Hellinger distance satisfies
\begin{align*}
\DH(\pi_{Y,\Theta}^\lowsup{\ell}, \pi_{Y,\Theta}^\lowsup{\ell{+}1}) 
&\leq \Delta \bigg( \frac{1}{2} \,V_\varphi\, \sup_{\ell}\exp\big(- m_\varphi \beta_\ell \big) \, \sup_{\ell}\exp\big( c_\varphi \beta_\ell\big) \bigg)^\frac12, 
\end{align*}
which concludes the proof.
\end{proof}

\section{Further numerical tests}

\subsection{Elasticity model}\label{sec:wrench2}

\begin{figure}[h!]
\centering
\noindent\begin{tikzpicture}
\begin{axis}[%
width=0.48\linewidth,
height=0.4\linewidth,
xlabel={TT rank},
title={$\DH(\widetilde\pi_{\Theta|Y=y},~p_{\Theta|Y=y})$},
xmin=4,xmax=14,
]  
\addplot+[mark=*,only marks,mark options={mark size=3},error bars/.cd,y dir=both,y explicit] coordinates{
(5,   0.3109) +- (0, 0.1166)
(6,   0.2552) +- (0, 0.0465)
(7 ,  0.2487) +- (0, 0.0473)
(8,   0.2276) +- (0, 0.0571)
(9 ,  0.2183) +- (0, 0.0455)
(10,  0.2184) +- (0, 0.0313)
(11,  0.2243) +- (0, 0.0277)
(12,  0.2099) +- (0, 0.0225)
(13,  0.2256) +- (0, 0.0542)
                                };
\end{axis}
\end{tikzpicture}
\hfill\noindent\begin{tikzpicture}
\begin{axis}[%
width=0.48\linewidth,
height=0.4\linewidth,
xmode=log,
xlabel={$\beta_{\ell+1}/\beta_\ell$},
title={$\DH(\widetilde\pi_{\Theta|Y=y},~p_{\Theta|Y=y})$},
xmin=1,xmax=11,
legend style={at={(0.9,0.99)},anchor=north east}
]  
\addplot+[only marks,mark options={mark size=3},error bars/.cd,y dir=both,y explicit] coordinates{
(10^(1/1), 0.3757) +- (0, 0.1058)
(10^(1/2), 0.2183) +- (0, 0.0455)
(10^(1/3), 0.2527) +- (0, 0.0597)
(10^(1/4), 0.2477) +- (0, 0.0242)
                                }; 
\addlegendentry{$y_{s} \mydots y_1, \theta_1 \mydots \theta_{t}$}; 
\addplot+[only marks,mark options={mark size=3},error bars/.cd,y dir=both,y explicit] coordinates{
(10^(1/1), 0.3928) +- (0, 0.1475)
(10^(1/2), 0.2613) +- (0, 0.0345)
(10^(1/3), 0.3064) +- (0, 0.0750)
(10^(1/4), 0.3018) +- (0, 0.0423)
                                }; 
\addlegendentry{$y_{1} \mydots y_s, \theta_t \mydots \theta_{1}$}; 
\addplot+[only marks,mark=triangle*,mark options={mark size=3},error bars/.cd,y dir=both,y explicit] coordinates{
(10^(1/1), 0.3022) +- (0, 0.0689)
(10^(1/2), 0.2689) +- (0, 0.0271)
(10^(1/3), 0.2596) +- (0, 0.0280)
(10^(1/4), 0.3082) +- (0, 0.0463)
                                }; 
\addlegendentry{$y_{1} \mydots y_s, \theta_1 \mydots \theta_{t}$}; 
\end{axis}
\end{tikzpicture}
\caption{Wrench example: Hellinger distances (mean$\pm$std over 10 data realizations) between the conditional DIRT approximation and the reduced-dimensional posterior with $s=13$ and $t=10$. Left: TT rank varies and $\beta_{\ell+1}/\beta_{\ell}=\sqrt{10}$ is fixed. Right: $\beta_{\ell+1}/\beta_{\ell}$ varies and TT ranks are fixed to $9$.}
\label{fig:wrenchnew-dirt}
\end{figure}
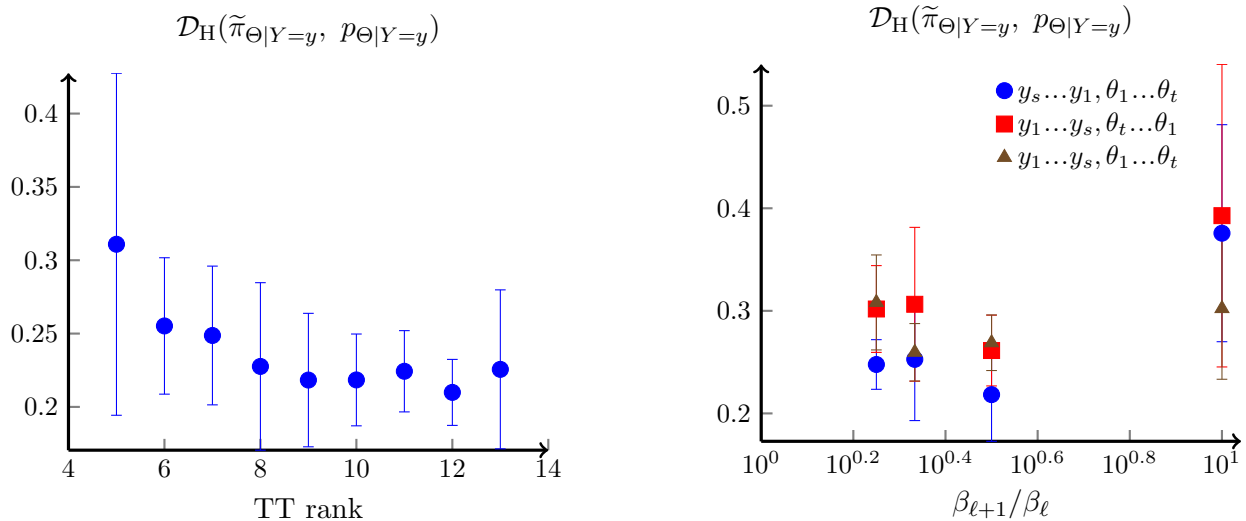

In Figure~\ref{fig:wrenchnew-dirt} we investigate the Hellinger error between the conditional DIRT approximation and the exact posterior on a given reduced subspace for different TT ranks and tempering parameters.
We observe that the approximation error of TT decreases asymptotically with increasing TT ranks. However, the error stops decaying after rank $9$. This may suggest that one needs to further increase the resolution of the discretization grid in constructing TT (cf. \ref{sec:tt_marginal}) if further accuracy improvement is desired. 
Fixing the TT ranks to $9$, we observe that among different sequences of temperatures, the sequence with the ratio $\beta_{\ell+1}/\beta_{\ell} = \sqrt{10}$ gives the best accuracy.

\begin{figure}[htb]
\centering
\begin{minipage}[b]{0.45\linewidth}
\includegraphics[scale=0.5]{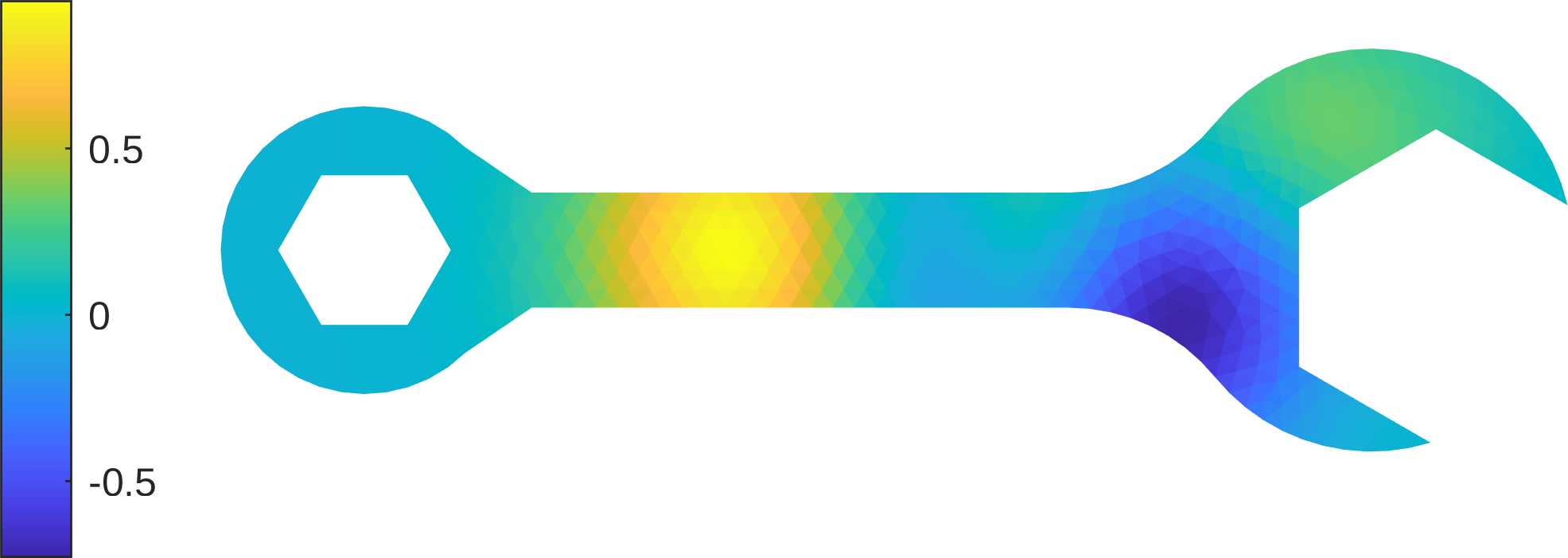}
\end{minipage}
\hfill
\begin{minipage}[b]{0.45\linewidth}
\includegraphics[scale=0.5]{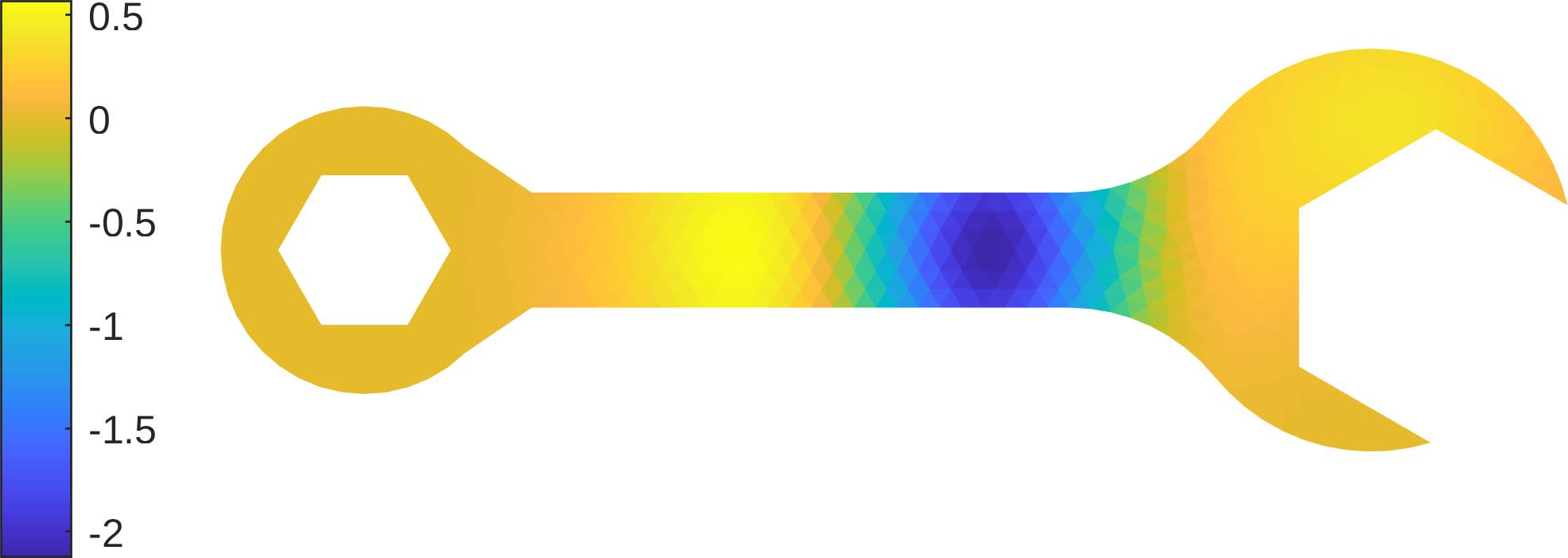}
\end{minipage}\\[1ex]
\begin{minipage}[b]{0.45\linewidth}
\includegraphics[scale=0.5]{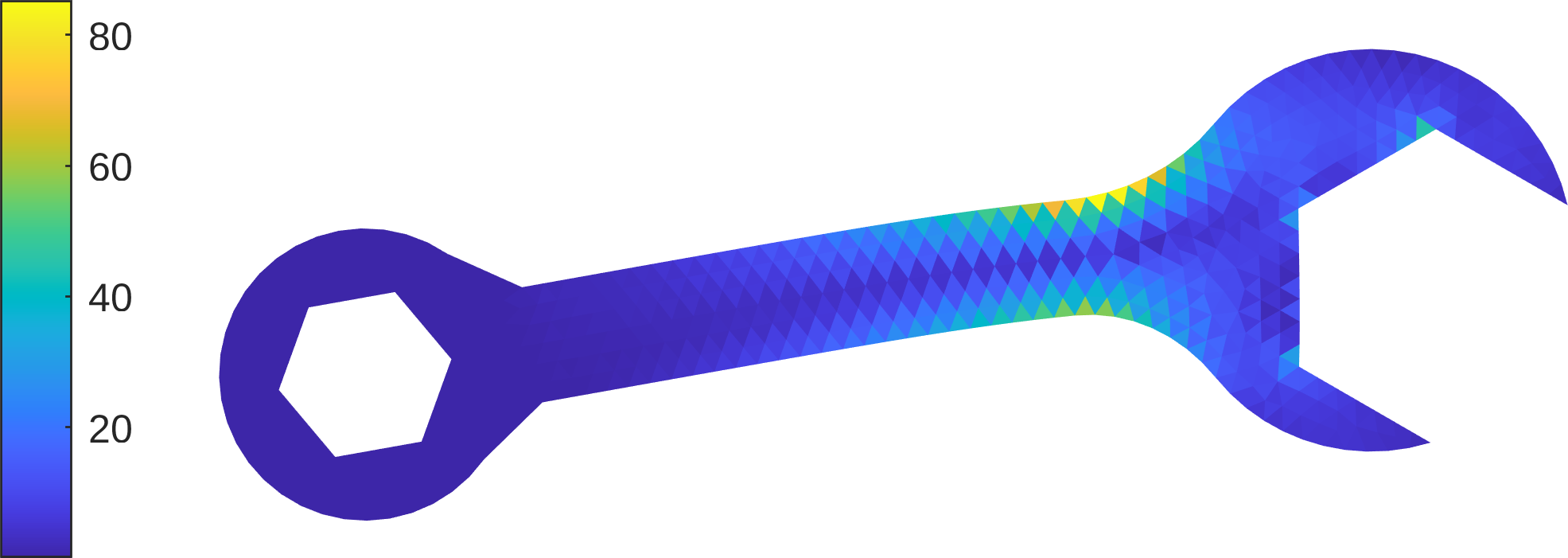}
\end{minipage}
\hfill
\begin{minipage}[b]{0.45\linewidth}
\includegraphics[scale=0.5]{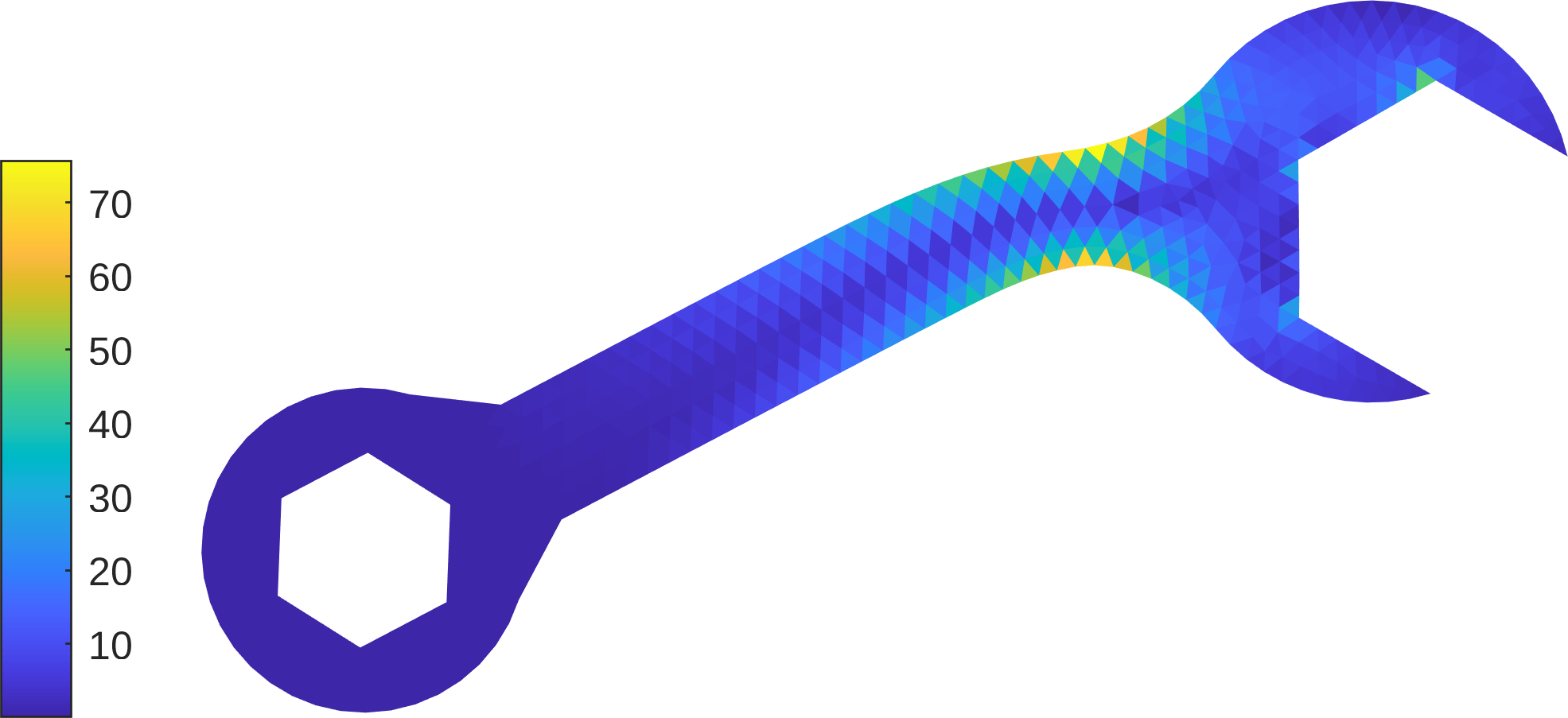}
\end{minipage}\\[5ex]
\begin{minipage}[b]{0.45\linewidth}
\includegraphics[scale=0.5]{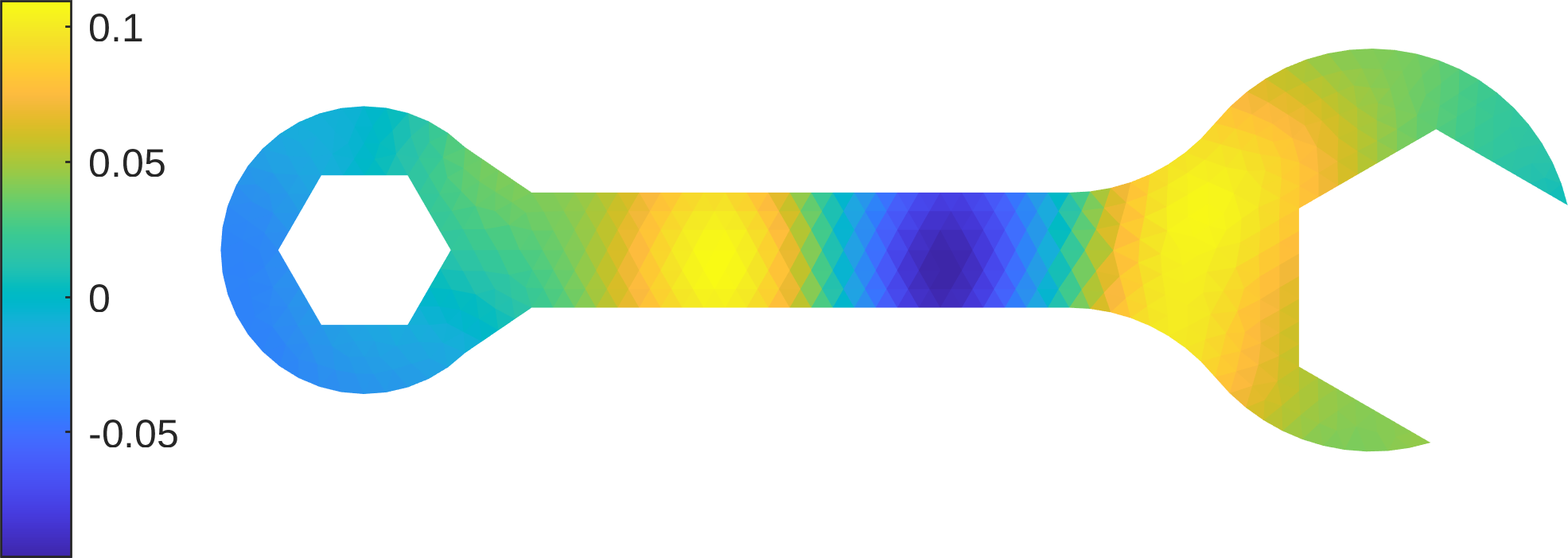}
\end{minipage}
\hfill
\begin{minipage}[b]{0.45\linewidth}
\includegraphics[scale=0.5]{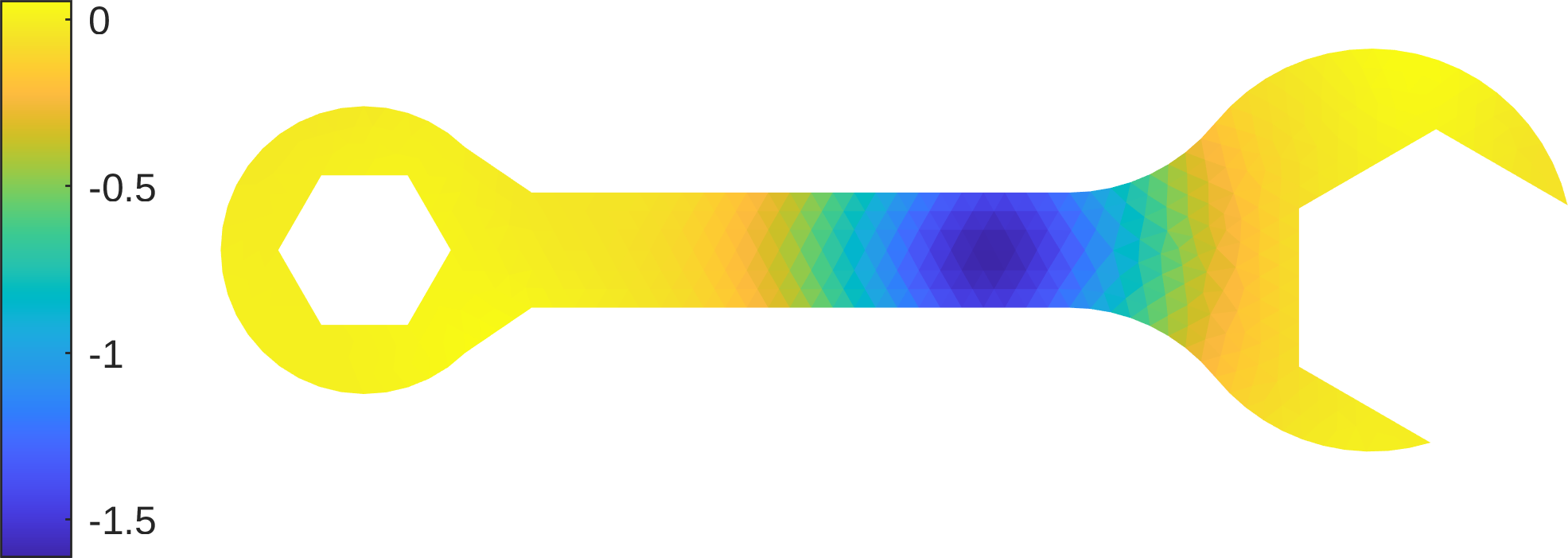}
\end{minipage}\\[1ex]
\begin{minipage}[b]{0.45\linewidth}
\includegraphics[scale=0.5]{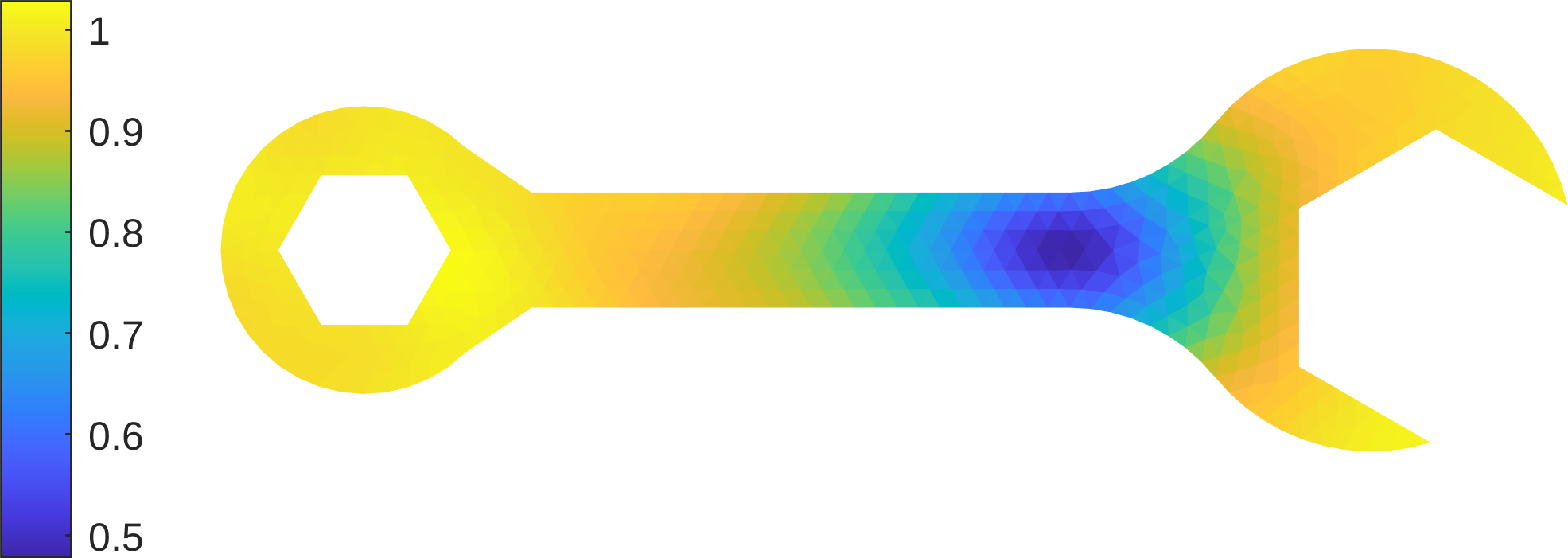}
\end{minipage}
\hfill
\begin{minipage}[b]{0.45\linewidth}
\includegraphics[scale=0.5]{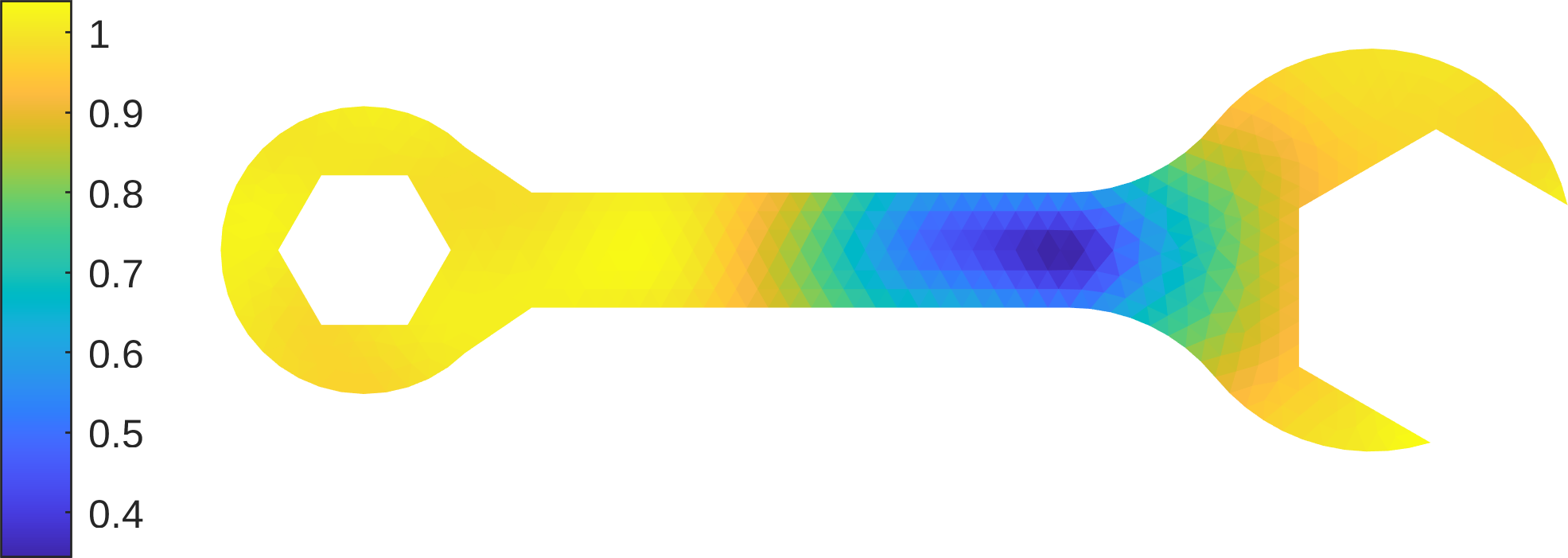}
\end{minipage}\\[5ex]
\begin{minipage}[b]{0.45\linewidth}
\centering
\begin{tabular}{c|c}
$\DH(\widetilde\pi,~ p)$ & $\DH(\pi,~ \widetilde\pi)$ \\\hline
0.172 & 0.151
\end{tabular}
\end{minipage}
\hfill
\begin{minipage}[b]{0.45\linewidth}
\centering
\begin{tabular}{c|c}
$\DH(\widetilde\pi,~ p)$ & $\DH(\pi,~ \widetilde\pi)$ \\\hline
0.298 & 0.289
\end{tabular}
\end{minipage}
\caption{Top two rows: two prior realizations of $\log E$ (log-Young's modulus) and the von Mises stress. Rows 3 and 4: mean and standard deviation of $\log E$ with respect to the posterior conditioned on the data produced from the corresponding prior realizations.
Bottom row: Hellinger-distance errors between the conditional DIRT approximation and the exact reduced posterior, and between reduced and original posterior densities.
}
\label{fig:wrench:samples}
\end{figure}

In Figure~\ref{fig:wrench:samples} (top) we show two realizations of the logarithm of Young's modulus $\log E$ drawn from the prior distribution and the corresponding von Mises stress. These samples are treated as the ``ground truth'' to simulate the observation data. Using truncation parameters $s =13$ and $t=7$, we compute posterior mean and standard deviation of $\log E$ conditioned on those ``ground truth'' data and show the corresponding results in the remaining rows of Figure~\ref{fig:wrench:samples}. The last row of Figure~\ref{fig:wrench:samples} shows the Hellinger distance between the conditional DIRT approximations and the reduced posterior \eqref{eq:api_with_independence_structure}, and the Hellinger distance between the reduced posterior and the original posterior.
%

\subsection{Elliptic PDE with Besov prior}\label{sec:besov2}

\begin{figure}[h!]
\centering
\begin{minipage}[b]{0.45\linewidth}
\includegraphics[scale=0.4]{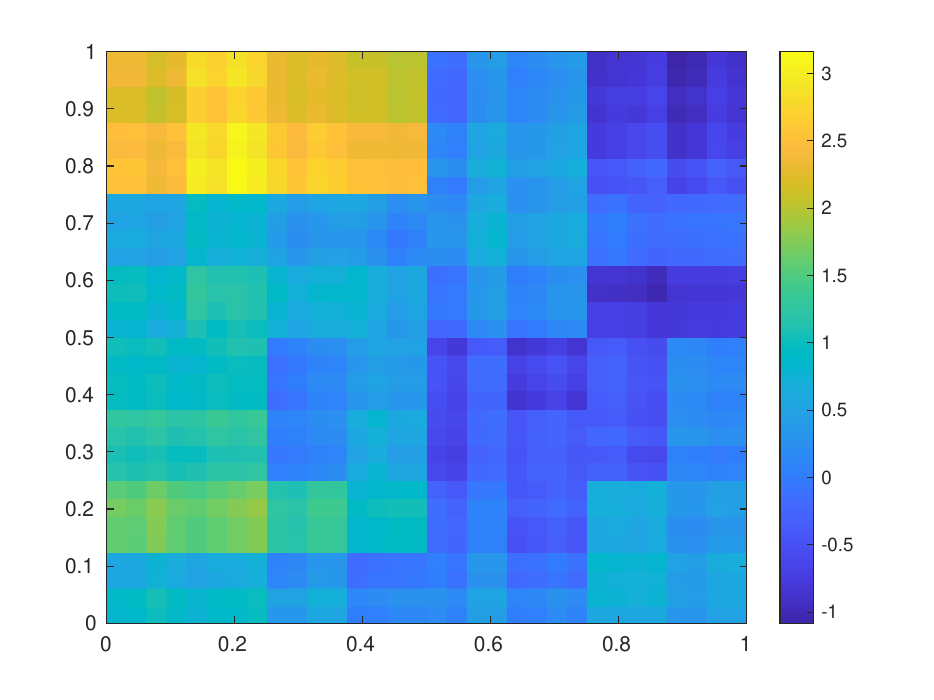}
\end{minipage}
\hfill
\begin{minipage}[b]{0.45\linewidth}
\includegraphics[scale=0.4]{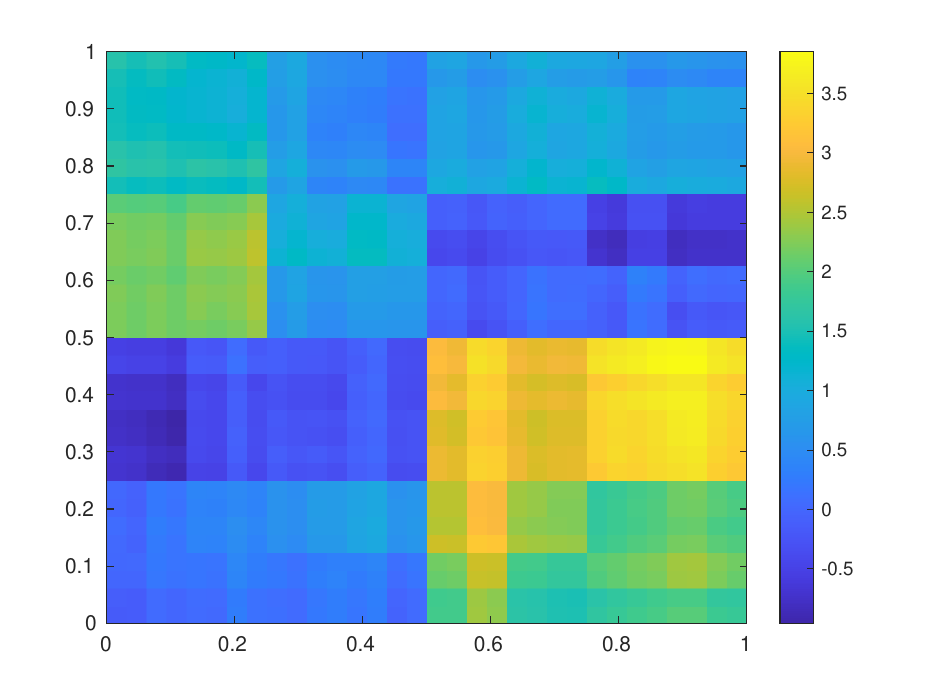}
\end{minipage}\\[1ex]
\begin{minipage}[b]{0.45\linewidth}
\includegraphics[scale=0.4]{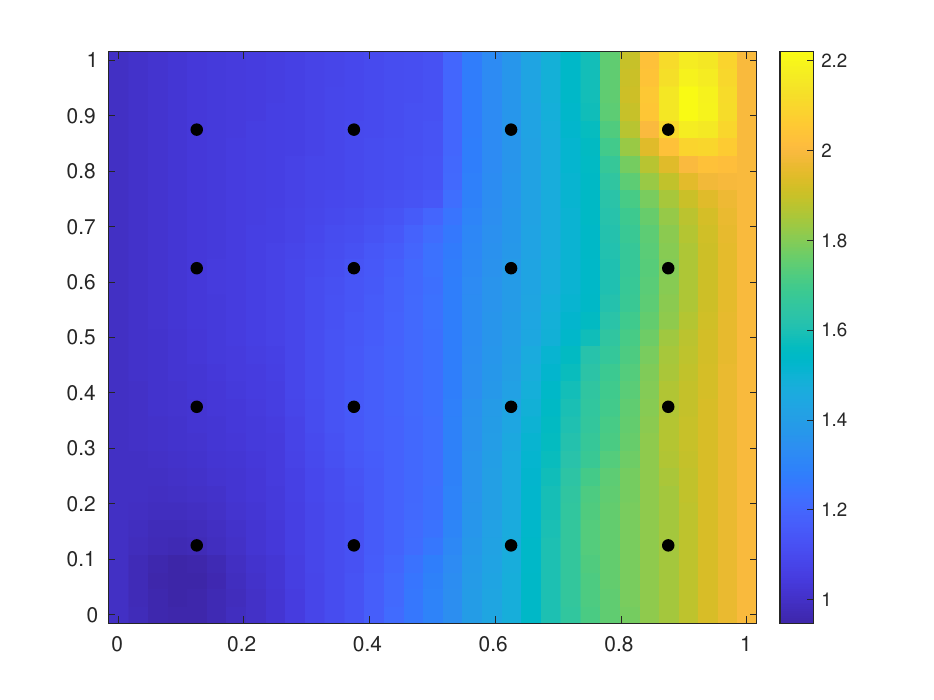}
\end{minipage}
\hfill
\begin{minipage}[b]{0.45\linewidth}
\includegraphics[scale=0.4]{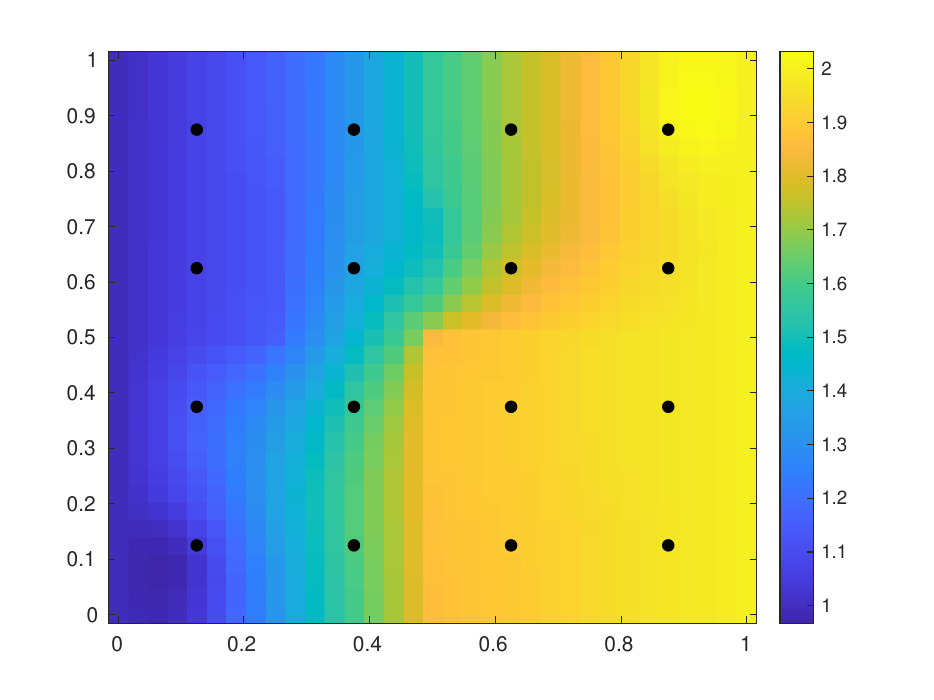}
\end{minipage}\\[1ex]
\begin{minipage}[b]{0.45\linewidth}
\centering
\begin{tabular}{c|c}
$\DH(\widetilde\pi,~ p)$ & $\DH(\pi,~ \widetilde\pi)$ \\\hline
0.247 & 0.135
\end{tabular}
\end{minipage}
\hfill
\begin{minipage}[b]{0.45\linewidth}
\centering
\begin{tabular}{c|c}
$\DH(\widetilde\pi,~ p)$ & $\DH(\pi,~ \widetilde\pi)$ \\\hline
0.255  &  0.127
\end{tabular}
\end{minipage}
\caption{Top two rows: two prior realizations of $\log \kappa$ and the solution $u$. Black dots show the locations of the observations.
Bottom row: Hellinger-distance errors between the conditional DIRT approximation and the reduced posterior, and between reduced and original posterior densities.
}
\label{fig:besov:truth}
\end{figure}

\begin{figure}[h!]
\centering
\begin{minipage}[b]{0.45\linewidth}
\includegraphics[scale=0.4]{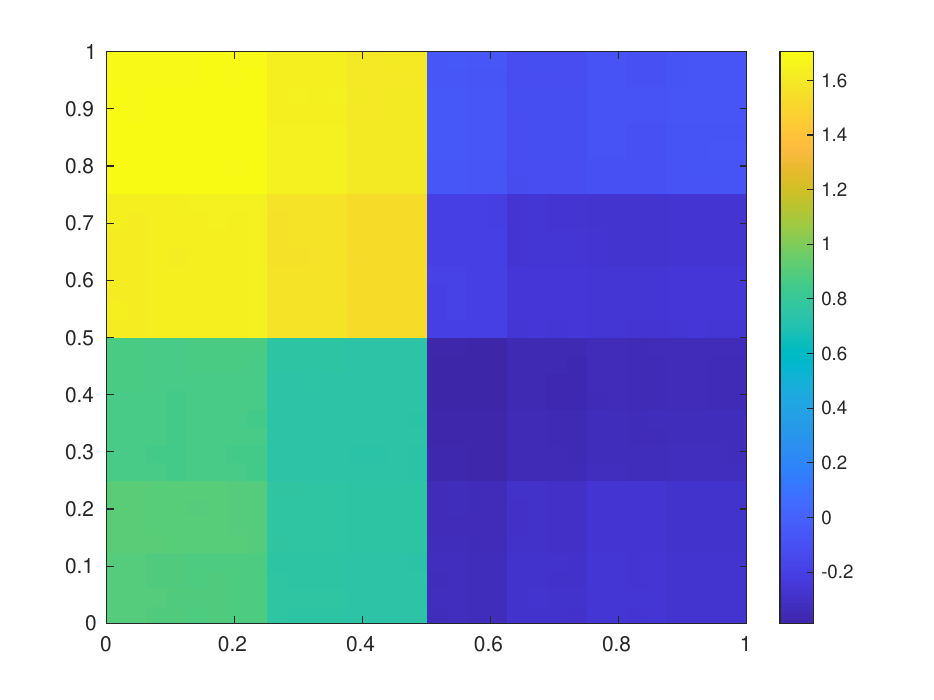}
\end{minipage}
\hfill
\begin{minipage}[b]{0.45\linewidth}
\includegraphics[scale=0.4]{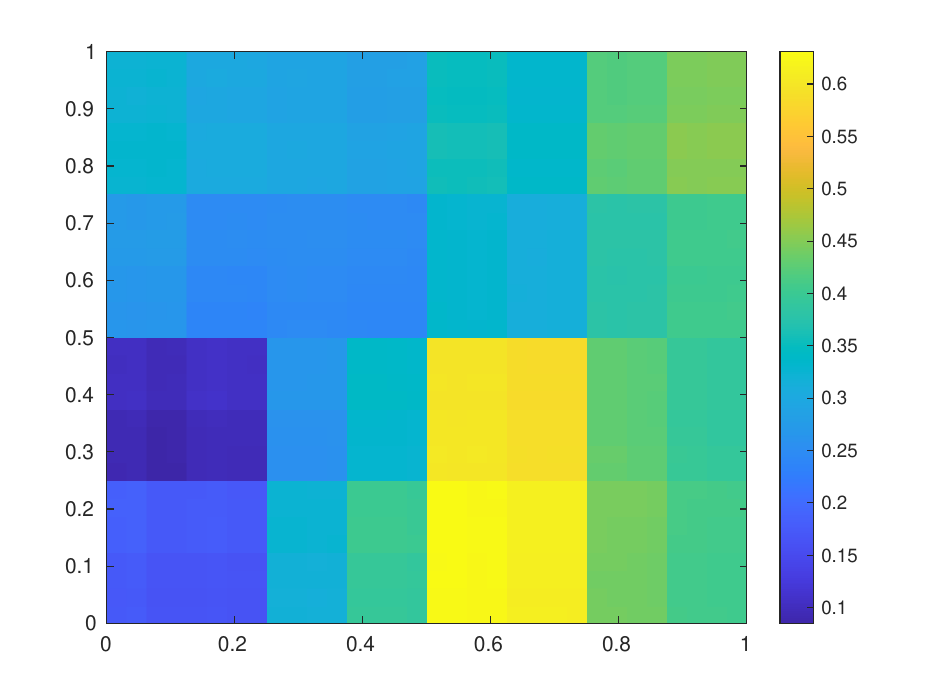}
\end{minipage}\\[1ex]
\begin{minipage}[b]{0.45\linewidth}
\includegraphics[scale=0.4]{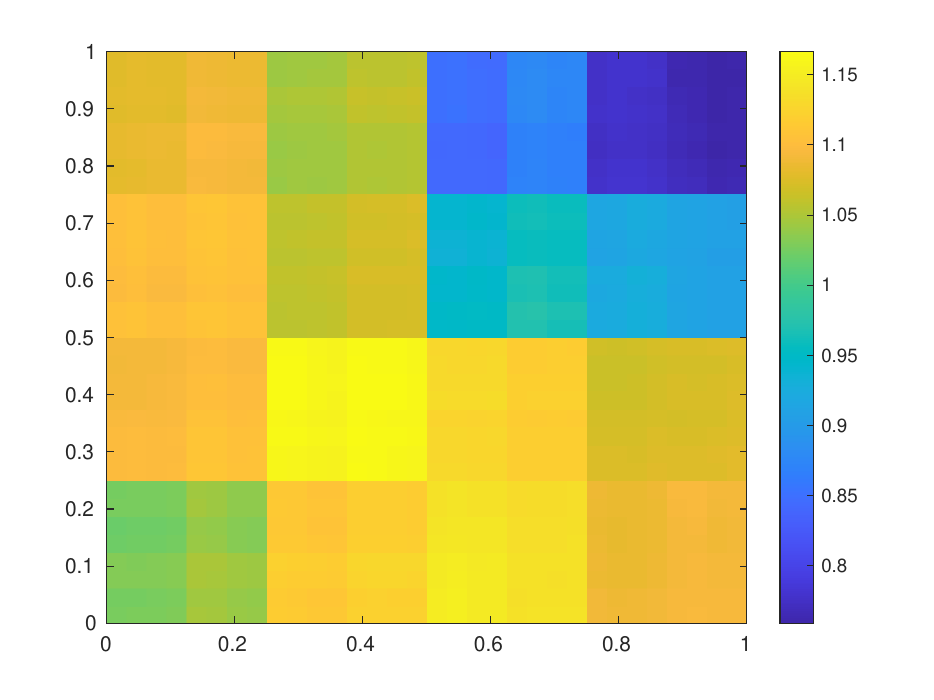}
\end{minipage}
\hfill
\begin{minipage}[b]{0.45\linewidth}
\includegraphics[scale=0.4]{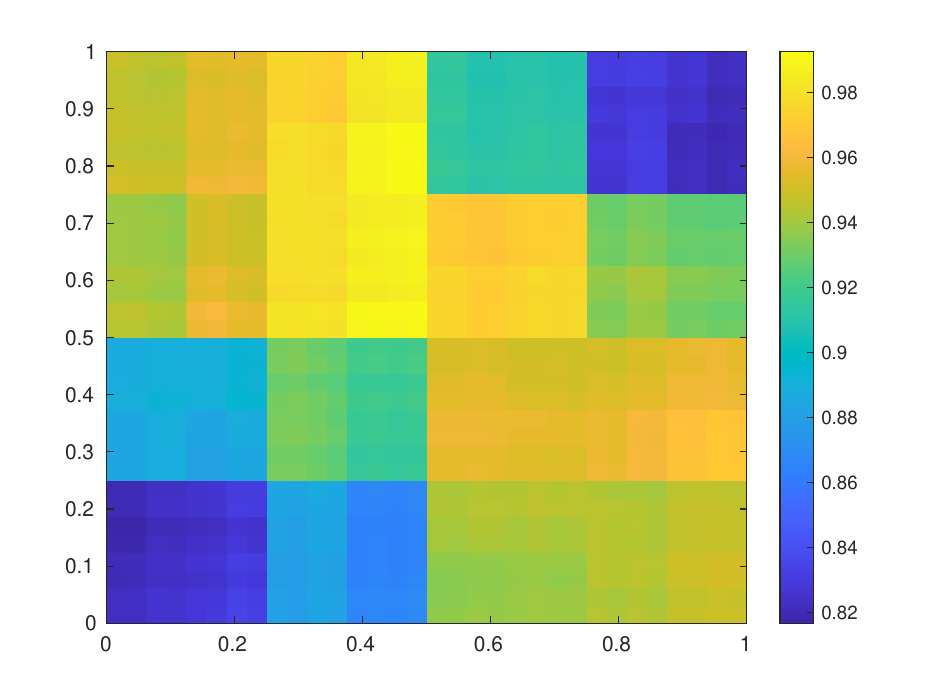}
\end{minipage}\\[1ex]
\begin{minipage}[b]{0.45\linewidth}
\includegraphics[scale=0.4]{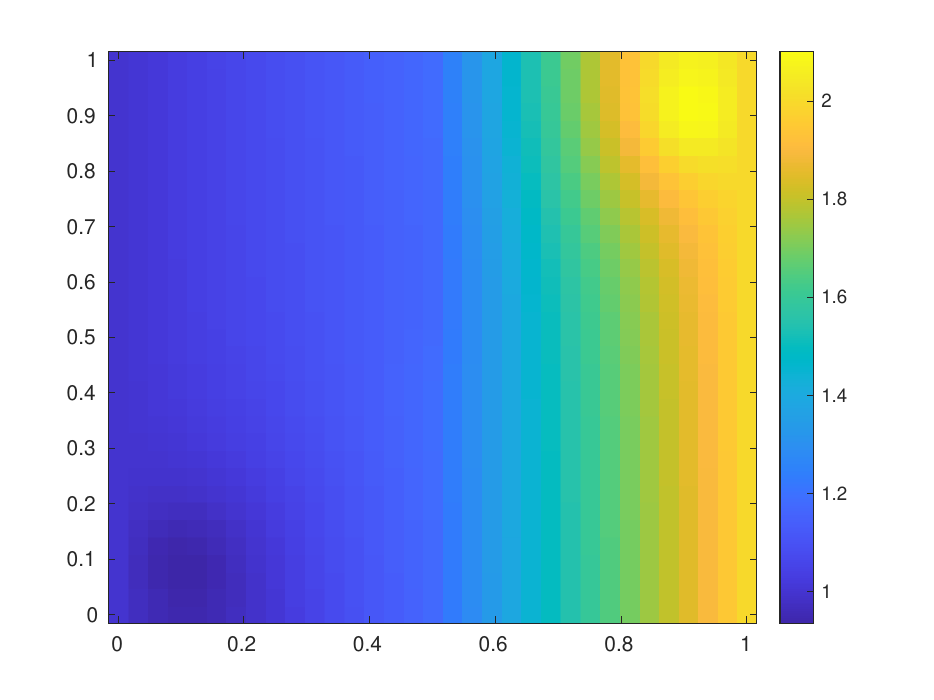}
\end{minipage}
\hfill
\begin{minipage}[b]{0.45\linewidth}
\includegraphics[scale=0.4]{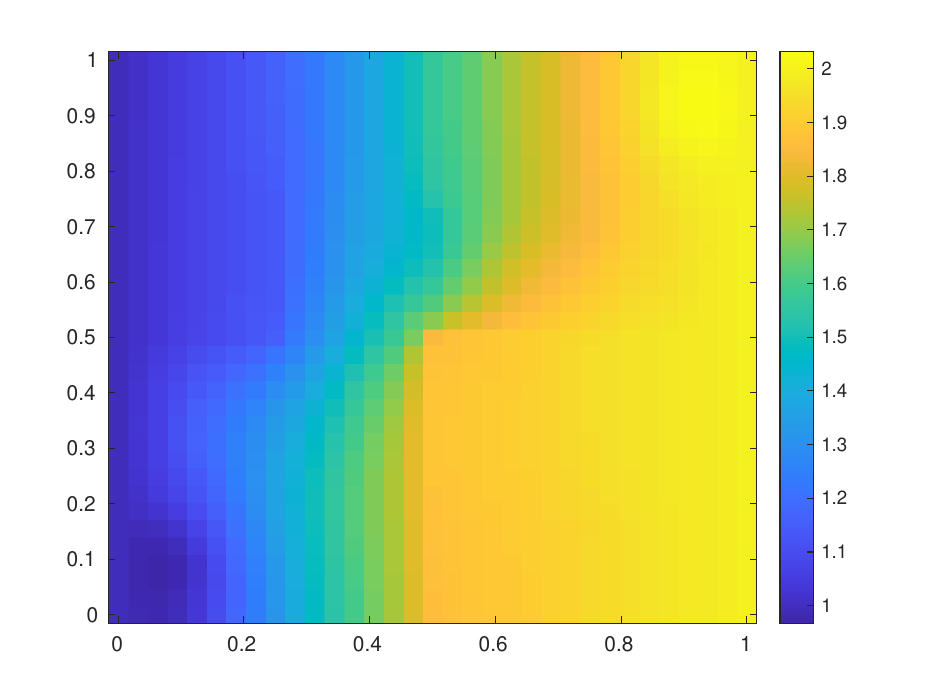}
\end{minipage}
\caption{Top two rows: mean and standard deviation of $\log \kappa$ with respect to the posterior conditioned on the data produced from the corresponding prior realizations. Bottom row: the solution computed from mean $\log\kappa$.}
\label{fig:besov:post}
\end{figure}

In Figure~\ref{fig:besov:truth} we plot two synthetic ``ground truth'' samples of the logarithm of the diffusion coefficient, and the corresponding solution used to generate the observation data.
We reduce the parameter dimension to $t=12$.
The last row of Figure~\ref{fig:besov:truth} shows the Hellinger distance between the conditional DIRT approximations and the reduced posterior \eqref{eq:api_with_independence_structure}, and the Hellinger distance between the reduced posterior and the original posterior.
In Figure~\ref{fig:besov:post} we compute posterior mean and standard deviation of $\log \kappa$ conditioned on those ``ground truth'' data.
We observe that the inferred posterior mean coefficient reproduces the observed solution.

\bibliographystyle{elsarticle-num}
\bibliography{references}

\end{document}